\documentclass{article}

\PassOptionsToPackage{numbers,sort}{natbib}
\usepackage[preprint]{neurips_2022}
\newif\ifsupplementary
\supplementarytrue 
\usepackage[utf8]{inputenc} % allow utf-8 input
\usepackage[T1]{fontenc}    % use 8-bit T1 fonts
\usepackage{url}            % simple URL typesetting
\usepackage{booktabs}       % professional-quality tables
\usepackage{amsfonts}       % blackboard math symbols
\usepackage{nicefrac}       % compact symbols for 1/2, etc.
\usepackage{microtype}      % microtypography
\usepackage{xcolor}         % colors

\usepackage[bookmarks,colorlinks]{hyperref}
\definecolor{darkblue}{rgb}{0.0, 0.0, 0.8}
\hypersetup{citecolor=cyan,linkcolor=darkblue}

\newif\ifchecklist
\checklisttrue 
\newif\ifappendix
\appendixfalse
\newif\iftableofcontents
\tableofcontentsfalse
\newif\ifacknowledgment
\acknowledgmentfalse
\makeatletter
\if@submission
	\checklisttrue
	\acknowledgmentfalse
	\ifsupplementary
		\appendixtrue
		\tableofcontentstrue
	\else
		\appendixfalse
		\tableofcontentsfalse
	\fi
\fi
\if@preprint
	\def\@noticestring{}
	\checklistfalse
	\appendixtrue 
	\tableofcontentstrue 
	\acknowledgmenttrue
\fi
\if@neuripsfinal
	\checklisttrue
	\checklistfalse
	\acknowledgmenttrue
	\ifsupplementary
		\appendixtrue
		\tableofcontentstrue
	\else
		\appendixfalse
		\tableofcontentsfalse
	\fi
\fi
\makeatother
\usepackage{mathrsfs,amsmath,amsfonts,amssymb,comment}
\usepackage{mathtools,xcolor,enumerate}
\usepackage{dsfont}
\usepackage{bm}
\usepackage{MnSymbol}
\usepackage{tabularx,multirow,array,booktabs}
\usepackage{colortbl}%%% color in tables
\usepackage[pdftex]{graphicx}%% 
\graphicspath{{figures/}}%
\usepackage{float}%% 
\usepackage{subcaption,cleveref} %% subfigure env
\usepackage{amsthm}
\theoremstyle{plain}
\newtheorem{theorem}{Theorem}[section]%  
\newtheorem{corollary}[theorem]{Corollary} %
\newtheorem{lemma}[theorem]{Lemma}
\newtheorem*{lemma*}{Lemma}

\newtheorem{proposition}[theorem]{Proposition}
\theoremstyle{definition}

\theoremstyle{remark}

\newtheorem*{remark*}{Remark}
\newcommand{\R}{\ifmmode\mathbb{R}\else$\mathbb{R}$\fi}
\newcommand{\N}{\ifmmode\mathbb{N}\else$\mathbb{N}$\fi}
\newcommand{\Z}{\ifmmode\mathbb{Z}\else$\mathbb{Z}$\fi}
\newcommand{\Q}{\ifmmode\mathbb{Q}\else$\mathbb{Q}$\fi}
\newcommand{\A}{\ifmmode\mathbb{A}\else$\mathbb{A}$\fi}
\let\tn\textnormal

\newcommand{\bmx}{{\bm{x}}}

\newcommand{\bmW}{{\bm{W}}}
\newcommand{\bmw}{{\bm{w}}}
\newcommand{\bme}{{\bm{e}}}

\newcommand{\bmb}{{\bm{b}}}

\newcommand{\bmPsi}{{\bm{\Psi}}}

\newcommand{\bmbeta}{{\bm{\beta}}}
\newcommand{\bmy}{{\bm{y}}}

\newcommand{\bmphi}{{\bm{\phi}}}

\newcommand{\bmPhi}{{\bm{\Phi}}}
\newcommand{\bmzero}{{\bm{0}}}

\newcommand{\calO}{{\mathcal{O}}}
\newcommand{\calC}{{\mathcal{C}}}

\newcommand{\calS}{{\mathcal{S}}}

\newcommand{\calA}{{\mathcal{A}}}

\newcommand{\bmcalL}{{\bm{\mathcal{L}}}}

\newcommand{\tildephi}{{\widetilde{\phi}}}
\newcommand{\tildeg}{{\widetilde{g}}}

\newcommand{\tildex}{{\widetilde{x}}}
\newcommand{\tildey}{{\widetilde{y}}}

\newcommand{\tilden}{{\widetilde{n}}}

\newcommand{\tildef}{{\widetilde{f}}}
\newcommand{\tildeM}{{\widetilde{M}}}

\newcommand{\tildepsi}{{\widetilde{\psi}}}

\newcommand{\tildedelta}{{\widetilde{\delta}}}

\newcommand{\scrE}{{\mathscr{E}}}

\newcommand{\scrL}{{\mathscr{L}}}
\newcommand{\hatn}{{\widehat{n}}}

\newcommand{\hatdelta}{{\widehat{\delta}}}
\def\one{{\ensuremath{\mathds{1}}}}
\newcommand{\middleValue}{{\tn{mid}}}

\newcommand{\nestnet}{{\hspace{0.6pt}\mathcal{N}\hspace{-1.9pt}\mathcal{N}\hspace{-1.25pt}}}

\newcommand{\bin}{\tn{bin}\hspace{1.2pt}}

\newcommand{\mystep}[2]{\par \vspace{0.25cm}\noindent\textbf{\hspace{8pt}Step }$#1\colon$ #2 \vspace{0.18cm} \par }

\usepackage{tikz}
\newcommand{\myto}[2][1]{\mathop{
		\vcenter{\hbox{\scalebox{1}[#1]{\tikz{\draw[->,line width=0.72pt] (0,0.5) to (0.69*#2,0.5);}}}}
}}

\usepackage{makecell}
\usepackage{amsopn}

\usepackage{url}
\newcommand{\subjclass}[2][1991]{%
  \let\@oldtitle\@title%
  \gdef\@title{\@oldtitle\footnotetext{#1 \emph{Mathematics subject classification.} #2}}%
}

\makeatletter
\if@preprint
\usepackage[left,mathlines]{lineno}
\usepackage{refcount}
\definecolor{mylinenumbercolor}{HTML}{BEBEBE}

\newcommand*\patchAmsMathEnvironmentForLineno[1]{%
	\expandafter\let\csname old#1\expandafter\endcsname\csname #1\endcsname
	\expandafter\let\csname oldend#1\expandafter\endcsname\csname end#1\endcsname
	\renewenvironment{#1}%
	{\linenomath\csname old#1\endcsname}%
	{\csname oldend#1\endcsname\endlinenomath}}% 
\newcommand*\patchBothAmsMathEnvironmentsForLineno[1]{%
	\patchAmsMathEnvironmentForLineno{#1}%
	\patchAmsMathEnvironmentForLineno{#1*}}%
\patchBothAmsMathEnvironmentsForLineno{equation}%
\patchBothAmsMathEnvironmentsForLineno{align}%
\patchBothAmsMathEnvironmentsForLineno{flalign}%
\patchBothAmsMathEnvironmentsForLineno{alignat}%
\patchBothAmsMathEnvironmentsForLineno{gather}%
\patchBothAmsMathEnvironmentsForLineno{multline}%
\fi
\makeatother

\makeatletter
\newcommand{\biggg}{\bBigg@{3.5}}
\newcommand{\Biggg}{\bBigg@{4.5}}
\makeatother
\let\tilde\widetilde
\let\epsilon\varepsilon

\let\mycdots\cdots
\def\cdots{{\mycdots}}
\long\def\black#1{{\color{black}#1}}

\definecolor{mypurple}{HTML}{A000A0}
\definecolor{mygray}{rgb}{0.92,0.92,0.92}

\makeatletter

\title{Neural Network Architecture Beyond\\ Width and Depth}
\author{%
	Zuowei Shen \\
	Department of Mathematics\\
	National University of Singapore\\
	\texttt{matzuows@nus.edu.sg} \\
	\And
		Haizhao Yang  \\
	Department of Mathematics\\
	University of Maryland, College Park\\
	\texttt{hzyang@umd.edu} \\
	\And
		Shijun Zhang\thanks{\hspace{1.5pt}Corresponding author.} \\
	Department of Mathematics\\
	National University of Singapore\\
	\texttt{zhangshijun@u.nus.edu} \\
}

\begin{document}	
\maketitle

% \begin{abstract}
% 	This paper introduces one more dimension, called height, in addition to two common dimensions width and depth in the characterization of neural network size.	It is proved that three-dimensional networks with height, width, and depth varying have much better approximation power than standard networks, i.e., two-dimensional ones with only width and depth varying. We say a network is of height $s$ if each hidden neuron of this network is activated by a network of height $\le s-1$. In the base case, the height of standard networks is defined as $1$. Networks of height $s$ are also called height-$s$ nested networks (NestNets) due to their nested architectures.
% 	We first prove by construction that $s$-th order ReLU NestNets with $\mathcal{O}(n)$ parameters can approximate $1$-Lipschitz continuous functions on $[0,1]^d$ within an error $\mathcal{O}(n^{-(s+1)/d})$, which is much better than the optimal error $\mathcal{O}(n^{-2/d})$ achieved by standard ReLU networks with $\mathcal{O}(n)$ parameters. 
% 	Next, we extend such a result to generic continuous functions on $[0,1]^d$ with the approximation error characterized by the modulus of continuity.	Finally, we conduct numerical experiments to verify the super-approximation power of ReLU NestNets.
% \end{abstract}

\begin{abstract}
This paper proposes a new neural network architecture by introducing an additional dimension called height beyond width and depth. Neural network architectures with height, width, and depth as hyper-parameters are called three-dimensional architectures. It is shown that neural networks with three-dimensional architectures are significantly more expressive than the ones with two-dimensional architectures (those with only width and depth as hyper-parameters), e.g., standard fully connected networks. The new network architecture is constructed recursively via a nested structure, and hence we call a network with the new architecture nested network (NestNet). A NestNet of height $s$ is built with each hidden neuron activated by a NestNet of height $\le s-1$. When $s=1$, a NestNet degenerates to a standard network with a two-dimensional architecture. It is proved by construction that height-$s$ ReLU NestNets with $\mathcal{O}(n)$ parameters can approximate $1$-Lipschitz continuous functions on $[0,1]^d$ with an error $\mathcal{O}(n^{-(s+1)/d})$, while the optimal approximation error of standard ReLU networks with $\mathcal{O}(n)$ parameters is $\mathcal{O}(n^{-2/d})$. Furthermore, such a result is extended to generic continuous functions on $[0,1]^d$ with the approximation error characterized by the modulus of continuity. Finally, we use numerical experimentation to show the advantages of the super-approximation power of ReLU NestNets.
% 	Finally, a numerical example is provided to explore the advantages of the super-approximation power of ReLU NestNets.
\end{abstract}

\section{Introduction}
\label{sec:intro}

% The approximation power of deep neural networks has been widely studied in recent years. 
% It is of great interest to investigate the (optimal) approximation error of various types of neural networks for common target functions like H\"older continuous function space. 
%
%There is a typical question in this topic: How to 
%
%In this paper, we focus on feed-forward neural networks, which are constructed via the compositions of activation functions and the affine linear maps.
%
%
%In this paper, we  introduce a new type of neural network
%with super-approximation power. 
%This type of network is built with the nested architecture, i.e., each hidden neuron is activated by a trainable sub-network. We call such networks nested networks (\textbf{NestNet}). 
In this paper, we design a new neural network architecture by introducing one more dimension, called height, in addition to width and depth in the characterization of dimensions of neural networks. 
% We call the new network architecture three-dimensional architecture since it is freely changing in three dimensions: height, width, and depth.
% Neural network architectures with height, width, and depth varying are called three-dimensional neural  network architectures.
% introducing one more dimension, called height, in addition to width and depth in the characterization of dimensions of neural networks. Neural networks with height, width, and depth varying are called three-dimensional neural  networks.
% It is shown that three-dimensional neural networks are significantly more expressive than two-dimensional ones, i.e., the classical networks with only width and depth varying (e.g., standard fully connected neural networks).
We call neural network architectures with height, width, and depth as hyper-parameters three-dimensional architectures.
It is proved by construction that neural networks with three-dimensional architectures improve the approximation power significantly, compared to standard networks with two-dimensional architectures (those with only width and depth as hyper-parameters).
The approximation power of standard neural networks has been widely studied in recent years. 
The optimality of the approximation of standard fully-connected rectified linear unit (ReLU) networks (e.g., see \cite{yarotsky18a,shijun:6,shijun:2,shijun:thesis}) implies limited room for further improvements.
This
motivates us to design a new neural network architecture by
introducing an additional dimension of height beyond width and depth.

We will focus on the ReLU ($\max\{0,x\}$) activation function and use it to demonstrate our ideas. 
% It would be interesting
% for future work to extend our work to other activation functions.
% We will prove 
% that ReLU NestNets outperform standard ReLU networks from a perspective of approximation power.
% It is of great interest to investigate the (optimal) approximation error of various types of neural networks for common target functions like H\"older continuous function space. 
Our new network architecture is constructed recursively  via a nested structure, and hence we call a neural network with the new architecture nested network (\textbf{NestNet}). 
A NestNet of height $s$ is built with each hidden neuron activated by a NestNet of height $\le s-1$. In the case of $s=1$, a NestNet degenerates to a standard network with a two-dimensional architecture.
Let us use a simple example to explain the height of a NestNet.
We say a network is activated by $\varrho_1,\cdots,\varrho_r$ if each hidden neuron of this network is activated by one of $\varrho_1,\cdots,\varrho_r$. Here, $\varrho_1,\cdots,\varrho_r$ are trainable functions mapping $\R$ to $\R$.
Then, a network of height $s\ge 2$ can be regarded as 
a $(\varrho_1,\cdots,\varrho_r)$-activated network, where $\varrho_1,\cdots,\varrho_r$ are (realized by) networks of height $\le s-1$.
%The parameters corresponding to $\varrho_1,\cdots,\varrho_r$ is called \textbf{activation parameters}, while the parameters in the affine linear maps are called \textbf{core parameters}. The number of parameters for a NestNet is the sum of the activation and core parameters.
See an example of a height-$2$ network  in Figure~\ref{fig:nesnet}. The network therein can be regarded as a $(\varrho_1,\varrho_2)$-activated network, where $\varrho_1$ and $\varrho_2$ are (realized by) networks of height $1$ (i.e., standard networks). The number of parameters in the network of Figure~\ref{fig:nesnet} is the sum of the numbers of parameters in $\bmcalL_0,\bmcalL_1,\bmcalL_2$ and $\varrho_1,\varrho_2$.

\begin{figure}[htbp!]
% 	\vskip 0.1in
	\centering
	\includegraphics[width=0.66\textwidth]{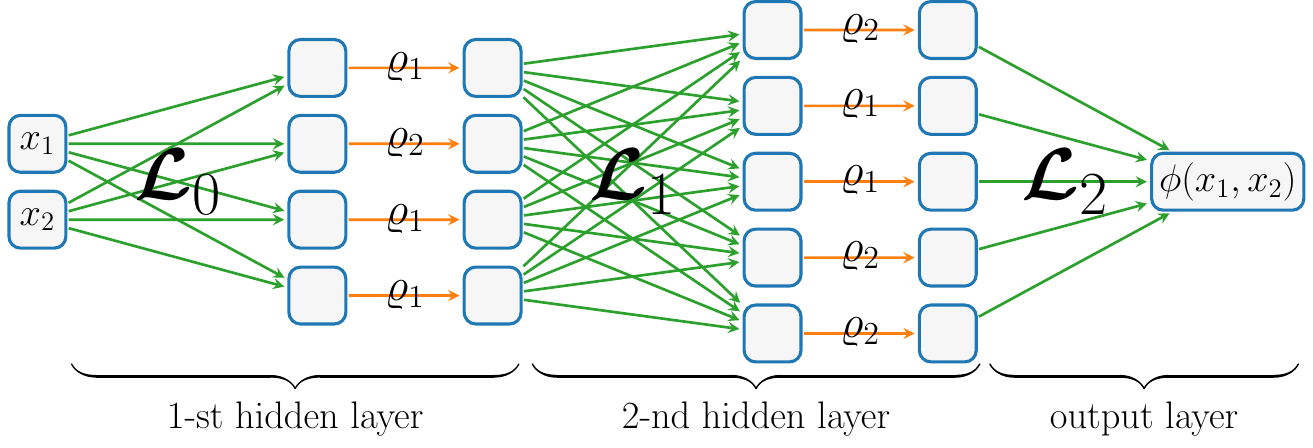}
	\caption{An example of a network of height $2$,   where $\varrho_1$ and $\varrho_2$ are (realized by) networks of height $1$ (i.e., standard networks).  Here, $\bmcalL_0$, $\bmcalL_1$ and $\bmcalL_2$ are affine linear maps. 
%		This network can also
%		be regarded as a $(\varrho_1,\varrho_2)$-activated network.
	}
	\label{fig:nesnet}
% 	\vskip 0.02in
\end{figure}

%s, i.e., networks activated by some sub-networks.
%We call such networks \textbf{nested networks}, abbreviated as \textbf{NestNets}.
%To be exact,
%we say  $(g_1,g_2,\cdots,g_k)$-activated networks are NestNets  if  $g_1,g_2,\cdots,g_k$ are normal networks. See Figure~\ref{fig:nesnet} for an example. The parameters in $g_1,g_2,\cdots,g_k$ and the outer network architecture are called activation and core parameters, respectively.
%
%Using a similar idea, we can define high-order NestNet recursively. 
%We say $(g_1,g_2,\cdots,g_k)$-activated networks are $\ell$-order NestNets if each of $g_1,g_2,\cdots,g_k$ is $j$-order NestNets with $j<\ell$, where $1$-order NestNets are the normal NestNets.

We remark that a NestNet can be regarded as a sufficiently large standard network by expanding all of its sub-network activation functions. We propose the nested network architecture since it
shares the parameters via repetitions of sub-network activation functions. 
In other words, a NestNet can provide a special parameter-sharing scheme.
This is the key reason why the NestNet has much better approximation power than the standard network. 
If we regard the network  in Figure~\ref{fig:nesnet} as a NestNet of height $2$, then the number of parameters is the sum of the numbers of parameters in $\bmcalL_0,\bmcalL_1,\bmcalL_2$ and $\varrho_1,\varrho_2$. However, if we expand the network  in Figure~\ref{fig:nesnet} to a large standard network, then the number of parameters in $\varrho_1$ and $\varrho_2$ will be added many times for computing the total number of parameters.

%See Table~\ref{tab:accuracy:comparison} for a comparison.

%Let us further discuss why we emphasize the parameters depending on the target function.
%It was shown in \cite{yarotsky18a,shijun:2,shijun:thesis,shijun:6} 
%that the approximation error $\calO(n^{-2/d})$ is (nearly) optimal for ReLU networks with $\calO(n)$ parameters to approximate $1$-Lipschitz continuous functions on $[0,1]^d$. To gain better approximation errors, existing results either consider smaller target function spaces  \cite{yarotsky:2019:06,yarotsky2017,shijun:3,barron1993,Weinan2019,doi:10.1002/mma.5575,bandlimit} or introduce new activation functions  \cite{shijun:4,shijun:7,shijun:5,pmlr-v139-yarotsky21a}. 
%Observe that, in many existing results, most parameters of networks constructed to approximate the target function $f$ are independent of $f$.
%We propose a new perspective to study the approximation error in terms of the number of parameters depending on $f$, which are called the \textbf{intrinsic parameters}, excluding those independent of $f$. We prove by construction that the approximation error can be greatly improved from our new perspective.

% Next, let us discuss the dimensions of a network architecture: height, width, and depth. 
% Recall that the width of a network (architecture) is defined as $\max_i\{N_i\}$, where $N_i$ denotes the number of neurons in $i$-th hidden layer; the depth of  a network (architecture) is defined as the number of hidden layers.
Next, let us discuss our new network architecture from the perspective of hyper-parameters.
We call the network architecture with only width as a hyper-parameter one-dimensional architecture. Its depth and height are both equal to one. Neural networks with this type of architecture are generally called shallow networks. See an example in Figure~\ref{fig:123D:net}\subref{subfig:1D}.
We call the network architecture with only width and depth as hyper-parameters two-dimensional architecture. Its height is equal to one. Neural networks with this type of architecture are generally called deep networks. See an example in Figure~\ref{fig:123D:net}\subref{subfig:2D}.
We call the network architecture with  height, width, and depth as hyper-parameters three-dimensional architecture, which is proposed
%. This is the architecture introduced 
in this paper.
Neural networks with this type of architecture are called NestNets. See an example in Figure~\ref{fig:123D:net}\subref{subfig:3D}.
One may refer to Table~\ref{tab:error:comparison} for the approximation power of networks with these three types of architectures discussed above.

\begin{table}[htbp!]   
%	\vskip 0.1in   
% 	\def\arraystretch{1.0}
	\caption{Comparison for the approximation error of $1$-Lipschitz continuous functions on $[0,1]^d$ approximated by ReLU NestNets and standard ReLU networks.  } 
	\label{tab:error:comparison}
	\vskip 0.075in
	\centering  
	\resizebox{0.95\textwidth}{!}{ 
		\begin{tabular}{ccccccccc} 			
			\toprule% \toprule[1.2pt]  
			& dimension(s)  &  \#parameters &   approximation error  &  remark & reference  \\
			\midrule

			one-hidden-layer  network & 
			width varies (depth\,=\,height\,=\,1) & $\calO(n)$ & $n^{-1}$ for $d=1$ &  linear combination     \\
			
			\midrule
			
			deep  network &  
			width and depth vary (height\,=\,1) & $\calO(n)$ & $n^{-2/d}$  & composition &  \cite{shijun:2,shijun:6,shijun:thesis,yarotsky18a}\\
			
			%			\midrule
			%			
			%			(height-$1$) NestNet & $\calO(n)$ & $n^{-3/d}$ &  & nested compositions & this paper \\
			
			\midrule
			
			 NestNet of height $s$ & 
			width, depth, and height vary  & $\calO(n)$ & $n^{-(s+1)/d}$ &  nested composition & this paper \\
			
			\bottomrule% \bottomrule[1.2pt] 
		\end{tabular} 
	}%%% \resizebox
% 		\vskip 0.05in
\end{table}

% \begin{figure}[htbp!]
% % 		\vskip 0.1in
% 	\centering
% 	\begin{subfigure}[c]{0.24\textwidth}
% 		\centering            \includegraphics[width=0.95\textwidth]{1D.pdf}
% 		\subcaption{}
% 		\label{subfig:1D}
% 	\end{subfigure}%\hspace{3pt}
% 		\hspace{1pt}
% 	\begin{subfigure}[c]{0.24\textwidth}
% 		\centering           \includegraphics[width=0.95\textwidth]{2D.pdf}
% 		\subcaption{}
% 		\label{subfig:2D}
% 	\end{subfigure}\hspace{1pt}
% % 	\\ \vspace{6pt}
% 	\begin{subfigure}[c]{0.24\textwidth}
% 		\centering           \includegraphics[width=0.95\textwidth]{3D.pdf}
% 		\subcaption{}
% 		\label{subfig:3D}
% 	\end{subfigure}\hspace{1pt}
% % 	\hspace{3pt}
% 		\begin{subfigure}[c]{0.24\textwidth}
% 		\centering           \includegraphics[width=0.95\textwidth]{3D2.pdf}
% 		\subcaption{}
% 		\label{subfig:3D2}
% 	\end{subfigure}
% 	\caption{Illustrations of neural networks with one-, two-, and three-dimensional architectures.  (a) One-dimensional case (width\,=\,3, depth\,=\,height\,=\,1). (b) Two-dimensional case (width\,=\,depth\,=\,3, height\,=\,1). (c) Three-dimensional case (width\,=\,depth\,=\,height\,=\,3). (d) Zoom-in of an activation function of the network in (c). The network in (d) can also be regarded as a network of height $2$.}
% 	\label{fig:123D:net}
% % 		\vskip 0.02in
% \end{figure}
\begin{figure}[htbp!]
% 		\vskip 0.1in
	\centering
	\begin{subfigure}[c]{0.42\textwidth}
		\centering            \includegraphics[width=0.86435\textwidth]{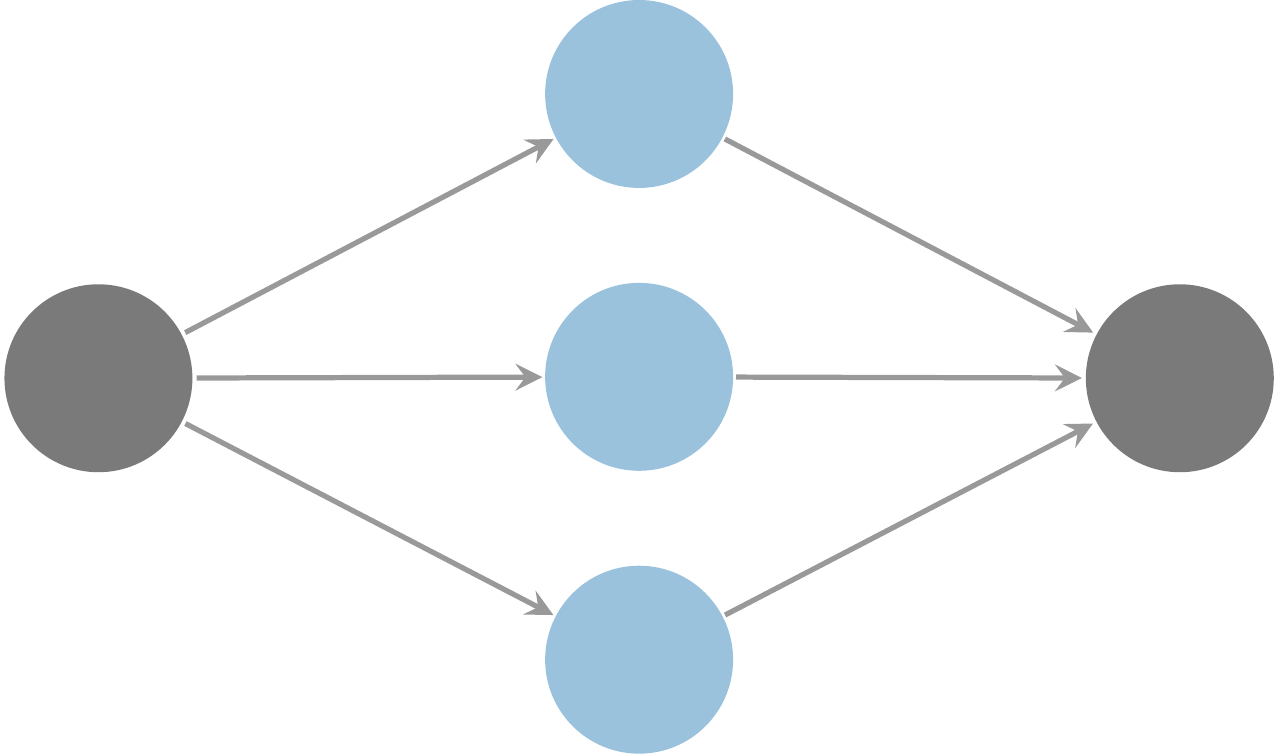}
		\subcaption{}
		\label{subfig:1D}
	\end{subfigure}\hspace{1pt}
	%	\hspace{8pt}
	\begin{subfigure}[c]{0.42\textwidth}
		\centering           \includegraphics[width=0.86435\textwidth]{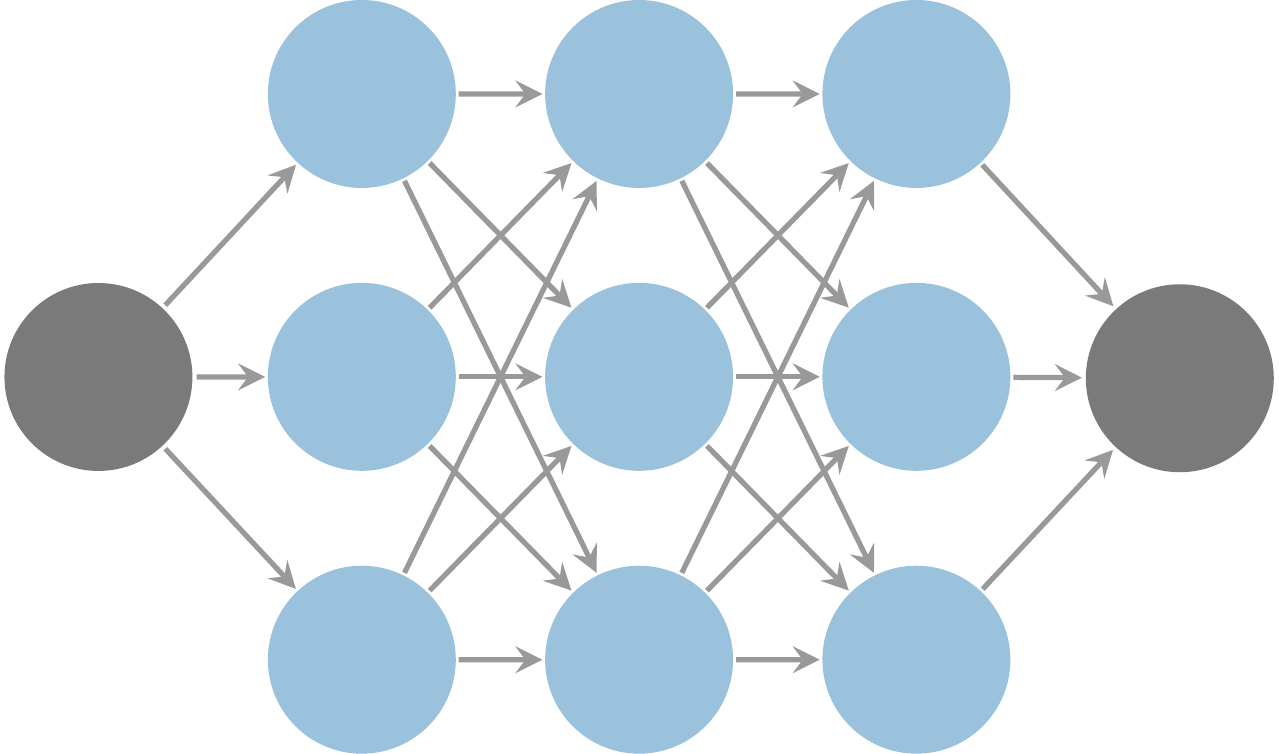}
		\subcaption{}
		\label{subfig:2D}
	\end{subfigure}\\
	\vspace{6pt}\begin{subfigure}[c]{0.42\textwidth}
		\centering           \includegraphics[width=0.86435\textwidth]{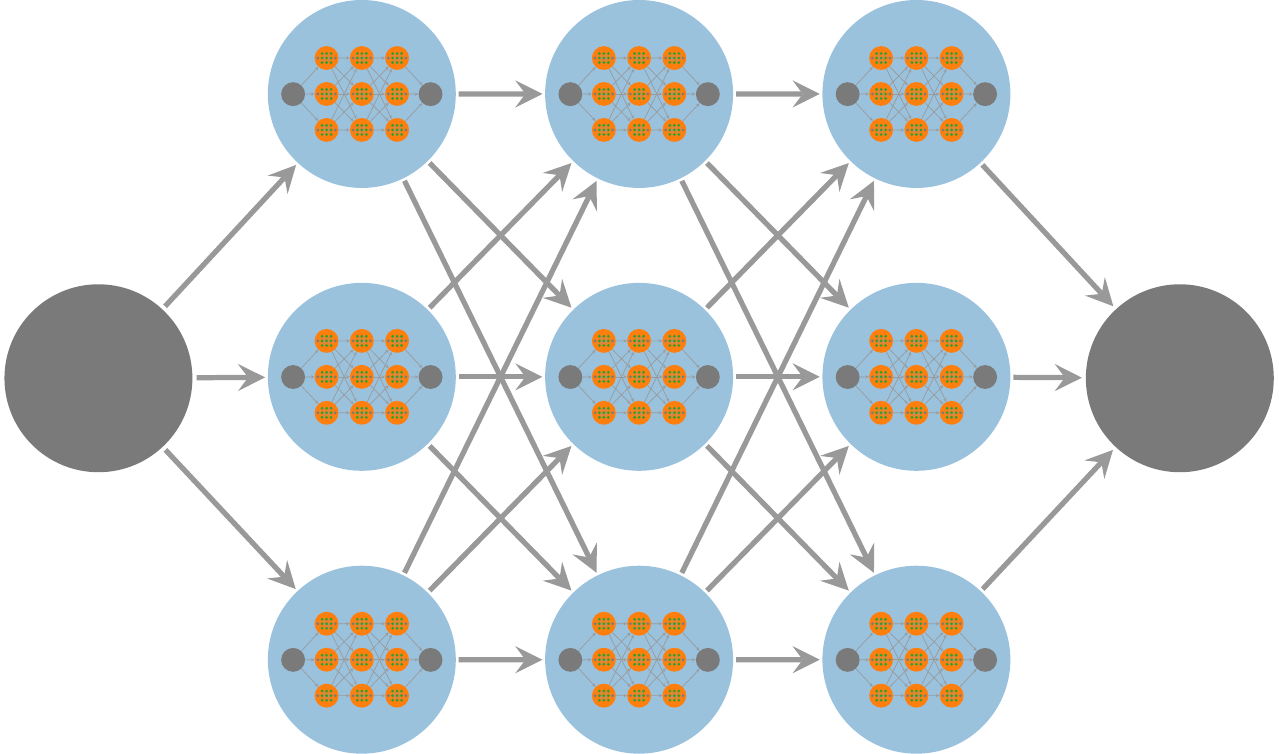}
		\subcaption{}
		\label{subfig:3D}
	\end{subfigure}\hspace{1pt}
		\begin{subfigure}[c]{0.42\textwidth}
		\centering           \includegraphics[width=0.86435\textwidth]{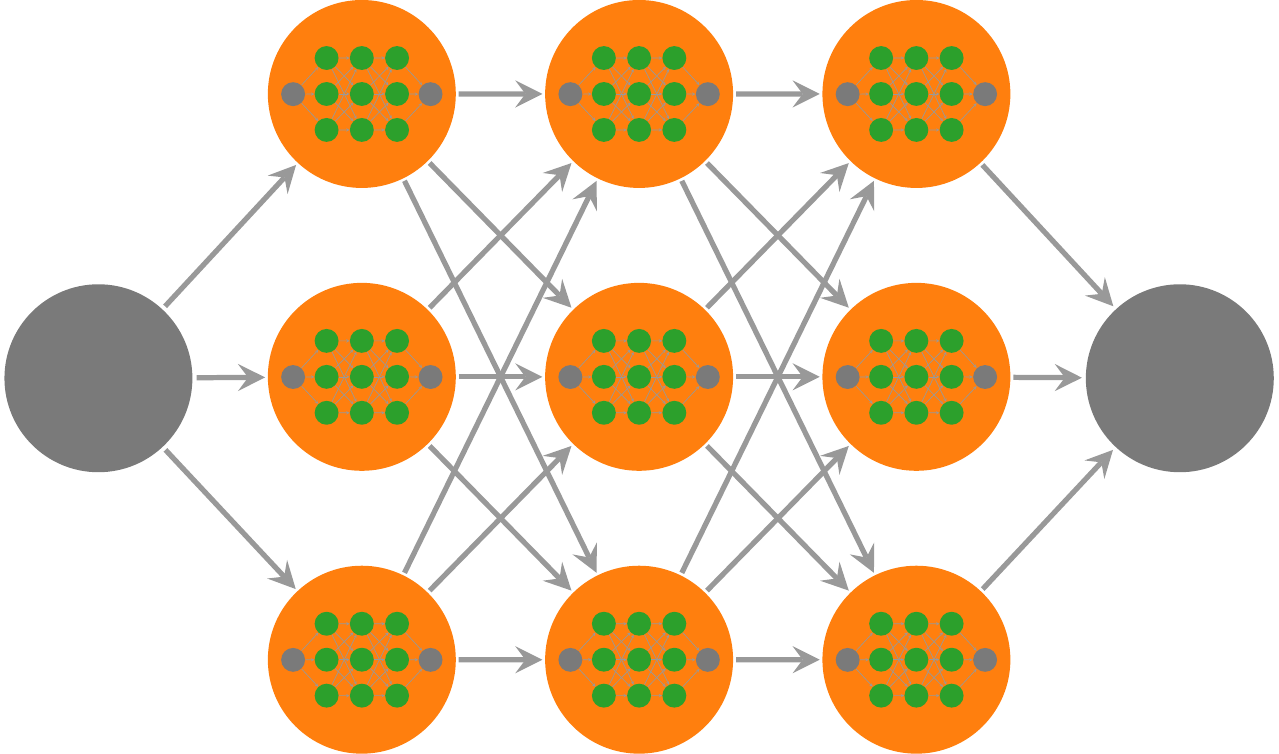}
		\subcaption{}
		\label{subfig:3D2}
	\end{subfigure}
	\caption{Illustrations of neural networks with one-, two-, and three-dimensional architectures.  (a) One-dimensional case (width\,=\,3, depth\,=\,height\,=\,1). (b) Two-dimensional case (width\,=\,depth\,=\,3, height\,=\,1). (c) Three-dimensional case (width\,=\,depth\,=\,height\,=\,3). (d) Zoom-in of an activation function of the network in (c). The network in (d) can also be regarded as a network of height $2$.}
	\label{fig:123D:net}
% 		\vskip 0.02in
\end{figure}

% We first introduce one more dimension, called height, besides two common dimensions width and depth in the characterization of neural network size. It is proved that three-dimensional networks with height, width, and depth varying have much better approximation power than standard networks, i.e., two-dimensional ones with only width and depth varying.
% Next, we 
%We say a network is of height $s$ if each hidden neuron of this network is activated by a network of height $s-1$. In the base case, the height of standard networks is $1$. Networks of height $s$ are also called height-$s$ nested networks (\textbf{NestNets}) due to their nested architectures.	
%We propose the nested architecture for neural networks (NestNets) and prove the super-approximation power of NestNets.

Our main contributions are summarized as follows.
We first propose a three-dimensional neural network architecture by introducing one more dimension called height beyond width and depth. 
We show that neural networks with three-dimensional architectures are significantly more expressive than
standard networks.
In particular, we prove that height-$s$ ReLU NestNets with $\mathcal{O}(n)$ parameters can approximate $1$-Lipschitz
continuous functions on $[0,1]^d$ with an error $\mathcal{O}(n^{-(s+1)/d})$, which is much better than the optimal error $\mathcal{O}(n^{-2/d})$ of standard ReLU networks with $\mathcal{O}(n)$ parameters. 
In the case of $s+1\ge d$, the approximation error is bounded by $\calO(n^{-(s+1)/d}) \le \calO(n^{-1})$, which means we overcome the curse of dimensionality.
Furthermore, we extend our result to generic continuous functions with the approximation error characterized by the modulus of continuity. See Theorem~\ref{thm:main} and Corollary~\ref{coro:main} for more details.
Finally, we conduct simple experiments to show the numerical advantages of the super-approximation power of ReLU NestNets.

%	\item We generalize the result in the $L^p$-norm  for $p\in[1,\infty)$ to a new one measured in the $L^\infty$-norm. Such a generalization is at a price of more intrinsic parameters. Refer to Theorem~\ref{thm:mainInfty} for more details.
%	
%	\item We further extend our results and show that the number of intrinsic parameters can be reduced to three. To be precise, ReLU networks with three intrinsic parameters can achieve an arbitrary error for approximating H\"older continuous function on $[0,1]^d$. In this scenario, extremely high precision is required as we shall see later.

% As a consequence, we overcome the curse of  dimensionality in the sense of the approximation error characterized by intrinsic parameters when the variation of $\omega_f(r)$ as $r\to 0$ is moderate (e.g., $\omega_f(r) \lesssim r^\alpha$ for H\"older continuous functions).

%\vspace{8pt}
%\textbf{Organization}:
The rest of this paper is organized as follows. 
In Section~\ref{sec:main},
we present the main results, provide the ideas of proving them, and discuss related work.
%In Section~\ref{sec:main},
%we present our main results in Section~\ref{sec:main:results}, discuss related work in Section~\ref{sec:related:work}, and present a further discussion in Section~\ref{sec:further:discussion} for a deeper understanding.
The detailed proofs of the main results are placed in the appendix. Next, we conduct experiments to show
the advantages of the super-approximation power of ReLU NestNets in Section~\ref{sec:experiments}.
Finally, Section~\ref{sec:conclusion} concludes this paper with a short discussion.

%%%%%%%%%%%%%%%%%%%%%%%%%%%%%%%%%%%%%%
%%%%%%%% main results
\section{Main results and related work}
\label{sec:main}	

In this section, we first present our main results and 
discuss the proof ideas. The detailed proofs of the main results are placed in the appendix.
Next, we discuss related work from multiple perspectives. 

%%%%%%%%%%%%%%%%%%%%
\subsection{Main results}
\label{sec:main:results}

% In the rest of this paper, 
We use $\nestnet_{s}\{n\}$ for $n,s\in\N$ to denote the set of functions realized by 
height-$s$ ReLU NestNets with as most $n$ parameters. 
We will give the mathematical definition of $\nestnet_{s}\{n\}$.
We first discuss some notations regarding affine linear maps.
We use $\scrL$ to denote the set of all affine linear maps, i.e.,
\begin{equation*}
	\scrL\coloneqq 
	\Big\{\bmcalL:\bmcalL(\bmx)=\bmW \bmx+\bmb,\   \bmW\in \R^{d_2\times d_1},\   \bmb\in\R^{d_2},\    d_1,d_2\in\N^+\Big\}.
\end{equation*}
Let $\#\bmcalL$ denote the number of parameters in $\bmcalL\in \scrL$, i.e., 
\begin{equation*}
    \#\bmcalL=(d_1+1)d_2\quad \tn{if $\bmcalL(\bmx)= \bmW \bmx+\bmb$\quad  for $\bmW\in \R^{d_2\times d_1}$ and $\bmb\in\R^{d_2}$.}
\end{equation*}

We use $\vec{g}=(\varrho_1,\cdots,\varrho_k)$ to denote an activation function vector, where $\varrho_i:\R\to\R$ is an activation function for each $i\in \{1,\cdots,k\}$. When $\vec{g}=(\varrho_1,\cdots,\varrho_k)$ is applied to a vector input $\bmx=(x_1,\cdots,x_{k})$,  
\begin{equation*}
	\vec{g}(\bmx)=\Big(\varrho_1(x_1),\,\cdots,\,\varrho_k(x_k)\Big) \quad \tn{for any $\bmx=(x_1,\cdots,x_k)\in\R^k$.}
\end{equation*} 
Let $\tn{set}(\vec{g})$ denote the function set containing all entries (functions) in $\vec{g}$. For example, if $\vec{g}=(\varrho_1,\varrho_2,\varrho_3,\varrho_2,\varrho_1)$, then $\tn{set}(\vec{g})=\{\varrho_1,\varrho_2,\varrho_3\}$.

To define $\nestnet_{s}\{n\}$ for $n,s\in\N$ recursively, we first consider 
the degenerate case. Define
\begin{equation*}
	\nestnet_0\{n\}\coloneqq \big\{\tn{id}_\R,\,\tn{ReLU}\big\}\eqqcolon \nestnet_s\{0\} \quad \tn{for $n,s\in\N$,} 
\end{equation*} 
where $\tn{id}_\R:\R\to\R$ is the identity map.
That is, we regard the identity map and ReLU
as height-$0$ ReLU NestNets with $n$ parameters
or as height-$s$ ReLU NestNets with $0$ parameters.

Next, let us present the recursive step.
For $n,s\in\N^+$, a (vector-valued) function $\bmphi\in\nestnet_{s}\{n\}$ has the following form:
\begin{equation}\label{eq:def:bmphi}
    \bmphi=\bmcalL_{m}\circ \vec{g}_m\circ 
	\ \cdots \  
	\circ \bmcalL_1\circ \vec{g}_1\circ \bmcalL_0,
\end{equation}
where $\bmcalL_0,\cdots,\bmcalL_m\in \scrL$ are affine linear maps. Moreover, Equation~\eqref{eq:def:bmphi} 
% shall
satisfies the following two conditions:
\begin{itemize}
    % \item Condition on affine linear maps:
    % \begin{equation}\label{eq:affine:linear:maps}
    %     \bmcalL_0,\cdots,\bmcalL_m\in \scrL.
    % \end{equation}
    % This condition means $\bmcalL_0,\cdots,\bmcalL_m$ are affine linear maps.
    \item Condition on activation functions:
    \begin{equation}\label{eq:activation:functions}
        \bigcup_{i=1}^{m}\tn{set}(\vec{g}_i)=\{\varrho_1,\cdots,\varrho_r\}\quad \tn{ and} \quad \varrho_j\in \bigcup_{i=0}^{s-1}\nestnet_{i}\{n_j\}\quad \tn{ for $j=1,\cdots,r$.}
    \end{equation}
     Here, $\vec{g}_i$ is an activation function vector for each $i\in\{1,\cdots,m\}$. All entries in $\vec{g}_1,\cdots,\vec{g}_m$ form an activation function set $\{\varrho_1,\cdots,\varrho_r\}$.
     For each $j\in \{1,\cdots,r\}$, $\varrho_j$  can be realized by a height-$i$ NestNet with $\le n_j$ parameters for some $i=i_j\le s-1$.
    This condition means each hidden neuron is activated by a NestNet of lower height.
    \item Condition on the number of parameters:
    \begin{equation}
    \label{eq:number:of:parameters}
        \sum_{i=0}^m \#\bmcalL_i \, +\, \sum_{j=1}^{r}n_j \le n.
    \end{equation}
     This condition means the total number of parameters is no more than $n$. The total number of parameters is calculated by adding two parts. The first one is the number of parameters in affine linear maps $\bmcalL_0,\cdots,\bmcalL_m$. The other part is the number of parameters in the activation set $\{\varrho_1,\cdots,\varrho_r\}$ formed by the entries in activation function vectors $\vec{g}_1,\cdots,\vec{g}_m$.
\end{itemize}
Then, with two conditions in Equations~%\eqref{eq:affine:linear:maps},
\eqref{eq:activation:functions} and \eqref{eq:number:of:parameters}, we can define $\nestnet_s\{n\}$ for $n,s\in\N^+$ as follows:
% \begin{small}
\begin{equation*}
    \begin{split}
        	\nestnet_s\{n\}\coloneqq \bigg\{
	\bmphi:\,\bmphi=\bmcalL_{m}\circ \vec{g}_m\circ 
	\,\cdots\,  
	\circ \bmcalL_1\circ \vec{g}_1\circ \bmcalL_0,
	\hspace{5.6pt}   
	\bmcalL_0,\cdots,\bmcalL_m\in \scrL,
	\hspace{4.5pt}      	\bigcup_{i=1}^{m}\tn{set}(\vec{g}_i)
		=\{\varrho_1,\cdots,\varrho_r\},&\\
		\varrho_j\in \bigcup_{i=0}^{s-1}\nestnet_{i}\{n_j\} 
		\hspace{5pt}    
		\tn{for $j=1,\cdots,r$},
		\hspace{5pt}     
		\sum_{i=0}^{m}\#\bmcalL_i \, +\, \sum_{j=1}^{r}n_j \le n &\bigg\}.
    \end{split}
\end{equation*}
% \end{small}
% \begin{equation*}
%     \begin{split}
%         	\nestnet_s\{n\}\coloneqq \Bigg\{
% 	\bmphi: \quad  &\bmphi=\bmcalL_{m}\circ \vec{g}_m\circ 
% 	\ \cdots \  
% 	\circ \bmcalL_1\circ \vec{g}_1\circ \bmcalL_0,\qquad
% 	\bmcalL_0,\cdots,\bmcalL_r\in \scrL,\\
% 	&
% 	\begin{aligned}	    	\bigcup_{i=1}^{m}\tn{set}(\vec{g}_i)
% 		=\{\varrho_1,\cdots,\varrho_r\},\qquad
% 		\varrho_j\in \bigcup_{i=0}^{s-1}\nestnet_{i}\{n_j\} \,\  \tn{for $j=1,\cdots,r$},&\\
% 		\sum_{i=0}^{m}\#\bmcalL_i \, +\, \sum_{j=1}^{r}n_j \le n &\,\Bigg\}.
% 	\end{aligned}
%     \end{split}
% \end{equation*}

% \begin{equation*}
% 	\nestnet_s\{n\}\coloneqq \Biggg\{\hspace{2pt}
% 	\bmphi: \bmphi=\bmcalL_{m}\circ \vec{g}_m\circ 
% 	\ \cdots \  
% 	\circ \bmcalL_1\circ \vec{g}_1\circ \bmcalL_0,\ 
% 	\begin{cases}[1.3]	
% 	    \bmcalL_0,\cdots,\bmcalL_r\in \scrL  \\ 
	
% 		\bigcup_{i=1}^{m}\tn{set}(\vec{g}_i)=\big\{\varrho_1,\cdots,\varrho_m\big\}\\
		
% 		\varrho_j\in \bigcup_{i=0}^{s-1}\nestnet_{i}\{n_j\} \  \tn{for $j=1,\cdots,r$} \\		
		
% 		\sum_{i=0}^{m}\#\bmcalL_i \ +\ \sum_{j=1}^{r}n_j \le n
% 	\end{cases} 
% 		\hspace{-9pt}\Biggg\}
% \end{equation*}
We remark that, in the definition above,
$m$ can be equal to $0$. In this case, the function $\bmphi$ degenerates to an affine linear map.

% This is the based case.
% Next, let us consider the general case $s\ge 2$. We can recursively define $\nestnet_{s}\{n\}$.

% \begin{equation*}
% 	\nestnet_s\{n\}\coloneqq \Biggg\{\hspace{2pt}
% 	\phi: \phi=\bmcalL_{m}\circ vec{g_{m}}\circ 
% 	\ \cdots \  
% 	\circ \bmcalL_1\circ \vec{g}_1\circ \bmcalL_0,\  
% 	\begin{cases}	
% 		\calS(\vec{g}_1),\cdots,\calS(vec{g_r})\subseteq \bigcup_{i=0}^{s-1}\nestnet_{i} \\		
% 		\tn{\small$\bmcalL_0,\cdots,\bmcalL_r\in \scrL$ are affine linear maps}  \\ 
% 		\tn{\small\#paremters in $\bmcalL_0,\cdots,\bmcalL_r$ and 
% 		$\vec{g}_1,\cdots, vec{g_r}$ is $\le n$}
% 	\end{cases} 
% 		\hspace{-9pt}\Biggg\}
% \end{equation*}
% and 
% $\nestnet_s\coloneqq \bigcup_{n\in\N} \nestnet_s\{n\}$
% for $s=1,2,3,\cdots$. 

% We remark that $\vec{g}_1,\cdots,\vec{g}_m$ in the definition above are activation function vectors generated by lower ReLU NestNets. 
In the NestNet example in Figure~\ref{fig:nesnet}, the function $\phi$ therein is in $\bigcup_{n\in\N}\nestnet_2\{n\}$ and the activation function vectors $\vec{g}_1$ and $\vec{g}_2$ can be represented as
\begin{equation*}
	\vec{g}_1=\big(\varrho_1,\,\varrho_2,\,\varrho_1,\,\varrho_1\big)
	\quad \tn{and}\quad \vec{g}_2=\big(\varrho_2,\,\varrho_1,\,\varrho_1,\,\varrho_2,\,\varrho_2\big).
\end{equation*}
Moreover, the activation function set containing all entries in $\vec{g}_1$ and $\vec{g}_2$ is a subset of $\bigcup_{n\in\N}\nestnet_{1}\{n\}$, i.e.,   $\{\varrho_1,\varrho_2\}\subseteq \bigcup_{n\in\N}\nestnet_{1}\{n\}$.

Let $C([0,1]^d)$ denote the set of continuous functions on $[0,1]^d$.
By convention, the modulus of continuity of a continuous function $f\in C([0,1]^d)$ is defined as 
\begin{equation*}
	\omega_f(r)\coloneqq \sup\big\{|f(\bmx)-f(\bmy)|: \|\bmx-\bmy\|_2\le r,\ \bmx,\bmy\in [0,1]^d\big\}\quad \tn{for any $r\ge0$.}
\end{equation*}
Under these settings, we can find a function in $\nestnet_{s}\big\{\calO(n)\big\}$ to approximate $f\in C([0,1]^d)$ with an approximation error $\calO\big(\omega_f(n^{-(s+1)/d})\big)$, as shown in the main theorem below.
\begin{theorem}
	\label{thm:main}
	Given a continuous function $f\in C([0,1]^d)$, for any $n,s\in\N^+$ and $p\in[1,\infty]$,
	there exists $\phi\in \nestnet_{s}\big\{C_{s,d}(n+1)  \big\}$
	such that 
	\begin{equation*}
		\|\phi-f\|_{L^p([0,1]^d)}\le 7\sqrt{d}\,\omega_f\big(n^{-(s+1)/d}\big),
	\end{equation*}
	where $C_{s,d}=10^3 d^2 (s+7)^2$ if $p\in [1,\infty)$ and  $C_{s,d}=10^{d+3} d^2 (s+7)^2$ if $p=\infty$.
\end{theorem}

We remark that the constant $C_{s,d}$ in Theorem~\ref{thm:main} is valid for all $n\in \N^+$. As we shall see later, $C_{s,d}$ can be greatly reduced if one only cares about large $n\in \N^+$.
Generally, it is challenging to simplify the approximation error in Theorem~\ref{thm:main} to make it explicitly depend on $n$
due to the complexity of $\omega_f(\cdot)$. 
However, the approximation error can be simplified to an explicit one depending on $n$ in the case of special target function spaces like H\"older continuous function space. To be exact, if $f$ is a H{\"o}lder continuous function on $[0,1]^d$ of order $\alpha\in(0,1]$ with a H\"older constant $\lambda>0$, then 
\begin{equation*}\label{eqn:Holder}
	|f(\bmx)-f(\bmy)|\leq \lambda \|\bmx-\bmy\|_2^\alpha\quad \tn{for any $\bmx,\bmy\in[0,1]^d$,}
\end{equation*}
implying $\omega_f(r)\le \lambda r^\alpha$ for any $r\ge 0$.
This means we can get an exponentially small approximation error $7\lambda\sqrt{d}\,n^{-(s+1)\alpha/d}$ as shown in Corollary~\ref{coro:main} below. 

\begin{corollary}
	\label{coro:main}
	Suppose $f$ is a H{\"o}lder continuous function on $[0,1]^d$ of order $\alpha\in(0,1]$ with a H\"older constant $\lambda>0$. For any $n,s\in \N^+$ and $p\in[1,\infty]$,
	there exists $\phi\in \nestnet_{s}\big\{C_{s,d} (n+1) \big\}$
	such that 
	\begin{equation*}
		\|\phi-f\|_{L^p([0,1]^d)}\le 7\lambda\sqrt{d}\,n^{-(s+1)\alpha/d},
	\end{equation*}
	where $C_{s,d}=10^3 d^2 (s+7)^2$ if $p\in [1,\infty)$ and  $C_{s,d}=10^{d+3} d^2 (s+7)^2$ if $p=\infty$.
\end{corollary}

In Corollary~\ref{coro:main}, if $\alpha=1$, i.e., $f$ is a Lipschitz continuous function with a Lipschitz constant $\lambda>0$, then the approximation error  can be further simplified to $7\lambda\sqrt{d}\,n^{-(s+1)/d}$. See Table~\ref{tab:error:comparison} for the comparison of the approximation error of $1$-Lipschitz continuous functions on $[0,1]^d$ approximated by ReLU NestNets and standard ReLU networks.

% as shown in the following corollary.
%\begin{corollary}
%	\label{coro:main:lipschitz}
%	Suppose $f$ is a H{\"o}lder continuous function on $[0,1]^d$ of order $\alpha\in(0,1]$ with a H\"older constant $\lambda>0$. For any $n\in \N^+$,  $s\in\N$, and $p\in[0,\infty]$,
%	there exists $\phi\in \nestnet_{s}\big\{C_{s,d} (n+1) \big\}$
%	such that 
%	\begin{equation*}
%		\|\phi-f\|_{L^p([0,1]^d)}\le 7\lambda\sqrt{d}\,n^{-(s+2)\alpha/d},
%	\end{equation*}
%	where $C_{s,d}=10^4 d^2 (s+2)^4$ if $p\in [1,\infty)$ and  $C_{s,d}=10^{d+4} d^2 (s+2)^4$ if $p=\infty$.
%\end{corollary}

%%%%%%%%%%%%%%%%%%%%%%%%%%%%
\subsection{Sketch of proving Theorem~\ref{thm:main}}
% \subsection{Sketch of proving Theorem~\ref{thm:main}}

We will discuss how to prove Theorem~\ref{thm:main}.
Given a target function $f\in C([0,1]^d)$, the key point is to
construct an almost piecewise constant function realized by a ReLU NestNet
to approximate $f$ well  except for a small region. Then we can get the desired result by dealing with the approximation in this small region. We divide the sketch of proving Theorem~\ref{thm:main} into three main steps.
\begin{enumerate}[1.]
	\item First, we divide $[0,1]^d$  into a union of  cubes $\{Q_\bmbeta\}_{\bmbeta\in \{0,1,\cdots,K-1\}^d}$ and a small region $\Omega$ with $K=\calO(n^{(s+1)/d})$. Each $Q_\bmbeta$ is associated with a representative $\bm{x}_\bmbeta\in Q_\bmbeta$ for each vector index  $\bmbeta$.  See Figure~\ref{fig:proof:sketch:thm:main} for an illustration for $K=4$ and $d=2$. 
	
	\item Next, we design a vector-valued function $\bmPhi_1(\bmx)$ to map the whole cube $Q_\bmbeta$ to its index $\bmbeta$ for each $\bmbeta$.
	Here, $\bmPhi_1$ can be defined/constructed via \[\bmPhi_1(\bmx)=\big[\phi_1(x_1),\,\phi_1(x_2),\,\cdots,\,\phi_1(x_d)\big]^T,\]
	where each one-dimensional function $\phi_1$  is a step function outside a small region.  We can efficiently construct ReLU NestNets with the desired size to approximate such an almost step function $\phi_1$ with sufficiently many ``steps'' by using the composition architecture of ReLU NestNets.
	See the appendix for the detailed construction.
	
	\item  Finally, we need to construct a function $\phi_2$ realized by a ReLU NestNet to map $\bmbeta$ approximately to $f(\bmx_\bmbeta)$ for each $\bmbeta\in \{0,1,\cdots,K-1\}^d$. Then we have 
	\begin{equation*}
		\phi_2\circ\bmPhi_1(\bmx)=\phi_2(\bmbeta)\approx f(\bmx_\bmbeta)\approx f(\bmx)
		\quad \tn{for any $\bmx\in Q_\bmbeta$ and each $\bmbeta$,}
	\end{equation*}
	 implying 
	\begin{equation*}
		\phi\coloneqq \phi_2\circ\bmPhi_1 \approx  f \quad \tn{on $[0,1]^d\backslash\Omega$}.
	\end{equation*} 
	Then, we can get a good approximation on $[0,1]^d$ by using Lemma~$3.4$ of our previous paper \cite{shijun:3} to deal with the approximation inside $\Omega$.
	We remark that, in the construction of $\phi_2:\R^d\to\R$, we only need to care about the values of $\phi_2$ at a set of $K^d$ points $\{0,1,\cdots,K-1\}^d$. As we shall see later, this is the key point to ease the design of a ReLU NestNet with the desired size to realize $\phi_2$.
% 	with the desired size. 
\end{enumerate}

\begin{figure}[htbp!]       
	\centering
	\includegraphics[width=0.72\textwidth]{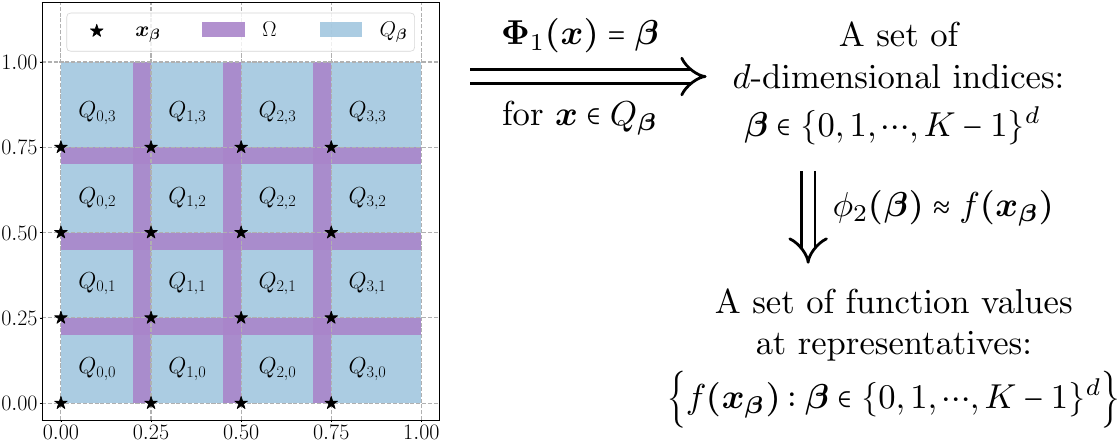}
	\caption{An illustration of the ideas of constructing  $\phi=\phi_2\circ\bmPhi_1$ to approximate $f$ for $K=4$ and $d=2$. 
		Note that $\phi\approx f$ outside $\Omega$
		  since $\phi(\bmx)=\phi_2\circ\bmPhi_1(\bmx)=\phi_2(\bmbeta)\approx f(\bmx_\bmbeta)\approx f(\bmx)$ for any $\bmx\in Q_\bmbeta$ and each $\bmbeta\in\{0,1,\cdots,K-1\}^d$.}
	\label{fig:proof:sketch:thm:main}
\end{figure}

See Figure~\ref{fig:proof:sketch:thm:main} for an illustration of the above steps.
Observe that in Figure~\ref{fig:proof:sketch:thm:main}, we have 
\begin{equation*}
	\phi(\bmx)=\phi_2\circ\bmPhi_1(\bmx)=\phi_2(\bmbeta)\mathop{\approx}\limits^{\scrE_1} f(\bmx_\bmbeta)\mathop{\approx}\limits^{\scrE_2} f(\bmx)
\end{equation*}
for any $\bmx\in Q_\bmbeta$ and each $\bmbeta\in \{0,1,\cdots,K-1\}^d$. That means $\phi-f$ is bounded by $\scrE_1+\scrE_2$ on $[0,1]^d\backslash \Omega$. 
% Clearly, $\scrE_2\le \omega_f(\sqrt{d}/K)$ since $\|\bmx-\bmx_\bmbeta\|_2\le \sqrt{d}/K$ 
For any $\bmx\in Q_\bmbeta$ and each $\bmbeta$,
we have 
\begin{equation*}
    \|\bmx_\bmbeta-\bmx\|_2\le \sqrt{d}/K\quad \Longrightarrow \quad |f(\bmx_\bmbeta)-f(\bmx)|\le \omega_f(\sqrt{d}/K)\quad \Longrightarrow \quad \scrE_2\le \omega_f(\sqrt{d}/K).
\end{equation*}
The upper bound of $\scrE_1$ is determined by the construction of $\phi_2:\R^d\to\R$. As stated previously, we only need to care about the values of $\phi_2$ at a set of $K^d$ points $\{0,1,\cdots,K-1\}^d\subseteq \R^d$, which gives us much freedom to control $\scrE_1$.
As we shall see later, $\scrE_1$ can be bounded by $\calO\big(\omega_f(\sqrt{d}/K)\big)$. Therefore, $\phi-f$ is controlled by $\calO\big(\omega_f(\sqrt{d}/K)\big)$ outside $\Omega$, from which we deduce the desired approximation error on $[0,1]^d\backslash\Omega$ since $K=\calO(n^{-(s+1)/d})$. Finally, by using Lemma~$3.4$ of our previous paper \cite{shijun:3} to deal with the approximation inside $\Omega$, we can get the desired approximation error on $[0,1]^d$.

%%%%%%%%%%%%%%%%%
\subsection{Related work}
\label{sec:related:work}

% We will connect our results to
% related existing ones for a deeper understanding.
% We 
We first compare  our results with existing ones from an approximation perspective. Next, we discuss the parameter-sharing schemes of neural networks.
Finally, 
we connect our NestNet architecture to existing trainable activation functions.  
\subsubsection*{Discussion from an approximation perspective}

The study of  the approximation power of deep neural networks has become an active topic in recent years. 
This topic has been extensively studied from many perspectives, e.g., in terms of combinatorics \cite{NIPS2014_5422}, topology \cite{ 6697897},
 information theory \cite{PETERSEN2018296},
 fat-shattering dimension \cite{Kearns,Anthony:2009}, Vapnik-Chervonenkis (VC) dimension \cite{Bartlett98almostlinear,Sakurai,pmlr-v65-harvey17a},   classical approximation theory \cite{Cybenko1989ApproximationBS,HORNIK1989359,barron1993,yarotsky18a,yarotsky2017,doi:10.1137/18M118709X,ZHOU2019,10.3389/fams.2018.00014,2019arXiv190501208G,2019arXiv190207896G,suzuki2018adaptivity,Ryumei,Wenjing,Bao2019ApproximationAO,2019arXiv191210382L,MO,shijun:1,shijun:2,shijun:3,shijun:thesis,shijun:intrinsic:parameters,shijun:arbitrary:error:with:fixed:size}, etc. 
 To the best of our knowledge, 
 the study of neural network approximation has two main stages: shallow (one-hidden-layer) networks and deep networks.
 
 In the early works of neural network approximation, 
 the approximation power of shallow networks is investigated.
 In particular,
 the universal approximation theorem  \cite{Cybenko1989ApproximationBS,HORNIK1991251,HORNIK1989359}, without approximation error estimate, showed that  a sufficiently large neural network can approximate a target function in a certain function space arbitrarily well. 
For one-hidden-layer neural networks of width $n$ and sufficiently smooth functions, an asymptotic approximation error $\calO(n^{-1/2})$ in the $L^2$-norm is proved in \cite{barron1993,barron2018approximation}, leveraging an idea that is similar to Monte Carlo sampling for high-dimensional integrals.  

Recently, a large number of works  focus on the study of deep neural networks.
It is shown in \cite{shijun:2,yarotsky18a,shijun:thesis} that the  optimal approximation error is $\calO(n^{-2/d})$ by using ReLU networks with $n$ parameters to approximate $1$-Lipschitz continuous functions on $[0,1]^d$. 
This optimal approximation error follows a natural question: How can we get  a better approximation error?
Generally, there are two ideas to get better errors. The first one is to consider smaller function spaces, e.g., smooth functions \cite{shijun:3,yarotsky:2019:06} and band-limited functions \cite{bandlimit}. The other one is to introduce new networks, e.g.,  Floor-ReLU networks \cite{shijun:4}, Floor-Exponential-Step (FLES) networks \cite{shijun:5}, and  (Sin, ReLU, $2^x$)-activated networks \cite{jiao2021deep}.

This paper proposes a three-dimensional neural network architecture by introducing one more dimension called height beyond width and depth. 	As shown in Theorem~\ref{thm:main} and Corollary~\ref{coro:main}, neural networks with three-dimensional architectures are significantly more expressive than
the ones with two-dimensional architectures. We will conduct experiments to explore the numerical properties of NestNets in Section~\ref{sec:experiments}.
% To be exact,
% we prove by construction that height-$s$ ReLU NestNets with $\mathcal{O}(n)$ parameters can approximate $1$-Lipschitz
% continuous functions on $[0,1]^d$ with an error $\mathcal{O}(n^{-(s+1)/d})$, which is better than the optimal error $\mathcal{O}(n^{-2/d})$ of standard ReLU networks with $\mathcal{O}(n)$ parameters. Such a result can be generalized to generic continuous functions in Theorem~\ref{thm:main}, i.e., the approximation error is $\calO\big(\omega_f(n^{-(s+1)/d})\big)$ for any $f\in C([0,1]^d)$.
%  Therefore, we overcome the curse of  dimensionality if $s+1\ge d$ and  the variation of $\omega_f(r)$ as $r\to 0$ is moderate (e.g., $\omega_f(r) \lesssim r^\alpha$ for H\"older continuous functions).

%First, we connect our results to transfer learning. Next, we discuss the error analysis of deep neural networks to reveal
%the motivation for reducing the number of parameters that need to be trained. Finally, we discuss related work from an approximation perspective.

\subsubsection*{Discussion from a parameter-sharing perspective}
As discussed previously, our NestNet architecture can be regarded as a sufficiently large standard network architecture with a specific parameter-sharing scheme.
Parameter-sharing schemes are used in neural networks to control the overall number of parameters for reducing memory and communication costs.
There are two common parameter-sharing schemes for a neural network.
% To the best of our knowledge, existing parameter-sharing schemes of a neural network can be roughly divided into two cases. 
The first scheme is to share parameters in the same layer. A typical network example with this scheme is the convolutional neural network (CNN). In CNN architectures, filters in a CNN layer are shared for all channels, which means the parameters in the filters are shared. 
 The second scheme is to share parameters  across different layers of networks, e.g., recurrent neural networks. 
 
 In the NestNet architecture, we share parameters via repetitions of sub-network activation functions. Both of parameter-sharing schemes discussed just above are used in the NestNet architecture. The nested architecture of NestNets gives us much freedom to determine how many parameters to share.
Beyond parameter-sharing schemes for a neural network, there are also parameter-sharing schemes among different neural networks or models, especially for multi-task learning.
One may refer to
\cite{savarese2018learning,2020arXiv200902386W,2006.10598,NEURIPS2020_42cd63cb,9859706,9879069} for more discussion on parameter sharing in neural networks.

%%%%%%%%%%%%%%%%%%%%%
%\subsection{Further discussion}
%\label{sec:further:discussion}

\subsubsection*{Connection to trainable activation functions}

The key idea of trainable activation functions is to add a small number of trainable parameters to existing activation functions. 
Let us present several existing trainable activation functions as follows. 
	A ReLU-like function is introduced in \cite{7410480} by modifying the negative part of ReLU using a trainable parameter $\alpha$, i.e.,
	the parametric ReLU (PReLU) is defined as
$    	\tn{PReLU}(x)\coloneqq 
	\left\{\tn{\footnotesize$
	\begin{array}{@{}l@{}l}
    	x &\   \   \tn{if}\  x\ge 0\\
    	\alpha x &\   \   \tn{if}\  x<0.
	\end{array}$}
		\right.$
 A variant of
	ELU unit  is introduced in \cite{Trottier2017ParametricEL} by  adding two trainable parameters $\beta,\gamma>0$, i.e., the parametric ELU (PELU) is given by 
    $    		\tn{PELU}(x)\coloneqq 
		\left\{\tn{\footnotesize$
		\begin{array}{@{}l@{}l}
			{\beta}/{\gamma} &\   \     \tn{if}\  x\ge 0\\
			\beta(\exp({x}/{\gamma})-1) x &\   \     \tn{if}\  x<0.
			\end{array}$}
			\right.$
	 Authors in \cite{8546022}  propose  a type of flexible ReLU (FReLU), which is defined via
		$
			\tn{FReLU}(x)\coloneqq 
				\tn{ReLU}(x+\alpha)+\beta,
		$
	where $\alpha$ and $\beta$ are two trainable parameters.
One may refer to \cite{APICELLA202114} for a survey of modern trainable activation functions. 
To the best of our knowledge, most existing trainable activation functions 
can be regarded as parametric variants of the original activation functions. That is, they are attained via parameterizing the original activation functions with a small number of (typically $1$ or $2$) 
trainable parameters.

By contrast, activation functions in our NestNets are much more flexible. They can be (realized by) either complicated or simple sub-NestNets.
That is, we can freely determine the number of parameters in the activation functions of NestNets. In other words, in NestNets, we can randomly distribute the parameters in the affine linear maps and activation functions. 
In short, compared to the networks with existing trainable activation functions, our NestNets are more flexible and have much more freedom in the choice of activation functions.

%(or sub-NestNets) to be activation functions. That means 
%are given by a (simple) mathematical expression with several trainable parameters based on the original activation functions. In other words, they

%NestNet
%
%In our NestNet architecture, we can randomly distribute the parameters in the affine linear maps and activation functions. However, in the existing trainable activation functions, the number of parameters in the activation is given.
%To the best of our knowledge, most trainable activation functions are given by a (simple) mathematical expression with several trainable parameters. 

%As contrast, our NestNets allow ``complicated'' sub-networks (or sub-NestNets) to be activation functions. That means our NestNets can randomly distribute the parameters in the affine linear maps and activation functions. 
%Most existing trainable activation functions can be regarded as special cases of our NestNet architecture. 
%In short, our NestNets are more flexible and have much more freedom in the choice of activation functions, compared to the networks with existing trainable activation functions. 

%%%%%%%%%%%%%%%%%%%
%\subsubsection*{Extension of main results}

%%%%%%%%%%%%%%%%%%%%%%%%%%%%%%%%%%%%%%%
%%%%%%%%% experiments
\section{Experimentation}
\label{sec:experiments}

In this section, we will conduct experiments as a proof of concept to explore the numerical properties of ReLU NestNets. It is challenging to tune the hyper-parameters of large NestNets due to their nested architectures. Thus, our experimentation focuses on relatively small NestNets of height $2$ and we introduce a simple sub-network activation function $\varrho$, which
 is realized by 
a trainable one-hidden-layer ReLU network of width $3$. To be exact,
$\varrho$ is given by 
\begin{equation}\label{eq:def:varrho}
	\varrho(x)=\bmw_1^T\cdot(x\bmw_0+\bmb_0)+b_1\quad \tn{for any $x\in\R$,}
\end{equation}
where $\bmw_0,\bmw_1,\bmb_0\in\R^3$ and $b_1\in \R$ are trainable parameters. There are $10$ parameters in $\varrho$.
The initial settings for $\varrho$ in our experiments are $\bmw_0=(1,1,1)$, $\bmw_1=(1,1,-1)$, $\bmb_0=(-0.2,-0.1,0.0)$, and $b_1=0$.
% given by As for the activation function, we adopt $\sigma$ for the standard network and $\varrho$ for the NesNet, where $\sigma$ is ReLU ($\max\{0,x\}$) and 
%  See Figure~\ref{fig:varrho} for illustrations.
We believe that NestNets can achieve good results in some real-world applications if proper optimization algorithms are developed for NestNets.
In this paper, we only consider two classification problems: a synthetic classification problem based on the Archimedean spiral in Section~\ref{sec:experiment:spiral} and 
an image classification problem corresponding to a standard benchmark dataset
Fashion-MNIST
\cite{2017arXiv170807747X} in Section~\ref{sec:experiment:fmnist}.
We remark that a classification function can be
continuously extended to $\R^d$ if each class of samples are located in a bounded closed subset of $\R^d$ and these subsets are pairwise disjoint. That means we can apply our theory to classification problems.
% two classification problems mentioned above.

% In this section
% we will conduct a simple experiment to explore the numerical advantages of the super-approximation power of ReLU NestNets. 
% To this end, we first discuss the experiment setup in Section~\ref{sec:experiment:setup} and then present the experiment results in Section~\ref{sec:experiment:results}.

\subsection{Archimedean spiral}
\label{sec:experiment:spiral}

%Since our emphasis is the approximation error,
%we need to reduce the optimization and optimization errors in our experiments. To this end, 
We will design a binary classification experiment by constructing two disjoint sets based on the Archimedean spiral, which can be described by the equation
$r=a+b\theta$
in polar coordinates $(r,\theta)$ for given $a,b\in\R$.
% To this end, we will design two disjoint sets.
% In other words, 
% we will design two sets $\calS_0$ and $\calS_1$ in $[0,1]^2$ based on the Archimedean spiral, where the label for $\calS_i$ is $i$ for $i=0,1$.
% %Next, let us discuss the construction of two classes of sets in detail.
Let us first define two curves 
(Archimedean spirals)
 as follows:
\begin{equation*}
	\widetilde{\calC}_i\coloneqq \Big\{(x,y):x=r_i\cos\theta,\  
	y=r_i\sin\theta,\   r_i=a_i+b_i\theta,\    
	\theta\in [0,s\pi]
	\Big\},
\end{equation*}
for $i=0,1$, where $a_0=0$, $a_1=1$, $b_0=b_1={1}/{\pi}$, and $s=30$.
To simplify the discussion below,
we normalize $\widetilde{\calC}_i$ as $\calC_i\subseteq [0,1]^2$, where $\calC_i$ is defined by
\begin{equation*}
	{\calC}_i\coloneqq \Big\{(x,y):x=\tfrac{\tildex}{2(s+2)}+\tfrac12,\    
	y=\tfrac{\tildey}{2(s+2)}+\tfrac12,\    
	(\tildex,\tildey)\in \widetilde{\calC}_i
	\Big\},
\end{equation*}
for $i=0,1$. Then, we can define the two desired sets as follows:
\begin{equation*}
	{\calS}_i\coloneqq \Big\{(u,v):
	\sqrt{(u-x)^2+(v-y)^2}\le \varepsilon,\   
	(x,y)\in {\calC}_i
	\Big\},
\end{equation*}
for $i=0,1$, where $\varepsilon=0.005$ in our experiments.
See an illustration for $\calS_0$ and $\calS_1$ in Figure~\ref{fig:spiral}.

\begin{figure}[htbp!]
    \centering
    \begin{minipage}{0.45\textwidth}
    % 	\vskip 0.1in
    	\centering
    	\includegraphics[height=0.58\textwidth]{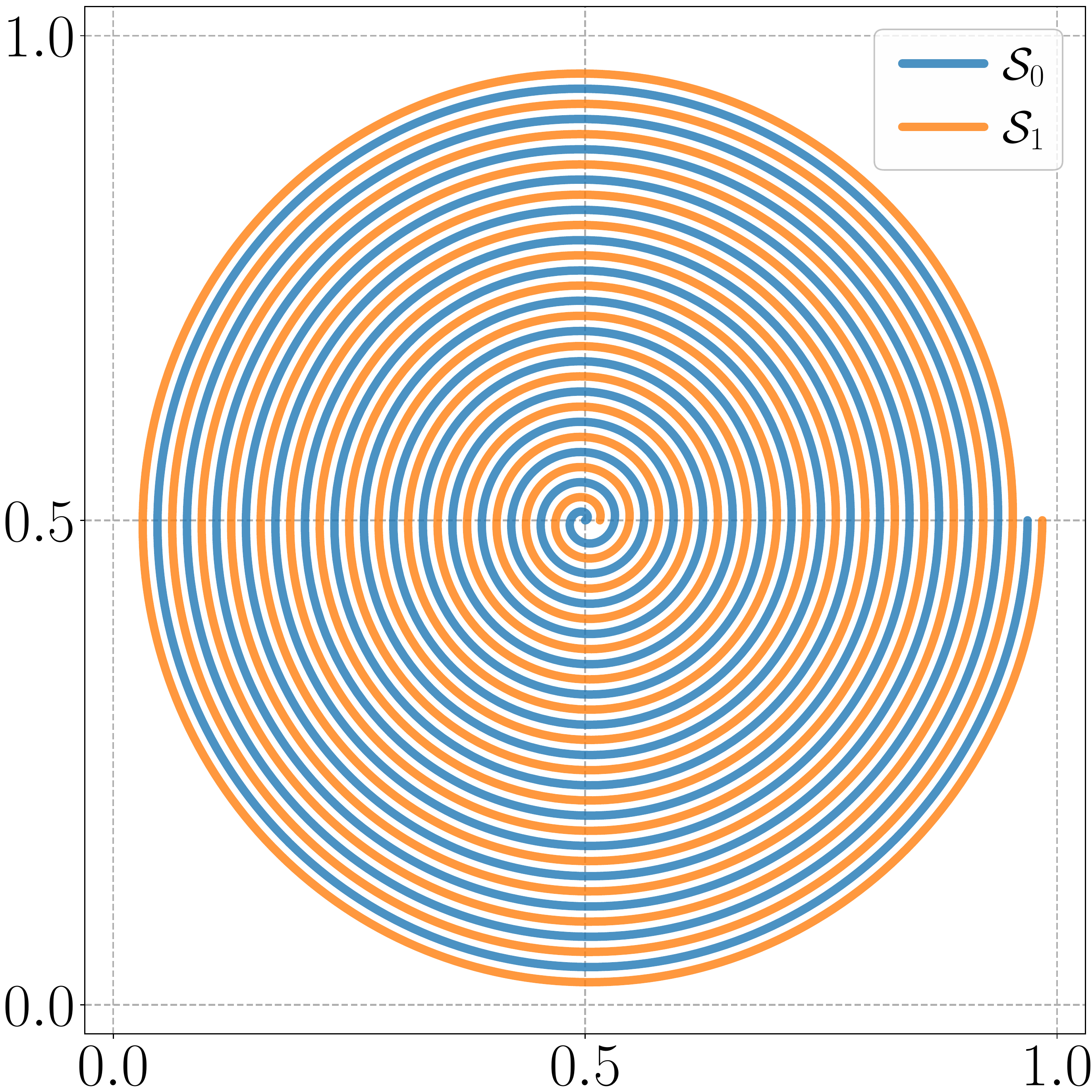}
    	\caption{An illustration for $\calS_0$ and $\calS_1$.
    	}
    	\label{fig:spiral}
    % 	\vskip 0.02in
    \end{minipage}
    \hspace{1pt}
    \begin{minipage}{0.51233\textwidth}
    % 	\vskip 0.1in
    	\centering
    	\includegraphics[height=0.5094362\textwidth]{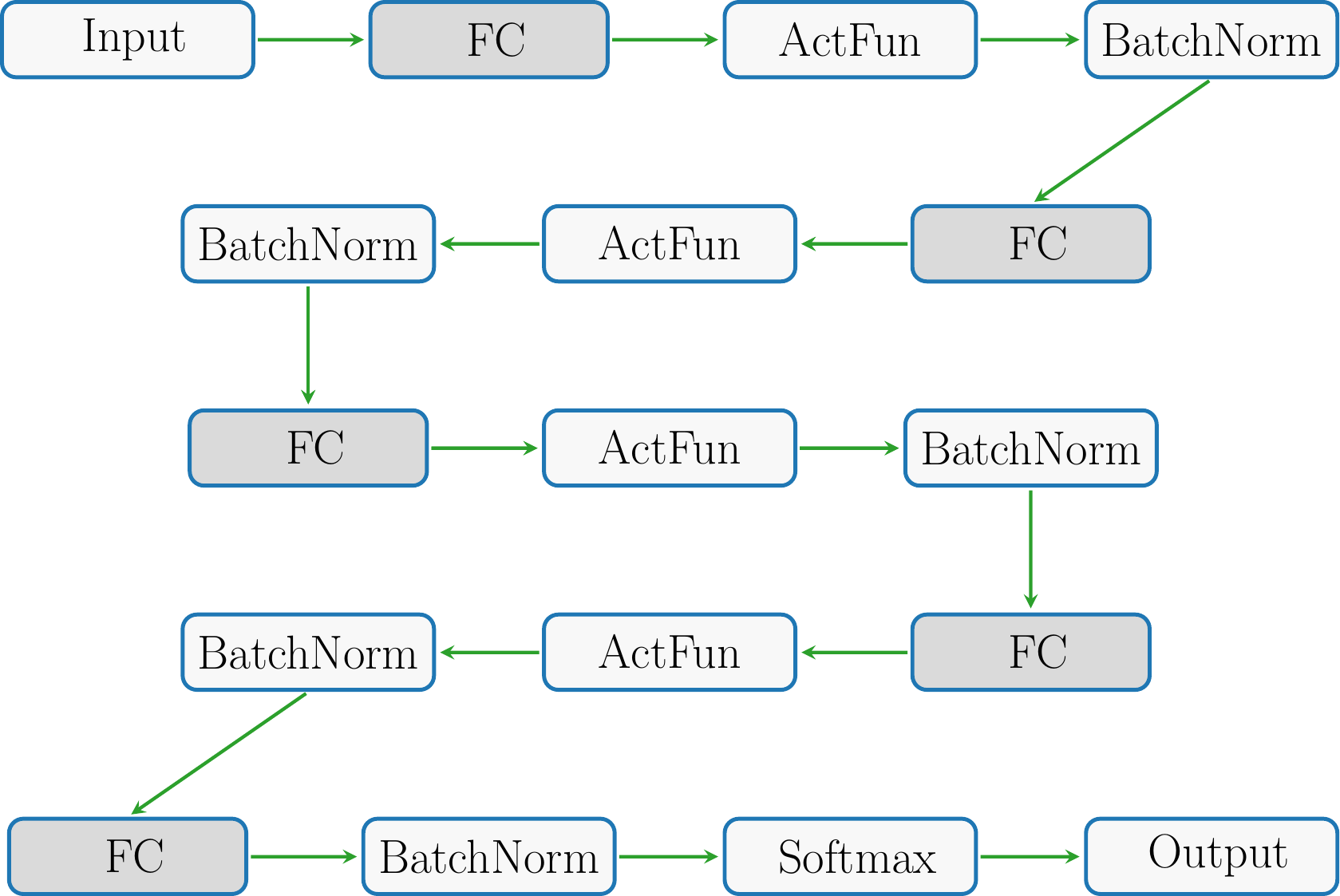}
    	\caption{A network architecture illustration.
    	}
    	\label{fig:arc:experiments}
    % 	\vskip 0.02in
    \end{minipage}
\end{figure}

To explore the numerical performance of NestNets, we design
NestNets and standard networks to classify samples in $\calS_0\bigcup \calS_1$.
We adopt four-hidden-layer fully connected network architecture of width $20$, $35$, or $50$. 
To make the optimization more stable,
we add the layers of batch normalization \cite{10.5555/3045118.3045167}. 
See Figure~\ref{fig:arc:experiments} for an illustration of the full network architecture. In Figure~\ref{fig:arc:experiments}, FC and ActFun are short of fully connected layer and activation function, respectively. ActFun is ReLU for standard networks, while for NestNets, ActFun is the learnable sub-network activation function $\varrho$ given in Equation~\eqref{eq:def:varrho}.

Before presenting the experiment results, let us present the hyper-parameters for training the networks mentioned above. For each $i\in\{0,1\}$, we randomly choose $3\times 10^5$ training samples and $3\times 10^4$ test samples in $\calS_i$ with label $i$. Then, we use these $6\times 10^5$ training samples to train the networks and use these $6\times 10^4$ test samples to compute the test accuracy.
We use the cross-entropy loss function to evaluate the loss between the networks and the target classification function. The number of epochs and the batch size are set to $500$ and $512$, respectively.
We adopt RAdam \cite{Liu2020On} as the
optimization method. 
% For the $i$-th epoch, the learning rate for the parameters in the sub-network activation function $\varrho$ is $0.2\times0.002\times 0.9^{i-1}$ and the learning rate for other parameters is $0.002\times0.9^{i-1}$.
In epochs $5(i-1)+1$ to $5i$ for $i=1,2,\cdots,100$, the learning rate is $0.2\times0.002\times0.9^{i-1}$ for the parameters in $\varrho$ and $0.002\times0.9^{i-1}$ for all other parameters. 
We remark that all training (test) samples are standardized before training, i.e., we rescale the samples to have a mean of $0$ and a standard deviation of $1$.

Finally, let us present the experiment results to compare the numerical performances of NestNets and standard networks.
% We will compare the test accuracies for NestNets and standard networks. 
We adopt the average of test accuracies in the last $100$ epochs as the target test accuracy.
As we can see from Table~\ref{tab:accuracy:comparison} and Figure~\ref{fig:accuracy:comparison}, by adding $10$ more parameters (stored in $\varrho$), NestNets achieve much better test accuracies   than standard networks though slightly more training time is required. In an ``unfair'' comparison, the test accuracy attained by
the NestNet with $1.4\times 10^3$ parameters is still better than that of the standard network with $7.9\times 10^3$ parameters. This numerically verifies that the NestNet has much better approximation power than the standard network.

\begin{table}[htbp!] 
%	\vskip 0.1in 
% 	\def\arraystretch{1.01} 
	\caption{Test accuracy comparison.} 
	\label{tab:accuracy:comparison}
%	\vskip 0.05in
\vskip 0.075in
	\centering  
	\resizebox{0.78\textwidth}{!}{ 
		\begin{tabular}{ccccccc} 
			\toprule% \toprule[1.2pt]  
			    &      width & depth & \#parameters & activation function &training time &   test accuracy  \\
			
			\midrule
			\rowcolor{mygray}
			standard network    & \multirow{1}{*}{$20$} & \multirow{1}{*}{$4$}& \multirow{1}{*}{$1362$} & ReLU  & $\approx 2532$ s  &     0.738290 \\
			
			NestNet & \multirow{1}{*}{$20$} & \multirow{1}{*}{$4$}& \multirow{1}{*}{$1362+10$} & sub-network ($\varrho$)& $\approx 4016$ s  &   0.873631 \\
			
			\midrule
			\rowcolor{mygray}
			standard network  &  \multirow{1}{*}{$35$} & \multirow{1}{*}{$4$}  &   \multirow{1}{*}{$3957$} & ReLU  & $\approx 2595$ s   & 0.816048 \\

			NestNet &  \multirow{1}{*}{$35$} & \multirow{1}{*}{$4$}  &   \multirow{1}{*}{$3957+10$} & sub-network ($\varrho$)& $\approx 4104$ s  &   0.995962 \\
			
			\midrule
            \rowcolor{mygray}
			standard network  &  \multirow{1}{*}{$50$} & \multirow{1}{*}{$4$} &  \multirow{1}{*}{$7902$} &  ReLU  & $\approx 2642$ s   &  0.866118  \\

			NestNet &\multirow{1}{*}{$50$} & \multirow{1}{*}{$4$} &  \multirow{1}{*}{$7902+10$}& sub-network ($\varrho$)& $\approx 4218$ s  &   0.999984 \\
			
%			\midrule
			
			\bottomrule% \bottomrule[1.2pt] 
		\end{tabular} 
	}%%% \resizebox
% 	\vskip 0.05in
\end{table} 

\begin{figure}[htbp!]
% 	\vskip 0.1in
	\centering
	\begin{subfigure}[b]{0.325\textwidth}
		\centering            \includegraphics[width=0.95\textwidth]{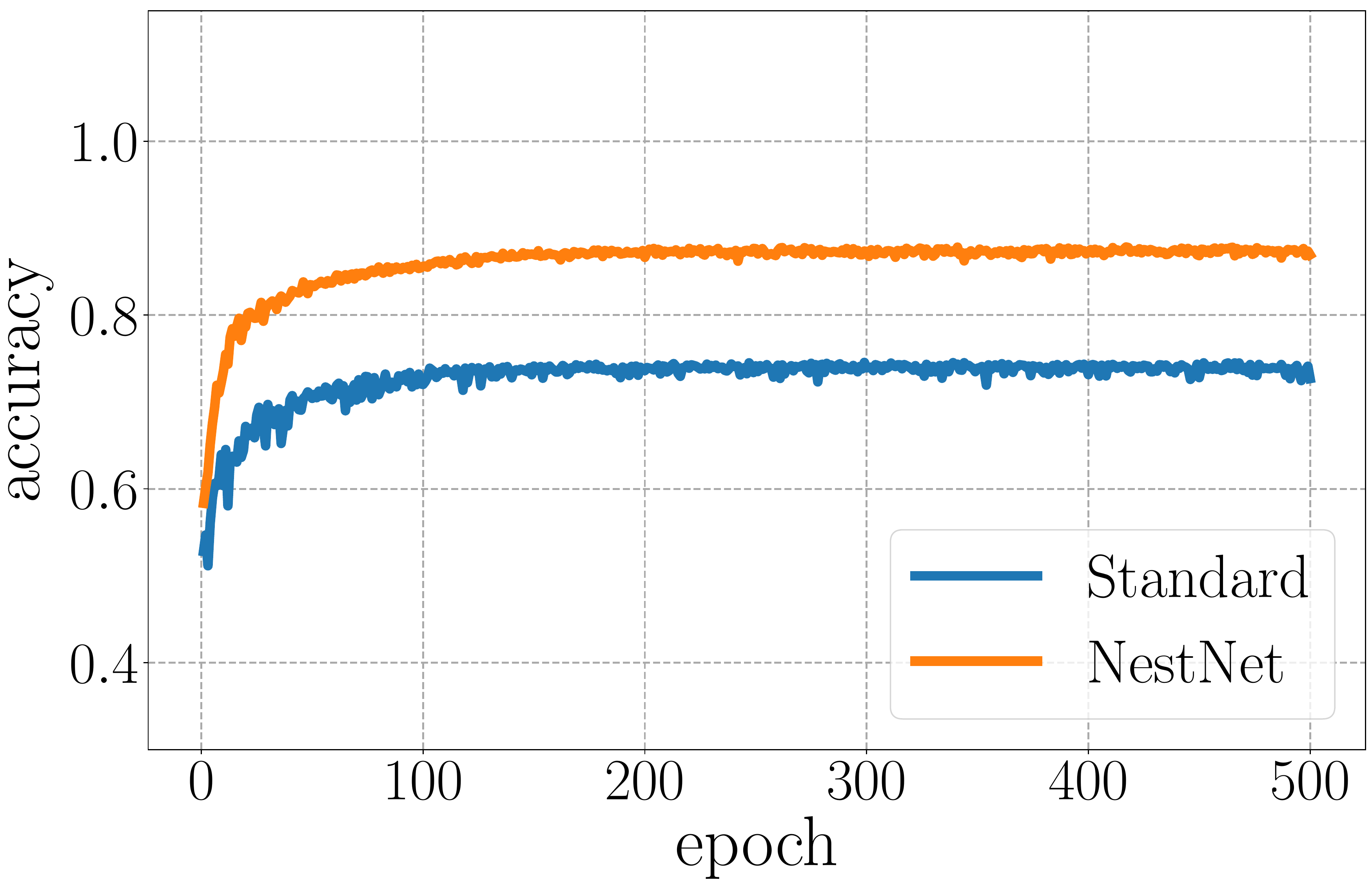}
		\subcaption{Width\,=\,20.}
	\end{subfigure}
	%	\hspace{8pt}
	\begin{subfigure}[b]{0.325\textwidth}
		\centering           \includegraphics[width=0.95\textwidth]{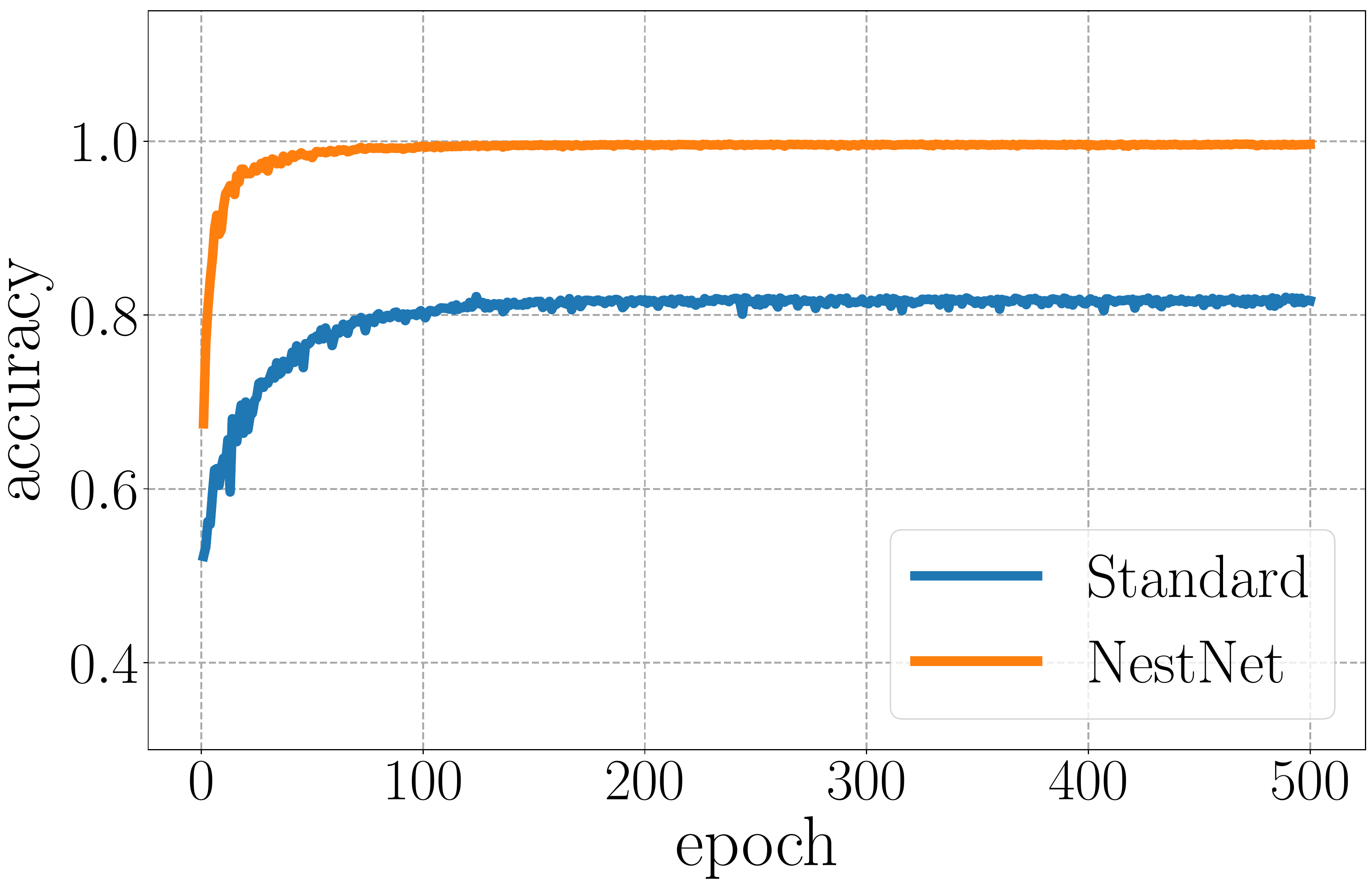}
		\subcaption{Width\,=\,35.}
	\end{subfigure}
	\begin{subfigure}[b]{0.325\textwidth}
		\centering           \includegraphics[width=0.95\textwidth]{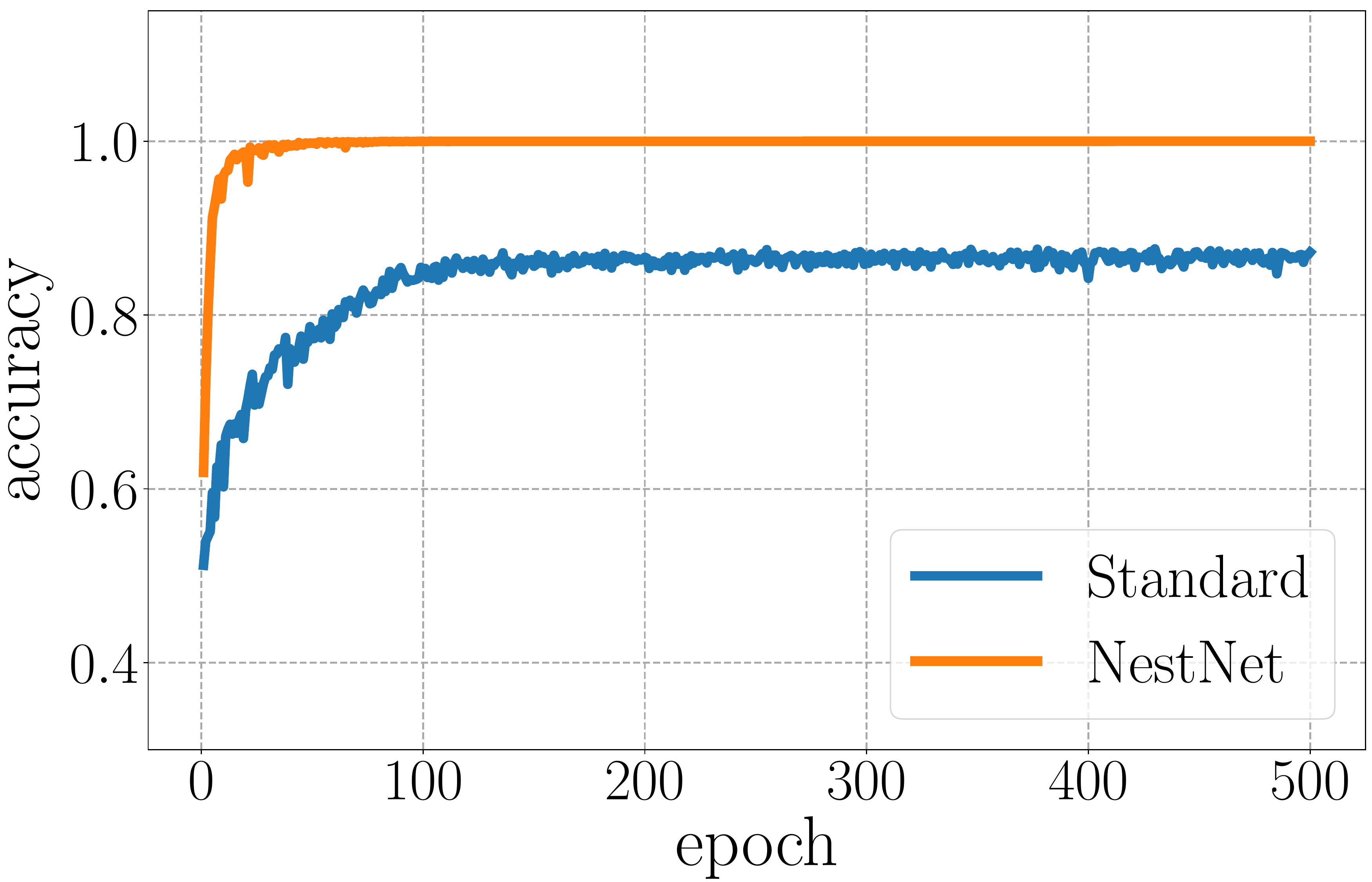}
		\subcaption{Width\,=\,50.}
	\end{subfigure}
	\caption{Test accuracy over epochs.}
	\label{fig:accuracy:comparison}
% 	\vskip 0.02in
\end{figure}

%%%%%%%%%%%%%%%%%%%%%%%%%%%%%%55
\subsection{Fashion-MNIST}
\label{sec:experiment:fmnist}

We will design convolutional neural network (CNN) architectures activated by ReLU or the sub-network activation function $\varrho$ given in Equation~\eqref{eq:def:varrho} to classify image samples in Fashion-MNIST \cite{2017arXiv170807747X}.
This dataset 
consists of a training set of $6\times 10^4$ samples and a test set of $10^4$ samples. Each sample is a $28\times28$ grayscale image, associated with a label from 10 classes. 
To compare the numerical performances of NestNets and standard networks, we design a standard CNN architecture and a NestNet architecture that is constructed by replacing a few activation functions of a standard CNN network by the sub-network activation function $\varrho$.
For simplicity, we denote the standard CNN and the NestNet as CNN1 and CNN2. 
To make the optimization more stable, we add the layers of
 dropout \cite{DBLP:journals/corr/abs-1207-0580,JMLR:v15:srivastava14a} and batch normalization \cite{10.5555/3045118.3045167}.
See illustrations of CNN1 and CNN2 in Figure~\ref{fig:CNN:arc}.
We present more details of them in Table~\ref{tab:CNN:arc}. 

\begin{figure}[htbp!]
%  	\vskip 0.1in
	\centering
	\begin{subfigure}[b]{0.49\textwidth}
		\centering            \includegraphics[width=0.985\textwidth]{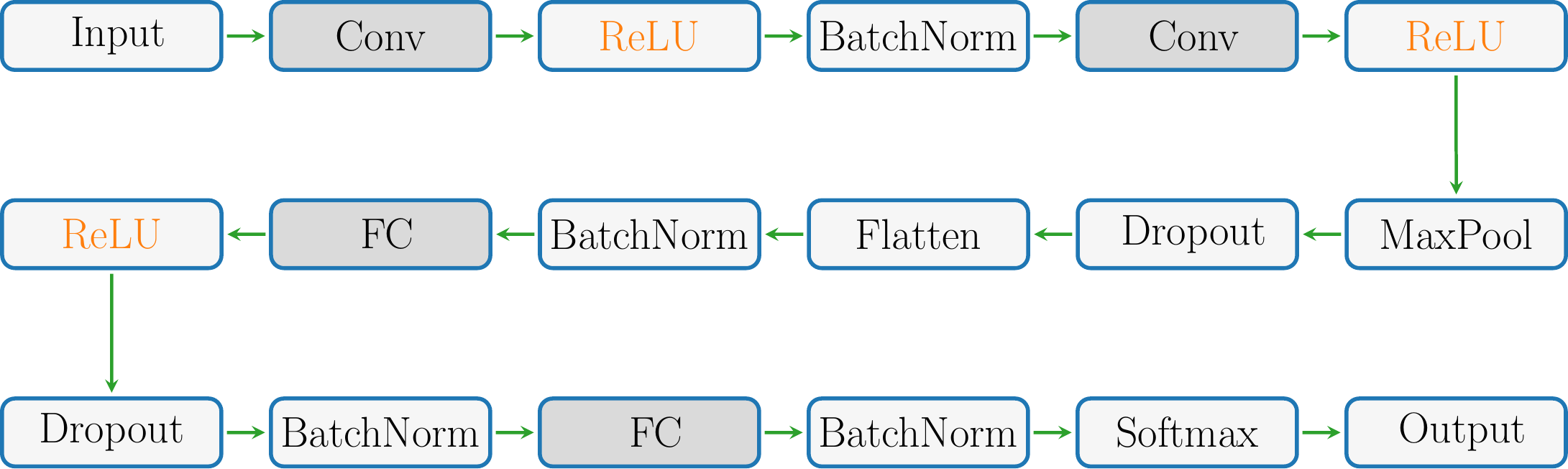}
		\vspace*{3pt}
		\subcaption{CNN1.}
	\end{subfigure}
		\hspace{3pt}
	\begin{subfigure}[b]{0.49\textwidth}
		\centering           \includegraphics[width=0.985\textwidth]{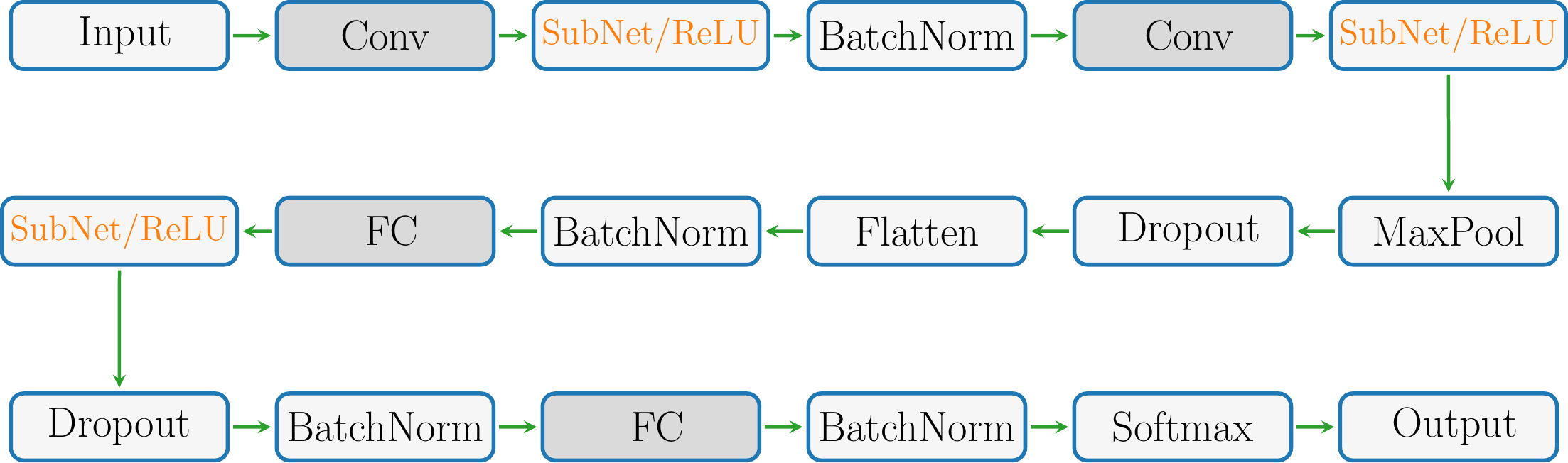}
		\vspace*{3pt}
		\subcaption{CNN2.}
	\end{subfigure}
	\caption{Illustrations of CNN1 and CNN2. Conv and FC represent convolutional  and fully connected layers, respectively. CNN2 is indeed a NestNet of height $2$.}
	\label{fig:CNN:arc}
% 	\vskip -0.2in
\end{figure}

\begin{table}[htbp!]  
% 	\vskip 0.15in	
    \caption{Details of CNN1 and CNN2.} 
	\label{tab:CNN:arc}
	\vskip 0.075in
	\centering  
	\resizebox{0.95\textwidth}{!}{ 
		\begin{tabular}{ccccccccccc} 
			\toprule% \toprule[1.2pt]  
			\multirow{2}{*}{layers}   &     
		\multicolumn{2}{c}{activation function} 
			& \multirow{2}{*}{output size of each layer} & \multirow{2}{*}{dropout} & \multirow{2}{*}{batch normalization }  \\
			\cmidrule{2-3}
% 			\cmidrule{4-5}
			& CNN1 & CNN2 \\
			
			\midrule
			input $\in \R^{28\times 28}$  &
			\phantom{$\tn{ReLU},\hspace{6pt}  22\times(26\times26)$}& & $28\times 28$   \\
			
			\midrule
			
			Conv-1: $1\times (3\times 3), \, 12$ & ReLU  & $\begin{array}{cc}
			   \tn{SubNet}\, (\varrho),\hspace{1pt}  &1\times (26\times26) \\
			    \tn{ReLU},\hspace{1pt}  &11\times(26\times26)
			\end{array}$ & $12\times(26\times 26)$ &  & yes   \\
			
			\midrule
			
			Conv-2: $12\times (3\times 3), \, 12$ & ReLU & $\begin{array}{cc}
			   \tn{SubNet}\, (\varrho),\hspace{1pt}  &1\times(24\times24) \\
			    \tn{ReLU},\hspace{1pt} & 11\times(24\times24)
			\end{array}$ & $1728$ (MaxPool \& Flatten) &  $0.25$ & yes  \\
			
			\midrule
			
			FC-1: $1728, \, 48$ & ReLU & $\begin{array}{cc}
			   \tn{SubNet}\, (\varrho),\hspace*{23pt}\, &  1\hspace*{23pt}\, \\
			    \tn{ReLU},\hspace*{23pt}\, & 47\hspace*{23pt}\,
			\end{array}$& $48$ & $0.5$  & yes   \\
				\midrule
			FC-2: $48, \, 10$ &  &    & $10$ (Softmax) &  & yes  \\
			
			\midrule
			
			output $\in \R^{10}$ &   \\
			
			\bottomrule% \bottomrule[1.2pt] 
		\end{tabular} 
	}%%% \resizebox
% 	\vskip -0.1in
\end{table}

Before presenting the numerical results, 
let us present the hyper-parameters for training two CNN architectures above. 
We use the cross-entropy loss function to evaluate the loss between the CNNs and the target classification function. The number of epochs and the batch size are set to $500$ and $128$, respectively.
We adopt 
RAdam \cite{Liu2020On}
% Adadelta \cite{DBLP:journals/corr/abs-1212-5701}
as the
optimization method and the weight decay of the optimizer is $0.0001$.
In epochs $5(i-1)+1$ to $5i$ for $i=1,2,\cdots,100$, the learning rate is $0.2\times0.002\times0.9^{i-1}$ for the parameters in $\varrho$ and $0.002\times0.9^{i-1}$ for all other parameters. 
% \blue{We remark that} 
All training (test) samples in the Fashion-MNIST dataset are standardized in our experiment, i.e., we rescale all training (test) samples to have a mean of $0$ and a standard deviation of $1$.
% In the settings above,  we repeat the experiment $12$ times and adopt the average of $12$ test accuracies as the target test accuracy for each epoch.
In the settings above,  we repeat the experiment $18$ times and discard $3$ top-performing and $3$ bottom-performing trials by using the average of test accuracy in the last 100 epochs as the performance criterion. For each epoch, we adopt the average of test accuracies in the rest $12$ trials as the target test accuracy.

Next, let us present the experiment results to compare the numerical performances of  CNN1 and CNN2.
The test accuracy comparison of CNN1 and CNN2 is summarized in Table~\ref{tab:accuracy:comparison:fmnist}. 

\begin{table}[htbp!]   
% 	\vskip 0.15in
\def\arraystretch{1.001} 
	\caption{Test accuracy comparison.} 
	\label{tab:accuracy:comparison:fmnist}
	\vskip 0.075in
	\centering  
	\resizebox{0.88\textwidth}{!}{ 
		\begin{tabular}{ccccccccc} 
			\toprule% \toprule[1.2pt]  
			 & training time &  largest accuracy & average of largest 100 accuracies & average accuracy in last 100 epochs \\
			\midrule
		\rowcolor{mygray}
		CNN1 & $\approx 5802$ s & 0.925290 & 0.924796 & 0.924447 \\			
			\midrule
			 CNN2 & $\approx 7217$ s & 0.926620  & 0.926287 &  0.926032 \\
			\bottomrule% \bottomrule[1.2pt] 
		\end{tabular} 
	}%%% \resizebox
% 	\vskip -0.1in
\end{table} 

For each of CNN1 and CNN2, we present the training time,
the largest test accuracy,  the average of the largest $100$ test accuracies, and the average of test accuracies in the last 100 epochs. 
For an intuitive comparison, we also provide illustrations of the test accuracy over epochs for CNN1 and CNN2 in Figure~\ref{fig:accuracy:comparison:fmnist}.
As we can see from Table~\ref{tab:accuracy:comparison:fmnist} and Figure~\ref{fig:accuracy:comparison:fmnist}, CNN2 performs better than CNN1 though slightly more training time and 10 more parameters are required.  This numerically
shows that the NestNet is significantly more expressive than the standard network.
% As discussed in Section~\ref{sec:further:interpretation}, the test  accuracy (error) is controlled by three errors: approximation, generalization, and optimization errors.
% A good approximation error cannot guarantee a good test accuracy.
% % Observe that the approximation error is bounded by the distance between the target function and the numerical solution.
% However, a good test accuracy numerically implies that the three errors are all well controlled. 

\begin{figure}[htbp!]
%  	\vskip 0.1in
	\centering
	\begin{subfigure}[b]{0.48\textwidth}
		\centering            \includegraphics[width=0.925\textwidth]{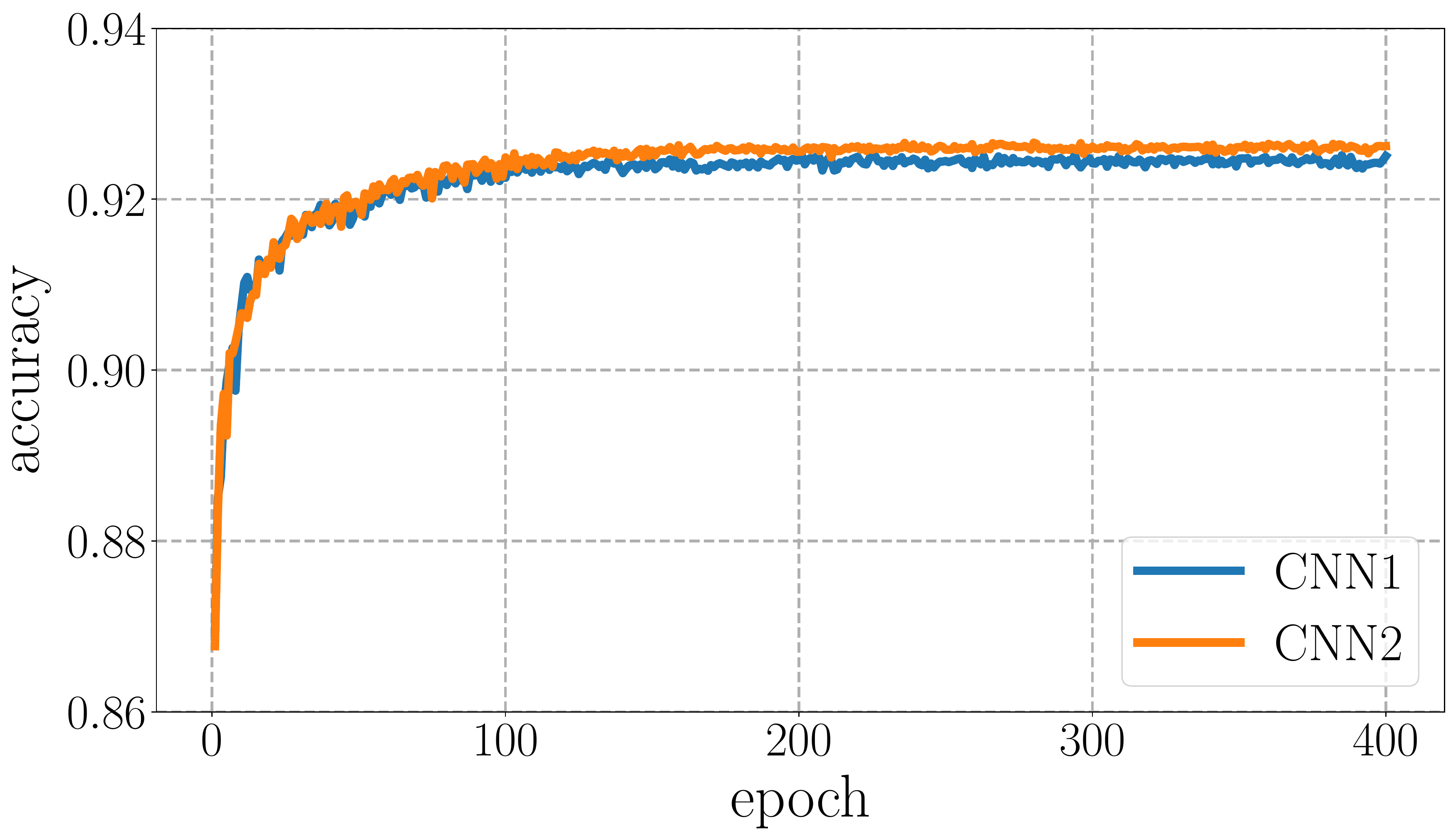}
% 		\vspace*{1pt}
		\subcaption{Epochs 1-400.}
	\end{subfigure}
		\hspace{3pt}
	\begin{subfigure}[b]{0.48\textwidth}
		\centering           \includegraphics[width=0.925\textwidth]{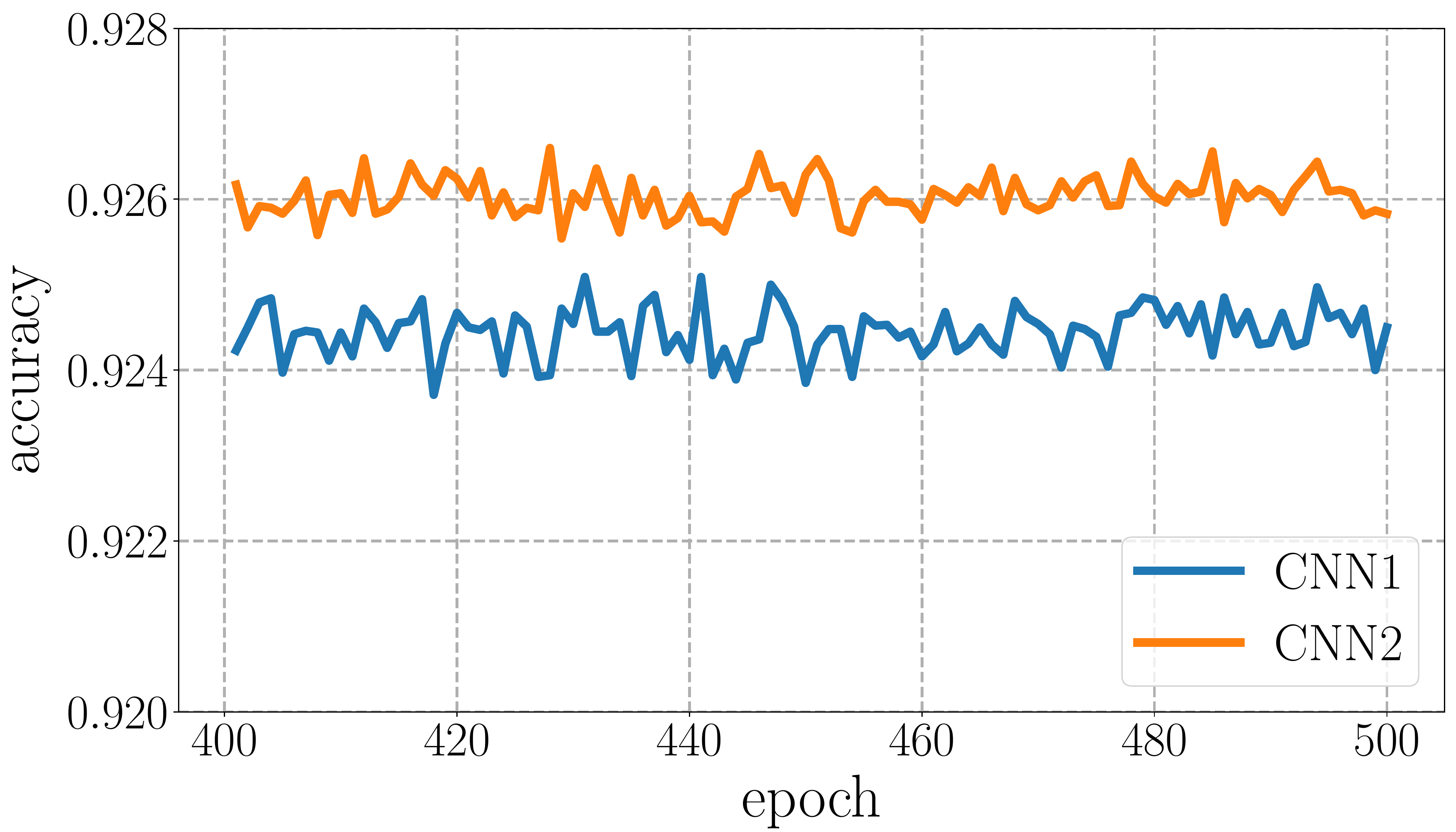}
% 		\vspace*{1pt}
		\subcaption{Epochs 401-500.}
	\end{subfigure}
	\caption{Test accuracy over epochs.}
	\label{fig:accuracy:comparison:fmnist}
% 	\vskip -0.2in
\end{figure}

%%%%%%%%%%%%%%%%%%%%%%%%%%%%%%%%%%%%%%
%%%%%%%% conclusion
\section{Conclusion}
\label{sec:conclusion}

% This paper introduces one more dimension, called height, besides two common dimensions width and depth in the characterization of neural network size.	It is proved that three-dimensional networks with, height, width, and depth varying have much better approximation power than standard networks, i.e., two-dimensional ones with only width and depth varying.
This paper proposes a three-dimensional neural network architecture by introducing one more dimension called height beyond width and depth. 
We show by construction that  neural networks with three-dimensional architectures are significantly more expressive than
the ones with two-dimensional architectures.
% In particular, we prove by construction that height-$s$ ReLU NestNets with $\mathcal{O}(n)$ parameters can approximate $1$-Lipschitz
% continuous functions on $[0,1]^d$ with an error $\mathcal{O}(n^{-(s+1)/d})$, which is much better than the optimal error $\mathcal{O}(n^{-2/d})$ achieved by standard ReLU networks with $\mathcal{O}(n)$ parameters. 
% Furthermore, we extend our result to generic continuous functions. 
% Finally, we conduct a simple experiment to show the numerical advantages of the super-approximation power of ReLU NestNets. 
We use simple numerical examples to show the advantages of the super-approximation power of ReLU NestNets, which is regarded as a proof of possibility.
% This is an example for .
% Our experiment results 
% There still is a big gap between our theoretical results and the real-world applications.
It would be of great interest to further explore the numerical performance of NestNets to bridge our theoretical results to  applications. We believe that NestNets can be further developed and  applied to real-world applications.
% This  indicates NestNets can be developed to  adapt to real applications.

We remark that our analysis is limited to  the ReLU activation function and the (H\"older) continuous function space. It would be interesting to generalize our results to other activation functions (e.g., tanh and sigmoid functions) and other function spaces (e.g, Lebesgue and Sobolev spaces). 
% Besides, it would be interesting to  explore the numerical performance of NestNets and apply it to real-world applications.

%Another important direction is to extend our results to the $L^\infty$-norm on the whole domain without the trifling region. 
%Besides, if identity maps are allowed in the construction of neural networks as in the residual networks \cite{7780459}, the size of networks in our construction can be further optimized. Finally, the proposed analysis could be generalized to other function spaces with explicit formulas to characterize the approximation error. These will be left as future work. 

\ifacknowledgment
\section*{Acknowledgments}
Z.~Shen was supported by 
Distinguished Professorship of National University of Singapore.
% Tan Chin Tuan Centennial Professorship.
H.~Yang 
was partially supported by the US National Science Foundation under award DMS-2244988, DMS-2206333, and the Office of Naval Research Award N00014-23-1-2007.
% was partially supported by the US National Science Foundation under award DMS-2244988, DMS-2206333, and the Office of Naval Research Young Investigator Award.
% Z.~Shen is supported by 
% Distinguished Professorship of National University of Singapore.
% % Tan Chin Tuan Centennial Professorship.  
% H.~Yang was partially supported by the US National Science Foundation under award DMS-2244988, DMS-2206333, and the Office of Naval Research Young Investigator Award.

\fi

%%%%%%%%%%%%%%%%%%%%%%%%%%%%%%%%%%%%%%%%
%%%%%%%%%%%%%%%% references	
%\cleardoublepage
%\section*{References}

%	\medskip

{
	\small
\bibliographystyle{plain}		
\bibliography{references}

\begin{thebibliography}{10}

\bibitem{Anthony:2009}
Martin Anthony and Peter~L. Bartlett.
\newblock {\em Neural Network Learning: Theoretical Foundations}.
\newblock Cambridge University Press, New York, NY, USA, 1st edition, 2009.

\bibitem{APICELLA202114}
Andrea Apicella, Francesco Donnarumma, Francesco Isgrò, and Roberto Prevete.
\newblock A survey on modern trainable activation functions.
\newblock {\em Neural Networks}, 138:14--32, 2021.

\bibitem{Bao2019ApproximationAO}
Chenglong Bao, Qianxiao Li, Zuowei Shen, Cheng Tai, Lei Wu, and Xueshuang
  Xiang.
\newblock Approximation analysis of convolutional neural networks.
\newblock {\em Semantic Scholar e-Preprint}, page Corpus ID: 204762668, 2019.

\bibitem{barron1993}
Andrew~R. Barron.
\newblock Universal approximation bounds for superpositions of a sigmoidal
  function.
\newblock {\em IEEE Transactions on Information Theory}, 39(3):930--945, May
  1993.

\bibitem{barron2018approximation}
Andrew~R. {Barron} and Jason~M. {Klusowski}.
\newblock Approximation and estimation for high-dimensional deep learning
  networks.
\newblock {\em arXiv e-prints}, page arXiv:1809.03090, September 2018.

\bibitem{Bartlett98almostlinear}
Peter Bartlett, Vitaly Maiorov, and Ron Meir.
\newblock Almost linear {VC}-dimension bounds for piecewise polynomial
  networks.
\newblock {\em Neural Computation}, 10(8):2159–2173, 1998.

\bibitem{6697897}
Monica Bianchini and Franco Scarselli.
\newblock On the complexity of neural network classifiers: A comparison between
  shallow and deep architectures.
\newblock {\em IEEE Transactions on Neural Networks and Learning Systems},
  25(8):1553--1565, Aug 2014.

\bibitem{doi:10.1137/18M118709X}
Helmut. B{\"o}lcskei, Philipp. Grohs, Gitta. Kutyniok, and Philipp. Petersen.
\newblock Optimal approximation with sparsely connected deep neural networks.
\newblock {\em SIAM Journal on Mathematics of Data Science}, 1(1):8--45, 2019.

\bibitem{Wenjing}
Minshuo Chen, Haoming Jiang, Wenjing Liao, and Tuo Zhao.
\newblock Efficient approximation of deep {ReLU} networks for functions on low
  dimensional manifolds.
\newblock In H.~Wallach, H.~Larochelle, A.~Beygelzimer, F.~d\textquotesingle
  Alch\'{e}-Buc, E.~Fox, and R.~Garnett, editors, {\em Advances in Neural
  Information Processing Systems}, volume~32. Curran Associates, Inc., 2019.

\bibitem{10.3389/fams.2018.00014}
Charles~K. Chui, Shao-Bo Lin, and Ding-Xuan Zhou.
\newblock Construction of neural networks for realization of localized deep
  learning.
\newblock {\em Frontiers in Applied Mathematics and Statistics}, 4:14, 2018.

\bibitem{Cybenko1989ApproximationBS}
George Cybenko.
\newblock Approximation by superpositions of a sigmoidal function.
\newblock {\em Mathematics of Control, Signals, and Systems}, 2:303--314, 1989.

\bibitem{2019arXiv190501208G}
R{\'e}mi Gribonval, Gitta Kutyniok, Morten Nielsen, and Felix Voigtlaender.
\newblock Approximation spaces of deep neural networks.
\newblock {\em Constructive Approximation}, 55:259--367, 2022.

\bibitem{2019arXiv190207896G}
Ingo {G{\"u}hring}, Gitta {Kutyniok}, and Philipp {Petersen}.
\newblock Error bounds for approximations with deep {ReLU} neural networks in
  ${W}^{s,p}$ norms.
\newblock {\em Analysis and Applications}, 18(05):803--859, 2020.

\bibitem{pmlr-v65-harvey17a}
Nick Harvey, Christopher Liaw, and Abbas Mehrabian.
\newblock Nearly-tight {VC}-dimension bounds for piecewise linear neural
  networks.
\newblock In Satyen Kale and Ohad Shamir, editors, {\em Proceedings of the 2017
  Conference on Learning Theory}, volume~65 of {\em Proceedings of Machine
  Learning Research}, pages 1064--1068, Amsterdam, Netherlands, 07--10 Jul
  2017. PMLR.

\bibitem{7410480}
Kaiming He, Xiangyu Zhang, Shaoqing Ren, and Jian Sun.
\newblock Delving deep into rectifiers: Surpassing human-level performance on
  imagenet classification.
\newblock In {\em 2015 IEEE International Conference on Computer Vision
  (ICCV)}, pages 1026--1034, 2015.

\bibitem{DBLP:journals/corr/abs-1207-0580}
Geoffrey~E. Hinton, Nitish Srivastava, Alex Krizhevsky, Ilya Sutskever, and
  Ruslan Salakhutdinov.
\newblock Improving neural networks by preventing co-adaptation of feature
  detectors.
\newblock {\em CoRR}, abs/1207.0580, 2012.

\bibitem{HORNIK1991251}
Kurt Hornik.
\newblock Approximation capabilities of multilayer feedforward networks.
\newblock {\em Neural Networks}, 4(2):251--257, 1991.

\bibitem{HORNIK1989359}
Kurt Hornik, Maxwell Stinchcombe, and Halbert White.
\newblock Multilayer feedforward networks are universal approximators.
\newblock {\em Neural Networks}, 2(5):359--366, 1989.

\bibitem{10.5555/3045118.3045167}
Sergey Ioffe and Christian Szegedy.
\newblock Batch normalization: Accelerating deep network training by reducing
  internal covariate shift.
\newblock In {\em Proceedings of the 32nd International Conference on
  International Conference on Machine Learning - Volume 37}, ICML'15, page
  448–456. JMLR.org, 2015.

\bibitem{jiao2021deep}
Yuling {Jiao}, Yanming {Lai}, Xiliang {Lu}, Fengru {Wang}, Jerry {Zhijian
  Yang}, and Yuanyuan {Yang}.
\newblock Deep neural networks with {ReLU-Sine-Exponential} activations break
  curse of dimensionality on h{\"o}lder class.
\newblock {\em arXiv e-prints}, page arXiv:2103.00542, February 2021.

\bibitem{Kearns}
Michael~J. Kearns and Robert~E. Schapire.
\newblock Efficient distribution-free learning of probabilistic concepts.
\newblock {\em J. Comput. Syst. Sci.}, 48(3):464--497, June 1994.

\bibitem{2019arXiv191210382L}
Qianxiao {Li}, Ting {Lin}, and Zuowei {Shen}.
\newblock Deep learning via dynamical systems: An approximation perspective.
\newblock {\em Journal of the European Mathematical Society}, to appear.

\bibitem{Liu2020On}
Liyuan Liu, Haoming Jiang, Pengcheng He, Weizhu Chen, Xiaodong Liu, Jianfeng
  Gao, and Jiawei Han.
\newblock On the variance of the adaptive learning rate and beyond.
\newblock In {\em International Conference on Learning Representations}, 2020.

\bibitem{shijun:3}
Jianfeng Lu, Zuowei Shen, Haizhao Yang, and Shijun Zhang.
\newblock Deep network approximation for smooth functions.
\newblock {\em SIAM Journal on Mathematical Analysis}, 53(5):5465--5506, 2021.

\bibitem{MO}
Hadrien Montanelli and Haizhao Yang.
\newblock Error bounds for deep {ReLU} networks using the {Kolmogorov-Arnold}
  superposition theorem.
\newblock {\em Neural Networks}, 129:1--6, 2020.

\bibitem{bandlimit}
Hadrien {Montanelli}, Haizhao {Yang}, and Qiang {Du}.
\newblock Deep {ReLU} networks overcome the curse of dimensionality for
  bandlimited functions.
\newblock {\em Journal of Computational Mathematics}, 39(6):801--815, 2021.

\bibitem{NIPS2014_5422}
Guido~F Montufar, Razvan Pascanu, Kyunghyun Cho, and Yoshua Bengio.
\newblock On the number of linear regions of deep neural networks.
\newblock In Z.~Ghahramani, M.~Welling, C.~Cortes, N.~D. Lawrence, and K.~Q.
  Weinberger, editors, {\em Advances in Neural Information Processing Systems
  27}, pages 2924--2932. Curran Associates, Inc., 2014.

\bibitem{Ryumei}
Ryumei Nakada and Masaaki Imaizumi.
\newblock Adaptive approximation and generalization of deep neural network with
  intrinsic dimensionality.
\newblock {\em Journal of Machine Learning Research}, 21(174):1--38, 2020.

\bibitem{PETERSEN2018296}
Philipp Petersen and Felix Voigtlaender.
\newblock Optimal approximation of piecewise smooth functions using deep {ReLU}
  neural networks.
\newblock {\em Neural Networks}, 108:296--330, 2018.

\bibitem{2006.10598}
Bryan~A. Plummer, Nikoli Dryden, Julius Frost, Torsten Hoefler, and Kate
  Saenko.
\newblock Neural parameter allocation search.
\newblock In {\em International Conference on Learning Representations}, 2022.

\bibitem{8546022}
Suo Qiu, Xiangmin Xu, and Bolun Cai.
\newblock Frelu: Flexible rectified linear units for improving convolutional
  neural networks.
\newblock In {\em 2018 24th International Conference on Pattern Recognition
  (ICPR)}, pages 1223--1228, Los Alamitos, CA, USA, aug 2018. IEEE Computer
  Society.

\bibitem{Sakurai}
Akito Sakurai.
\newblock Tight bounds for the {VC}-dimension of piecewise polynomial networks.
\newblock In {\em Advances in Neural Information Processing Systems}, pages
  323--329. Neural information processing systems foundation, 1999.

\bibitem{savarese2018learning}
Pedro Savarese and Michael Maire.
\newblock Learning implicitly recurrent {CNN}s through parameter sharing.
\newblock In {\em International Conference on Learning Representations}, 2019.

\bibitem{shijun:1}
Zuowei Shen, Haizhao Yang, and Shijun Zhang.
\newblock Nonlinear approximation via compositions.
\newblock {\em Neural Networks}, 119:74--84, 2019.

\bibitem{shijun:2}
Zuowei Shen, Haizhao Yang, and Shijun Zhang.
\newblock Deep network approximation characterized by number of neurons.
\newblock {\em Communications in Computational Physics}, 28(5):1768--1811,
  2020.

\bibitem{shijun:4}
Zuowei Shen, Haizhao Yang, and Shijun Zhang.
\newblock Deep network with approximation error being reciprocal of width to
  power of square root of depth.
\newblock {\em Neural Computation}, 33(4):1005--1036, 03 2021.

\bibitem{shijun:5}
Zuowei Shen, Haizhao Yang, and Shijun Zhang.
\newblock Neural network approximation: {T}hree hidden layers are enough.
\newblock {\em Neural Networks}, 141:160--173, 2021.

\bibitem{shijun:arbitrary:error:with:fixed:size}
Zuowei Shen, Haizhao Yang, and Shijun Zhang.
\newblock Deep network approximation: Achieving arbitrary accuracy with fixed
  number of neurons.
\newblock {\em Journal of Machine Learning Research}, 23(276):1--60, 2022.

\bibitem{shijun:intrinsic:parameters}
Zuowei Shen, Haizhao Yang, and Shijun Zhang.
\newblock Deep network approximation in terms of intrinsic parameters.
\newblock In Kamalika Chaudhuri, Stefanie Jegelka, Le~Song, Csaba Szepesvari,
  Gang Niu, and Sivan Sabato, editors, {\em Proceedings of the 39th
  International Conference on Machine Learning}, volume 162 of {\em Proceedings
  of Machine Learning Research}, pages 19909--19934. PMLR, 17--23 Jul 2022.

\bibitem{shijun:6}
Zuowei Shen, Haizhao Yang, and Shijun Zhang.
\newblock Optimal approximation rate of {ReLU} networks in terms of width and
  depth.
\newblock {\em Journal de Mathématiques Pures et Appliquées}, 157:101--135,
  2022.

\bibitem{JMLR:v15:srivastava14a}
Nitish Srivastava, Geoffrey Hinton, Alex Krizhevsky, Ilya Sutskever, and Ruslan
  Salakhutdinov.
\newblock Dropout: A simple way to prevent neural networks from overfitting.
\newblock {\em Journal of Machine Learning Research}, 15(56):1929--1958, 2014.

\bibitem{suzuki2018adaptivity}
Taiji Suzuki.
\newblock Adaptivity of deep {ReLU} network for learning in {Besov} and mixed
  smooth {Besov} spaces: optimal rate and curse of dimensionality.
\newblock In {\em International Conference on Learning Representations}, 2019.

\bibitem{Trottier2017ParametricEL}
Ludovic Trottier, Philippe Gigu{\`e}re, and Brahim Chaib-draa.
\newblock Parametric exponential linear unit for deep convolutional neural
  networks.
\newblock {\em 2017 16th IEEE International Conference on Machine Learning and
  Applications (ICMLA)}, pages 207--214, 2017.

\bibitem{9879069}
Matthew Wallingford, Hao Li, Alessandro Achille, Avinash Ravichandran, Charless
  Fowlkes, Rahul Bhotika, and Stefano Soatto.
\newblock Task adaptive parameter sharing for multi-task learning.
\newblock In {\em 2022 IEEE/CVF Conference on Computer Vision and Pattern
  Recognition (CVPR)}, pages 7551--7560, 2022.

\bibitem{NEURIPS2020_42cd63cb}
Jiaxing Wang, Haoli Bai, Jiaxiang Wu, Xupeng Shi, Junzhou Huang, Irwin King,
  Michael Lyu, and Jian Cheng.
\newblock Revisiting parameter sharing for automatic neural channel number
  search.
\newblock In H.~Larochelle, M.~Ranzato, R.~Hadsell, M.F. Balcan, and H.~Lin,
  editors, {\em Advances in Neural Information Processing Systems}, volume~33,
  pages 5991--6002. Curran Associates, Inc., 2020.

\bibitem{2020arXiv200902386W}
Ze~{Wang}, Xiuyuan {Cheng}, Guillermo {Sapiro}, and Qiang {Qiu}.
\newblock {ACDC}: Weight sharing in atom-coefficient decomposed convolution.
\newblock {\em arXiv e-prints}, page arXiv:2009.02386, September 2020.

\bibitem{2017arXiv170807747X}
Han {Xiao}, Kashif {Rasul}, and Roland {Vollgraf}.
\newblock {Fashion-MNIST}: a novel image dataset for benchmarking machine
  learning algorithms.
\newblock {\em arXiv e-prints}, page arXiv:1708.07747, August 2017.

\bibitem{yarotsky2017}
Dmitry Yarotsky.
\newblock Error bounds for approximations with deep {ReLU} networks.
\newblock {\em Neural Networks}, 94:103--114, 2017.

\bibitem{yarotsky18a}
Dmitry Yarotsky.
\newblock Optimal approximation of continuous functions by very deep {ReLU}
  networks.
\newblock In S\'ebastien Bubeck, Vianney Perchet, and Philippe Rigollet,
  editors, {\em Proceedings of the 31st Conference On Learning Theory},
  volume~75 of {\em Proceedings of Machine Learning Research}, pages 639--649.
  PMLR, 06--09 Jul 2018.

\bibitem{yarotsky:2019:06}
Dmitry Yarotsky and Anton Zhevnerchuk.
\newblock The phase diagram of approximation rates for deep neural networks.
\newblock In H.~Larochelle, M.~Ranzato, R.~Hadsell, M.~F. Balcan, and H.~Lin,
  editors, {\em Advances in Neural Information Processing Systems}, volume~33,
  pages 13005--13015. Curran Associates, Inc., 2020.

\bibitem{9859706}
Lijun Zhang, Qizheng Yang, Xiao Liu, and Hui Guan.
\newblock Rethinking hard-parameter sharing in multi-domain learning.
\newblock In {\em 2022 IEEE International Conference on Multimedia and Expo
  (ICME)}, pages 01--06, 2022.

\bibitem{shijun:thesis}
Shijun Zhang.
\newblock Deep neural network approximation via function compositions.
\newblock {\em PhD Thesis, National University of Singapore}, 2020.
\newblock URL: \url{https://scholarbank.nus.edu.sg/handle/10635/186064}.

\bibitem{ZHOU2019}
Ding-Xuan Zhou.
\newblock Universality of deep convolutional neural networks.
\newblock {\em Applied and Computational Harmonic Analysis}, 48(2):787--794,
  2020.

\end{thebibliography}

%	[1] Alexander, J.A.\ \& Mozer, M.C.\ (1995) Template-based algorithms for
%	connectionist rule extraction. In G.\ Tesauro, D.S.\ Touretzky and T.K.\ Leen
%	(eds.), {\it Advances in Neural Information Processing Systems 7},
%	pp.\ 609--616. Cambridge, MA: MIT Press.
%	
%	
%	[2] Bower, J.M.\ \& Beeman, D.\ (1995) {\it The Book of GENESIS: Exploring
%		Realistic Neural Models with the GEneral NEural SImulation System.}  New York:
%	TELOS/Springer--Verlag.
%	
%	
%	[3] Hasselmo, M.E., Schnell, E.\ \& Barkai, E.\ (1995) Dynamics of learning and
%	recall at excitatory recurrent synapses and cholinergic modulation in rat
%	hippocampal region CA3. {\it Journal of Neuroscience} {\bf 15}(7):5249-5262.
}

%%%%%%%%%%%%%%%%%%%%%%%%%%%%%%%%%%%%%%%%%%%%%%%%%%%%%%%%%%%%
%%%%%%%%%%%%%%%% checklist
\ifchecklist
\section*{Checklist}

%	
%	%%% BEGIN INSTRUCTIONS %%%
%	The checklist follows the references.  Please
%	read the checklist guidelines carefully for information on how to answer these
%	questions.  For each question, change the default \answerTODO{} to \answerYes{},
%	\answerNo{}, or \answerNA{}.  You are strongly encouraged to include a {\bf
%		justification to your answer}, either by referencing the appropriate section of
%	your paper or providing a brief inline description.  For example:
%	\begin{itemize}
%		\item Did you include the license to the code and datasets? \answerYes{See Section~\ref{gen_inst}.}
%		\item Did you include the license to the code and datasets? \answerNo{The code and the data are proprietary.}
%		\item Did you include the license to the code and datasets? \answerNA{}
%	\end{itemize}
%	Please do not modify the questions and only use the provided macros for your
%	answers.  Note that the Checklist section does not count towards the page
%	limit.  In your paper, please delete this instructions block and only keep the
%	Checklist section heading above along with the questions/answers below.
%	%%% END INSTRUCTIONS %%%

\begin{enumerate}

	\item For all authors...
	\begin{enumerate}
		\item Do the main claims made in the abstract and introduction accurately reflect the paper's contributions and scope?
		\answerYes{}
		\item Did you describe the limitations of your work?
		\answerYes{See Sections~\ref{sec:intro} and \ref{sec:conclusion}.}
		\item Did you discuss any potential negative societal impacts of your work?
		\answerNA{Our work
			is theoretical and we do not see any potential negative societal impacts.}
		\item Have you read the ethics review guidelines and ensured that your paper conforms to them?
		\answerYes{}
	\end{enumerate}

	\item If you are including theoretical results...
	\begin{enumerate}
		\item Did you state the full set of assumptions of all theoretical results?
		\answerYes{}
		\item Did you include complete proofs of all theoretical results?
		\answerYes{See the appendix.}
	\end{enumerate}

	\item If you ran experiments...
	\begin{enumerate}
		\item Did you include the code, data, and instructions needed to reproduce the main experimental results (either in the supplemental material or as a URL)?
		\answerNA{}
		\item Did you specify all the training details (e.g., data splits, hyper-parameters, how they were chosen)?
		\answerNA{}
		\item Did you report error bars (e.g., with respect to the random seed after running experiments multiple times)?
		\answerNA{}
		\item Did you include the total amount of compute and the type of resources used (e.g., type of GPUs, internal cluster, or cloud provider)?
		\answerNA{}
	\end{enumerate}

	\item If you are using existing assets (e.g., code, data, models) or curating/releasing new assets...
	\begin{enumerate}
		\item If your work uses existing assets, did you cite the creators?
		\answerNA{}
		\item Did you mention the license of the assets?
		\answerNA{}
		\item Did you include any new assets either in the supplemental material or as a URL?
		\answerNA{}
		\item Did you discuss whether and how consent was obtained from people whose data you're using/curating?
		\answerNA{}
		\item Did you discuss whether the data you are using/curating contains personally identifiable information or offensive content?
		\answerNA{}
	\end{enumerate}

	\item If you used crowdsourcing or conducted research with human subjects...
	\begin{enumerate}
		\item Did you include the full text of instructions given to participants and screenshots, if applicable?
		\answerNA{}
		\item Did you describe any potential participant risks, with links to Institutional Review Board (IRB) approvals, if applicable?
		\answerNA{}
		\item Did you include the estimated hourly wage paid to participants and the total amount spent on participant compensation?
		\answerNA{}
	\end{enumerate}

\end{enumerate}
\fi

%%%%%%%%%%%%5
\iftableofcontents
\cleardoublepage
\renewcommand{\contentsname}{Contents of main article and appendix}
\tableofcontents
\fi

%%%%%%%%%%%%%%%%%%%%%%%%%%%%%%%%%%%%%%%%%%%%%%%%%%%
%%%%%%%%%%%%%%%%%%%%%%%%%%%%%%%%%%%%%%%%%%%%%%%%%%%%%%%%%%%%
%%%%%%%%%%% appendix
\ifappendix
\cleardoublepage
\appendix

%%%%%%%%%%%%%%%%%%%%%%%%%%%%%%%%%%%%%
%%%%%%%%%%%%%%%%%%%%%%%%%%%%%%%%%%
\section{Proof of main theorem}
\label{sec:proof:main:thm}

In this section, we will prove the main theorem, Theorem~\ref{thm:main}, based on an auxiliary theorem, Theorem~\ref{thm:main:gap}, which will be proved in Section~\ref{sec:proof:thm:main:gap}. Notations throughout this paper are summarized in Section~\ref{sec:notation}. 

%%%%%%%%%%%%%%%%%%%%%%%%%%%%%%%%%%%%%%%%
\subsection{Notations}
\label{sec:notation}

Let us summarize all basic notations used in this paper as follows.
\begin{itemize}
	\item Let $\R$, $\Q$, and $\Z$ denote the set of real numbers, rational numbers, and integers, respectively.
	
	\item Let $\N$ and $\N^+$ denote the set of natural numbers and positive natural numbers, respectively.  That is,
	$\N^+=\{1,2,3,\cdots\}$ and $\N=\N^+\bigcup\{0\}$.
	
	\item For any $x\in \R$, let $\lfloor x\rfloor:=\max \{n: n\le x,\ n\in \Z\}$ and $\lceil x\rceil:=\min \{n: n\ge x,\ n\in \Z\}$.
	
	%$\one_{S}=\one_{\{x\in S\}}=\one_{x\in S}=\one_{S}(x)$
	\item Let $\one_{S}$ be the indicator (characteristic) function of a set $S$, i.e., $\one_{S}$ is equal to $1$ on $S$ and $0$ outside $S$.
	
	%\item Let $|S|$ denote the size of a set $S$, i.e., the number of all elements in $S$.
	
	\item The set difference of two sets $A$ and $B$ is denoted by $A\backslash B:=\{x:x\in A,\ x\notin B\}$. 
	
	\item Matrices are denoted by bold uppercase letters. 
	For instance,  $\bm{A}\in\mathbb{R}^{m\times n}$ is a real matrix of size $m\times n$ and $\bm{A}^T$ denotes the transpose of $\bm{A}$.  
	%Correspondingly, $\bm{A}(i,j)$ is the $(i,j)$-th entry of $\bm{A}$; $\bm{A}(:,j)$ is the $j$-th column of $\bm{A}$; $\bm{A}(i,:)$ is the $i$-th row of $\bm{A}$. 
	Vectors are denoted as bold lowercase letters. 
	For example, $\bm{v}=[v_1,\cdots,v_d]^T=\scalebox{0.85}{$\left[\begin{array}{c}
		v_1  \\
		\vdots \\
		v_d
	\end{array}\right]$}\in \R^d$ is a column vector. 
% 	%consisting of numbers $\{v_i\}_i$
% 	%with $\bm{v}(i)=v_i$ being the $i$-th element. 
% 	Besides, ``['' and ``]''  are used to  partition matrices (vectors) into blocks, e.g., $\bmA=\left[\begin{smallmatrix}\bmA_{11}&\bmA_{12}\\ \bmA_{21}&\bmA_{22}\end{smallmatrix}\right]$.

	\item For any $p\in [1,\infty)$, the $p$-norm (or $\ell^p$-norm) of a vector $\bmx=[x_1,x_2,\cdots,x_d]^T\in\R^d$ is defined by 
	\begin{equation*}
		\|\bmx\|_p=\|\bmx\|_{\ell^p}\coloneqq \big(|x_1|^p+|x_2|^p+\cdots+|x_d|^p\big)^{1/p}.
	\end{equation*}
	In the case of $p=\infty$, \begin{equation*}
		\|\bmx\|_{\infty}=\|\bmx\|_{\ell^\infty}\coloneqq \max\big\{|x_i|: i=1,2,\cdots,d\big\}.
	\end{equation*}
	
	\item By convention, $\sum_{j=n_1}^{n_2} a_j=0$
%	 and $\prod_{j=n_1}^{n_2} a_j=1$ 
	if $n_1>n_2$, no matter what $a_j$ is for each $j$.
	
	\item 
	Given any $K\in \N^+$ and $\delta\in (0, \tfrac{1}{K})$, define a trifling region  $\Omega([0,1]^d,K,\delta)$ of $[0,1]^d$ as 
	\begin{equation}
		\label{eq:triflingRegionDef}
		\Omega([0,1]^d,K,\delta)\coloneqq\bigcup_{j=1}^{d} \bigg\{\bmx=[x_1,x_2,\cdots,x_d]^T\black{\in [0,1]^d}: x_j\in \bigcup_{k=1}^{K-1}\big(\tfrac{k}{K}-\delta,\tfrac{k}{K}\big)\bigg\}.
	\end{equation}
	In particular, $\Omega([0,1]^d,K,\delta)=\emptyset$ if $K=1$. See Figure~\ref{fig:region} for two examples of trifling regions.
	
	\begin{figure}[htbp!]       
		\centering
		\begin{minipage}{0.805\textwidth}
			\centering
			\begin{subfigure}[b]{0.33\textwidth}
				\centering            
				\includegraphics[width=0.999\textwidth]{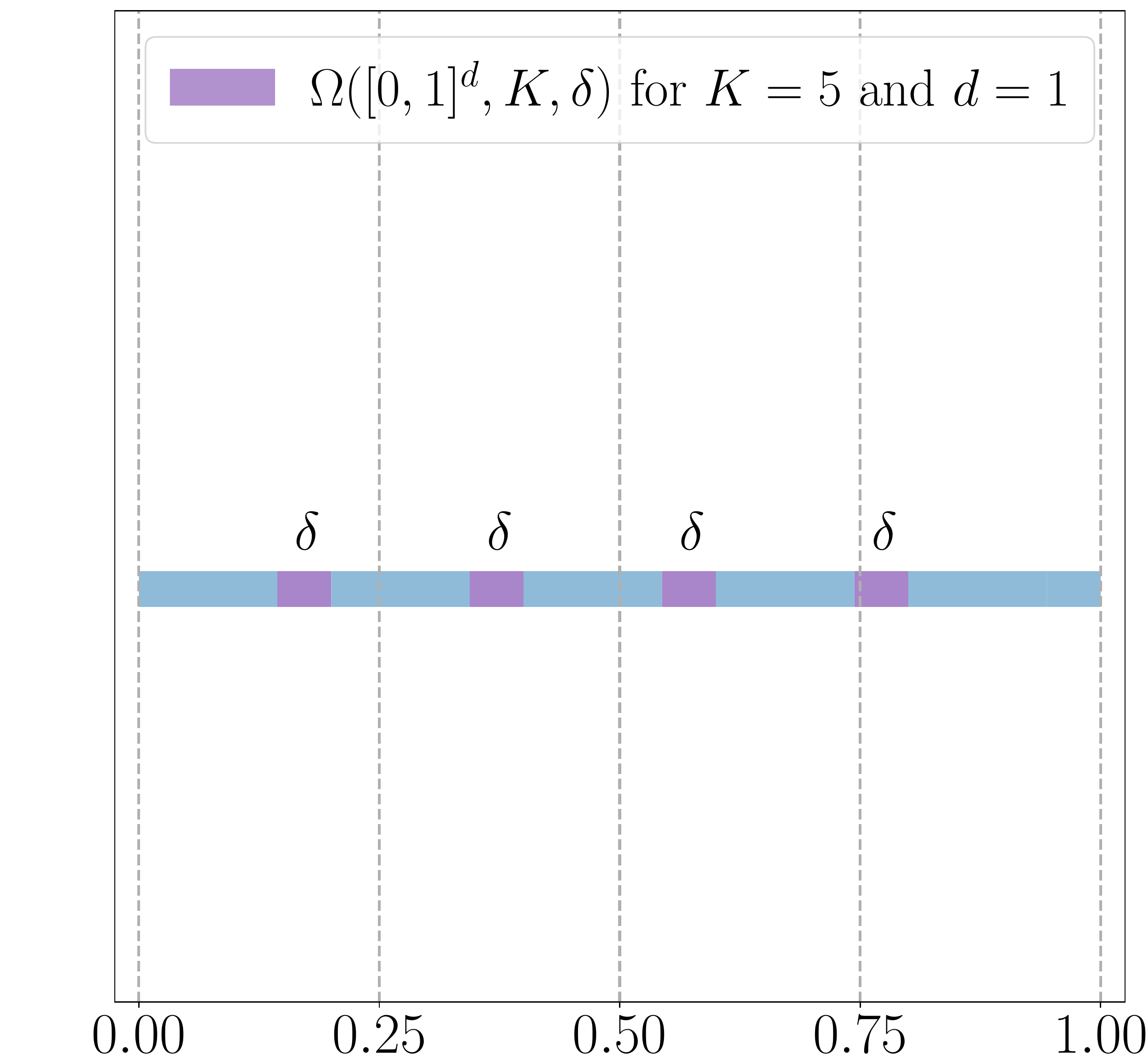}
				\subcaption{}
			\end{subfigure}
			\begin{minipage}{0.07\textwidth}
				\,
			\end{minipage}
			\begin{subfigure}[b]{0.33\textwidth}
				\centering            \includegraphics[width=0.999\textwidth]{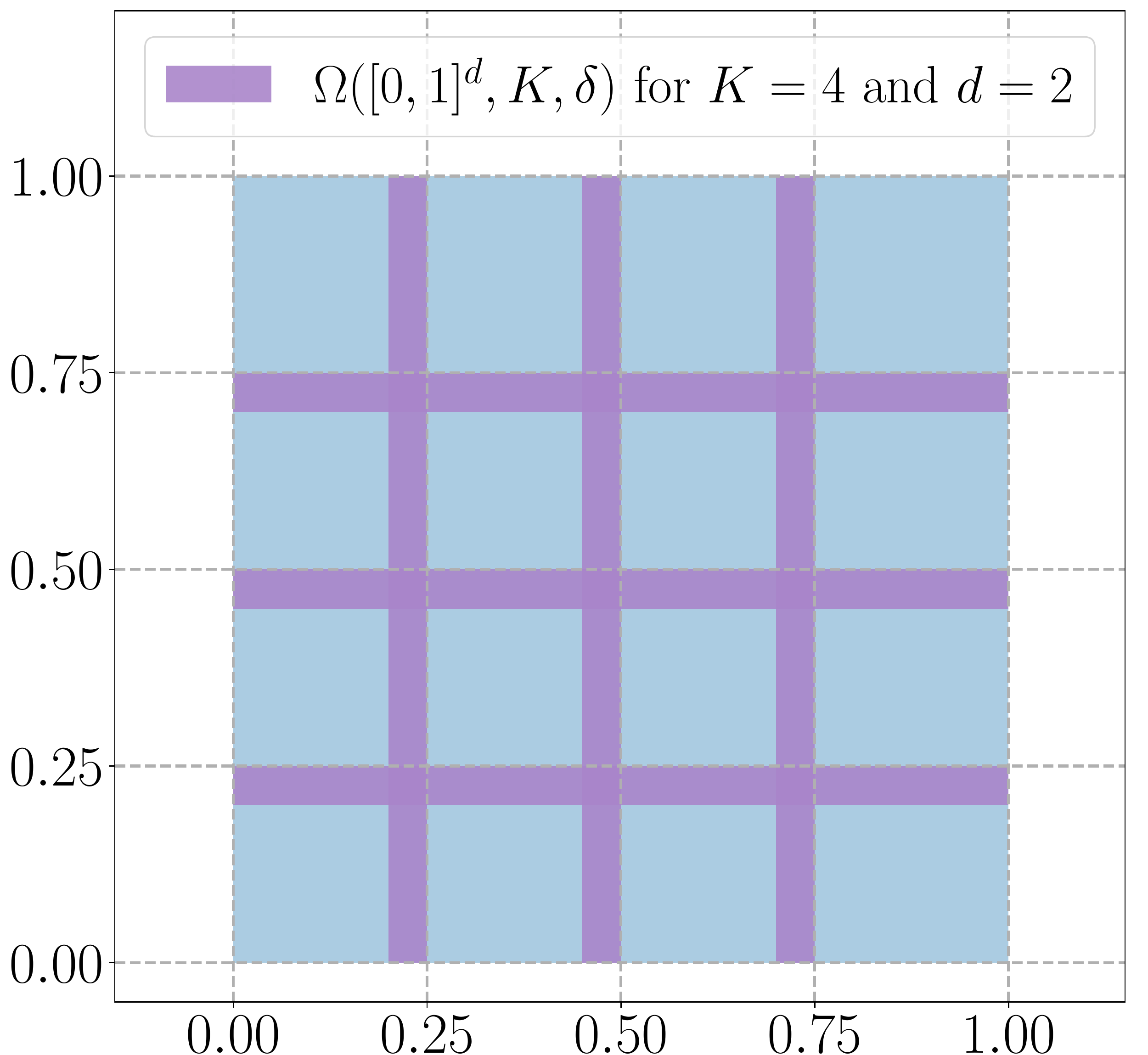}
				\subcaption{}
			\end{subfigure}
		\end{minipage}
		\caption{Two examples of trifling regions. (a)  $K=5,d=1$. (b) $K=4,d=2$.}
		\label{fig:region}
	\end{figure}
	
	%\item Let $\holder{\lambda}{\alpha}$ denote the space of  H{\"o}lder continuous functions on $[0,1]^d$ of order $\alpha\in (0,1]$ with  a H\"older constant $\lambda>0$. 
	
	\item For a continuous piecewise linear function $f(x)$, the $x$ values where the slope changes are typically called \textbf{breakpoints}. 
	
%	\item Let $\cpl(\R,n)$ denote the space that consists of all continuous piecewise linear functions with at most $n$ breakpoints on $\R$.
%	
	
	\item Let $\sigma:\R\to \R$ denote the rectified linear unit (ReLU), i.e. $\sigma(x)=\max\{0,x\}$ for any $x\in\R$. With a slight abuse of notation, we define $\sigma:\R^d\to \R^d$ as $\sigma(\bmx)=\left[
	\begin{array}{c}
		{\max\{0,x_1\}}  \\
		{\vdots} \\
		{\max\{0,x_d\}}
	\end{array}\right]$ for any $\bmx=[x_1,\cdots,x_d]^T\in \R^d$.
	
	\item Let $\nestnet_{s}\{n\}$ for $n,s\in \N^+$  denote the set of functions realized by 
	height-$s$ ReLU NestNets with as most $n$ parameters.

	\item A function $\phi$ realized by a ReLU network can be briefly described as follows:
	\begin{equation*}
		\begin{aligned}
			\bm{x}=\widetilde{\bm{h}}_0 
			\myto{2.2}^{\bm{W}_0,\ \bm{b}_0}_{\bmcalL_0} \bm{h}_1
			\myto{1.015}^{\sigma} \tilde{\bm{h}}_1 \ \cdots\ \myto{2.7}^{\bm{W}_{L-1},\ \bm{b}_{L-1}}_{\bmcalL_{L-1}} \bm{h}_L
			\myto{1.015}^{\sigma} \tilde{\bm{h}}_L
			\myto{2.2}^{\bm{W}_{L},\ \bm{b}_{L}}_{\bmcalL_L} \bm{h}_{L+1}=\phi(\bm{x}),
		\end{aligned}
	\end{equation*}
	where $\bm{W}_i\in \R^{N_{i+1}\times N_{i}}$ and $\bm{b}_i\in \R^{N_{i+1}}$ are the weight matrix and the bias vector in the $i$-th affine linear transformation $\bmcalL_i$, respectively, i.e., 
	\[\bm{h}_{i+1} =\bm{W}_i\cdot \tilde{\bm{h}}_{i} + \bm{b}_i\eqqcolon \bmcalL_i(\tilde{\bm{h}}_{i})\quad \tn{for $i=0,1,\cdots,L$,}\]  
	and
	\[
	\tilde{\bm{h}}_i=\sigma(\bm{h}_i)\quad \tn{for $i=1,2,\cdots,L$.}
	\]
	In particular, $\phi$ can be represented in a form of function compositions as follows
	\begin{equation*}
		\phi =\bmcalL_L\circ\sigma\circ
% 		\bmcalL_{L-1}\circ \sigma\circ
		\ \cdots \  \circ 
% 		\sigma\circ
		\bmcalL_1\circ\sigma\circ\bmcalL_0,
	\end{equation*}
	which has been illustrated in Figure~\ref{fig:ReLUeg}.
	
	\begin{figure}[htbp!]     
		\centering
		\centering            \includegraphics[width=0.72\textwidth]{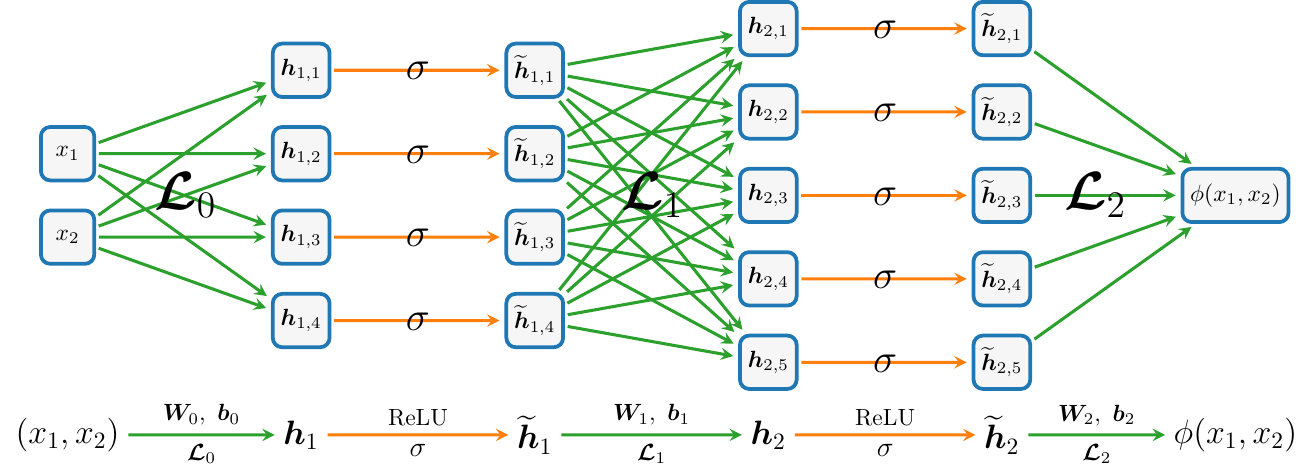}
		\caption{An example of a ReLU network of width $5$ and depth $2$. }
		\label{fig:ReLUeg}
	\end{figure}

	\item The expression ``a network of width $N$ and depth $L$'' means
	\begin{itemize}
		\item The number of neurons in each  \textbf{hidden} layer of this network (architecture) is no more than $N$.
		\item The number of \textbf{hidden} layers of this network (architecture) is no more than $L$.
	\end{itemize} 
\end{itemize}

\subsection{Detailed proof of Theorem~\ref{thm:main}}
\label{sec:proof:main}

The key point of proving Theorem~\ref{thm:main} is to construct a piecewise constant function to approximate the target continuous function. However, ReLU NestNets are unable to approximate piecewise constant functions well the continuity of ReLU NestNets. Thus, we introduce the trifling region  $\Omega([0,1]^d,K,\delta)$, defined in Equation~\eqref{eq:triflingRegionDef}, and use ReLU NestNets to implement piecewise constant functions outside the trifling region.
To simplify the proof of Theorem~\ref{thm:main}, we introduce an auxiliary theorem, Theorem~\ref{thm:main:gap} below.
It can be regarded as a weaker variant of Theorem~\ref{thm:main}, ignoring the approximation in the trifling region. 
\begin{theorem}
	\label{thm:main:gap}
	Given  a continuous function $f\in C([0,1]^d)$, for any $n,s\in \N^+$,
	there exists $\phi\in \nestnet_{s}\big\{355d^2(s+7)^2(2n+1)\big\}$
	%	a function $\phi$ realized by a  height-$s$ ReLU NestNet with at most $3633d^2(s+2)^4(2n+1)$
	%	parameters
	such that $ \|\phi\|_{L^\infty(\R^d)}\le |f(\bmzero)|+ \omega_f(\sqrt{d})$ and 
	\begin{equation*}
		|\phi(\bmx)-f(\bmx)|\le 6\sqrt{d}\,\omega_f\big(n^{-(s+1)/d}\big)\quad \tn{for any $\bmx\in [0,1]^d\backslash\Omega([0,1]^d,K,\delta)$},
	\end{equation*}
	where $K=\lfloor n^{(s+1)/d}\rfloor$ and $\delta$ is an arbitrary number in $(0,\tfrac{1}{3K}]$.
\end{theorem}

The proof of Theorem~\ref{thm:main:gap} can be found in Section~\ref{sec:proof:thm:main:gap}. 
By assuming Theorem~\ref{thm:main:gap} is true, we can easily prove Theorem~\ref{thm:main} for the case $p\in[1,\infty)$. 
To prove Theorem~\ref{thm:main} for the case  $p=\infty$, we need to control the approximation error in the trifling region. 
To this intent, we introduce a theorem to handle the approximation inside the trifling region.

\begin{theorem}[Lemma~$3.11$ of \cite{shijun:thesis} or Lemma~$3.4$ of \cite{shijun:3}]
	\label{thm:gap:handling}
	Given any $\varepsilon>0$, $K\in \N^+$, and $\delta\in (0, \tfrac{1}{3K}]$,
	assume $f\in C([0,1]^d)$ and $g:\R^d\to \R$ is a general function with
	\begin{equation*}
		|g(\bmx)-f(\bmx)|\le \varepsilon\quad \tn{for any $\bmx\in [0,1]^d\backslash \Omega([0,1]^d,K,\delta)$.}
	\end{equation*}
	Then
	\begin{equation*}
		|\phi(\bmx)-f(\bmx)|\le \varepsilon+d\cdot\omega_f(\delta)\quad \tn{for any $\bmx\in [0,1]^d$,}
	\end{equation*}
	where $\phi\coloneqq  \phi_d$ is defined by induction through $\phi_0\coloneqq g$ and
	\begin{equation*}
		%	\label{eq:phiInduction}
		\phi_{i+1}(\bmx)\coloneqq  \middleValue\big(\phi_{i}(\bmx-\delta\bme_{i+1}),\phi_{i}(\bmx),\phi_{i}(\bmx+\delta\bme_{i+1})\big)\quad \tn{for $i=0,1,\cdots,d-1$,}
	\end{equation*}
	where $\{\bme_i\}_{i=1}^d$ is the standard basis in $\mathbb{R}^d$ and $\middleValue(\cdot,\cdot,\cdot)$ is the function returning the middle value of three inputs. 	 
\end{theorem}

Now, let we prove Theorem~\ref{thm:main} by assuming Theorem~\ref{thm:main:gap} is true, the proof of which can be found  in Section~\ref{sec:proof:thm:main:gap}.

\begin{proof}[Proof of Theorem~\ref{thm:main}]		
	We may assume $f$ is not a constant function since it is  a trivial case. Then $\omega_f(r)>0$ for any $r>0$. 
	Let us first consider the case $p\in [1,\infty)$. Set 
	$K=\lfloor n^{{(s+1)}/d}\rfloor$ and choose a sufficiently small $\delta\in (0,\tfrac{1}{3K}]$ such that
	\begin{equation*}
		\begin{split}
			Kd\delta\big( \black{2}|f(\bmzero)|+\black{2}\omega_f(\sqrt{d})\big)^p
			&=
			\lfloor n^{{(s+1)}/d}\rfloor
			d\delta \Big( 2|f(\bmzero)|+2\omega_f(\sqrt{d})\Big)^p\\
			&\le \Big(\omega_f(n^{-{(s+1)}/d})\Big)^p. 
		\end{split}
	\end{equation*}
	By Theorem~\ref{thm:main:gap}, there exists 
	\begin{equation*}
		\begin{split}
				\phi\in \nestnet_{s}\big\{355d^2(s+7)^2(2n+1)\big\}
				&\subseteq
				\nestnet_{s}\big\{355d^2(s+7)^2\cdot 2(n+1)\big\}\\
				&\subseteq \nestnet_{s}\big\{10^3 d^2(s+7)^2 (n+1)\big\}
		\end{split}
	\end{equation*}
	%	a function $\phi$ realized by a ReLU NestNet with at most
	%	$3633d^2(s+2)^4(2n+1)$ parameters 
	such that 
	\black{$\|\phi\|_{L^\infty(\R^d)}\le |f(\bmzero)|+\omega_f(\sqrt{d})$} and 
	\begin{equation*}
		|\phi(\bmx)-f(\bmx)|\le 6\sqrt{d}\,\omega_f\big(n^{-{(s+1)}/d}\big)\quad \tn{for any $\bmx\in [0,1]^d\backslash\Omega([0,1]^d,K,\delta)$}.
	\end{equation*}
	Since $\|f\|_{L^\infty([0,1]^d)}\le |f(\bmzero)|+\omega_f(\sqrt{d})$ and  the Lebesgue measure of $\Omega([0,1]^d,K,\delta)$ is bounded by $Kd\delta$,
	we have
	\begin{equation*}
		\begin{split}
			\|\phi-f\|_{L^p([0,1]^d)}^p
			&=\int_{\Omega([0,1]^d,K,\delta)}|		\phi(\bmx)-f(\bmx)|^p\tn{d}\bmx+\int_{[0,1]^d\backslash\Omega([0,1]^d,K,\delta)}|\phi(\bmx)-f(\bmx)|^p\tn{d}\bmx\\
			&\le Kd\delta\Big( 2|f(\bmzero)|+2\omega_f(\sqrt{d})\Big)^p+ \Big(6\sqrt{d}\,\omega_f(n^{-(s+1)/d})\Big)^p\\
			&\le \Big(\omega_f\big(n^{-(s+1)/d}\big)\Big)^p
			+ \Big(6\sqrt{d}\,\omega_f(n^{-(s+1)/d})\Big)^p
			\le \Big(7\sqrt{d}\,\omega_f(n^{-(s+1)/d})\Big)^p.
		\end{split}
	\end{equation*}
	Hence, we have $\|\phi-f\|_{L^p([0,1]^d)}\le 7\sqrt{d}\,\omega_f\big(n^{-(s+1)/d}\big)$.
	
	Next, let us discuss the case $p=\infty$. Set $K=\lfloor n^{(s+1)/d}\rfloor$ and choose a sufficiently small $\delta\in(0,\tfrac{1}{3K}]$ such that 
	\begin{equation*}
		d\cdot \omega_f(\delta)\le \omega_f\big(n^{-(s+1)/d}\big).
	\end{equation*}	
	By Theorem~\ref{thm:main:gap}, 
		\begin{equation*}
		\phi_0\in \nestnet_{s}\big\{355d^2(s+7)^2(2n+1)\big\}
	\end{equation*}
%	there exists a function $\phi_0$ realized  by a height-$s$  ReLU NestNet with at most $3633d^2(s+1)^4(2n+1)$ 
	such that
	\begin{equation*}
		|\phi_0(\bmx)-f(\bmx)|\le 6\sqrt{d}\,\omega_f\big(n^{-(s+1)/d}\big)\quad 	\tn{for any $\bmx\in [0,1]^d\backslash\Omega([0,1]^d,K,\delta)$}.
	\end{equation*}
	By Theorem~\ref{thm:gap:handling}  with $g=\phi_0$ and $\varepsilon=6\sqrt{d}\,\omega_f\big(n^{-(s+1)/d}\big)$ therein, we have
	\begin{equation*}
		|\phi(\bmx)-f(\bmx)|\le \varepsilon+d\cdot \omega_f(\delta)\le  7\sqrt{d}\,\omega_f\big(n^{-(s+1)/d}\big)\quad \tn{for any $\bmx\in [0,1]^d$},
	\end{equation*}
	where $\phi\coloneqq  \phi_d$ is defined by induction through 
	\begin{equation*}
		\phi_{i+1}(\bmx)\coloneqq  \middleValue\big(\phi_{i}(\bmx-\delta\bme_{i+1}),\phi_{i}(\bmx),\phi_{i}(\bmx+\delta\bme_{i+1})\big)\quad \tn{for $i=0,1,\cdots,d-1$,}
	\end{equation*}
	where $\{\bme_i\}_{i=1}^d$ is the standard basis in $\mathbb{R}^d$ and $\middleValue(\cdot,\cdot,\cdot)$ is the function returning the middle value of three inputs.

	It remains to estimate the number of parameters in the NestNet realizing $\phi=\phi_d$.
	By Lemma~$3.1$ of \cite{shijun:5}, $\middleValue(\cdot,\cdot,\cdot)$ can be realized by a ReLU network of width $14$ and depth $2$, and hence with at most $14\times(14+1)\times(2+1)=630$ parameters.
	
	By defining a vector-valued function $\bmPhi_0:\R^d\to \R^3$ as \[\bmPhi_0(\bmx)\coloneqq\big[\phi_{0}(\bmx-\delta\bme_{1}),\,\phi_{0}(\bmx),\,\phi_{0}(\bmx+\delta\bme_{1})\big]^T\quad \tn{for any $\bmx\in \R^d$,}\]
	we have $\bmPhi_0\in \nestnet_{s}\big\{3^2\big(355d^2(s+7)^2(2n+1)\big)\big\}$, implying
	\begin{equation*}
		\begin{split}
			\phi_1=\middleValue(\cdot,\cdot,\cdot)\circ\bmPhi_0
			&\in 
			\nestnet_{s}\Big\{630+3^2\Big(355d^2(s+7)^2(2n+1)\Big)\Big\}\\
			&\subseteq \nestnet_{s}\Big\{10\Big(355d^2(s+7)^2(2n+1)\Big)\Big\}.
		\end{split}
	\end{equation*} 
	Similarly, we have
	\begin{equation*}
		\begin{split}
			\phi=\phi_d
			\in \nestnet_{s}\Big\{10^d\Big(355d^2(s+7)^2(2n+1)\Big)\Big\}
			&\subseteq
			\nestnet_{s}\Big\{10^d\Big(355d^2(s+7)^2\cdot 2(n+1)\Big)\Big\}\\
			&\subseteq
			\nestnet_{s}\Big\{10^{d+3}d^2(s+7)^2 (n+1)\Big\}.
		\end{split}
	\end{equation*} 
	Thus, we finish the proof of Theorem~\ref{thm:main}.
	
	%	Recall that  
	%	by Lemma~\ref{lem:approxMid}.
	%	Therefore, $\phi_1=\min(\cdot,\cdot,\cdot)\circ\bmPhi_0$ can be implemented by a ReLU FNN of width $\max\{3N,14\}\le 3(N+4)$ and depth $L+2$.
	%	Similarly, $\phi=\phi_d$ can be implemented by a ReLU FNN of width $3^d(N+4)$ and depth $L+2d$.	
\end{proof}

%%%%%%%%%%%%%%%%%%%%%%%%%%%%%%%%%%%%%
%%%%%%%%%%%%%%%%%%%%%%%%%%%%%%%
\section{Proof of auxiliary theorem}
\label{sec:proof:thm:main:gap}

We will prove the auxiliary theorem, Theorem~\ref{thm:main:gap}, in this section. We first present the key ideas in Section~\ref{sec:ideas:thm:main:gap}. 
%Next, we discuss the construction of the final network in Section~\ref{sec:final:network:main:gap}.
Next, the detailed proof is presented in Section~\ref{sec:detaied:proof:main:gap}, based on two propositions in Section~\ref{sec:ideas:thm:main:gap}, 
the proofs of which can be found in Sections~\ref{sec:proof:prop:floor:approx} and \ref{sec:proof:prop:point:fitting}.

\subsection{Key ideas of proving Theorem~\ref{thm:main:gap}}
\label{sec:ideas:thm:main:gap}

%We will show that an almost piecewise constant function $\phi$ implemented by a ReLU network is enough to achieve the desired approximation rate in Theorem~\ref{thm:main}. 
Our goal is to construct an almost piecewise constant function realized by a ReLU NestNet
to approximate the target function $f\in C([0,1]^d)$ well.
The construction can be divided into three main steps.
\begin{enumerate}[1.]
	\item First, we divide $[0,1]^d$  into a union of ``important'' cubes $\{Q_\bmbeta\}_{\bmbeta\in \{0,1,\cdots,K-1\}^d}$ and the trifling  region $\Omega([0,1]^d,K,\delta)$, where $K=\calO(n^{(s+1)/d})$. Each $Q_\bmbeta$ is associated with a representative $\bm{x}_\bmbeta\in Q_\bmbeta$ for each vector index  $\bmbeta$.  See Figure~\ref{fig:Q+TR} for illustrations. 
	
	\item Next, we design a vector-valued function $\bmPhi_1(\bmx)$ to map the whole cube $Q_\bmbeta$ to its index $\bmbeta$ for each $\bmbeta$.
	Here, $\bmPhi_1$ can be defined/constructed via \[\bmPhi_1(\bmx)=\big[\phi_1(x_1),\,\phi_1(x_2),\,\cdots,\,\phi_1(x_d)\big]^T,\]
	where each one-dimensional function $\phi_1$  is a step function outside the trifling region and hence can be  realized by a ReLU NestNet. 
	
	\item  The aim of the final step is essentially to solve a point fitting problem. We will construct a function $\phi_2$ realized by a ReLU NestNet to map $\bmbeta$ approximately to $f(\bmx_\bmbeta)$ for each $\bmbeta$. Then we have 
	\begin{equation*}
		\phi_2\circ\bmPhi_1(\bmx)=\phi_2(\bmbeta)\approx f(\bmx_\bmbeta)\approx f(\bmx)
		\quad \tn{for any $\bmx\in Q_\bmbeta$ and each $\bmbeta$,}
	\end{equation*}
	 implying 
	\begin{equation*}
		\phi\coloneqq \phi_2\circ\bmPhi_1 \approx  f \quad \tn{on $[0,1]^d\backslash\Omega([0,1]^d,K,\delta)$}.
	\end{equation*} 
	We remark that, in the construction of $\phi_2$, we only need to care about the values of $\phi_2$ sampled inside the set $\{0,1,\cdots,K-1\}^d$, which is a key point to ease the design of a ReLU NestNet to realize $\phi_2$ as we shall see later. 
\end{enumerate}

\begin{figure}[htbp!]       
	\centering
	\includegraphics[width=0.75\textwidth]{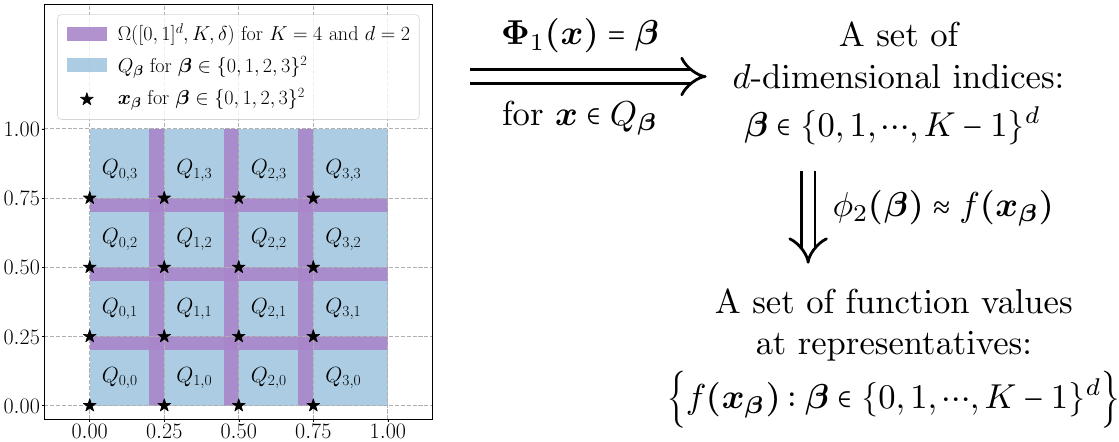}
	\caption{An illustration of the ideas of constructing the desired function $\phi=\phi_2\circ\bmPhi_1$. 
		Note that $\phi\approx f$ outside the trifling region
		  since $\phi(\bmx)=\phi_2\circ\bmPhi_1(\bmx)=\phi_2(\bmbeta)\approx f(\bmx_\bmbeta)\approx f(\bmx)$ for any $\bmx\in Q_\bmbeta$ and each $\bmbeta\in\{0,1,\cdots,K-1\}^d$.}
	\label{fig:idea}
\end{figure}

Observe that in Figure~\ref{fig:idea}, we have 
\begin{equation*}
	\phi(\bmx)=\phi_2\circ\bmPhi_1(\bmx)=\phi_2(\bmbeta)\mathop{\approx}\limits^{\scrE_1} f(\bmx_\bmbeta)\mathop{\approx}\limits^{\scrE_2} f(\bmx)
\end{equation*}
for any $\bmx\in Q_\bmbeta$ and each $\bmbeta\in \{0,1,\cdots,K-1\}^d$. That means $\phi-f$ is controlled by $\scrE_1+\scrE_2$ on $[0,1]^d\backslash \Omega([0,1]^d,K,\delta)$. 
Since $\|\bmx-\bmx_\bmbeta\|_2\le \sqrt{d}/K$ for any $\bmx\in Q_\bmbeta$ and each $\bmbeta$,  $\scrE_2$ is bounded by $\omega_f(\sqrt{d}/K)$. As we shall see later, $\scrE_1$ can be bounded by $\calO\big(\omega_f(\sqrt{d}/K)\big)$ by applying Proposition~\ref{prop:point:fitting}. Therefore, $\phi-f$ is controlled by $\calO\big(\omega_f(\sqrt{d}/K)\big)$ outside the trifling region, from which we deduce the desired approximation error since $K=\calO(n^{-(s+1)/d})$.

Finally, we introduce two propositions to simplify the constructions of $\bmPhi_1$ and $\phi_2$ mentioned above.
We first show how to  construct a ReLU network  to implement a one-dimensional step function $\phi_1$ in Proposition~\ref{prop:floor:approx} below. Then $\bmPhi_1$ can be defined via \[\bmPhi_1(\bmx)\coloneqq\big[\phi_1(x_1),\,\phi_1(x_2),\,\cdots,\,\phi_1(x_d)\big]^T\quad \tn{for any $\bmx=[x_1,x_2,\cdots,x_d]^T\in \R^d$.}\]

\begin{proposition}\label{prop:floor:approx}
	Given any $n,r\in \N^+$, $\delta\in (0,1)$, and $J\in\N^+$ with $J\le 2^{n^r}$, there exists
	$\phi\in \nestnet_{r}\big\{36(r+7)n\big\}$
%	  a function $\phi$ realized by a ReLU network with at most $256r^3n$ parameters 
	  such that  
	\begin{equation*}
		\phi(x)=\lfloor x\rfloor \quad \tn{for any $x\in \bigcup_{j=0}^{J-1}[j,\,j+1-\delta]$}
	\end{equation*}
	and 
	\begin{equation*}
		\phi(x)=J \quad \tn{for any $x\in [J,\,J+1]$.}
	\end{equation*}
\end{proposition}

%\begin{proposition}\label{prop:floor:approx}
%	Given any $n,r\in \N^+$, $\delta\in (0,1)$, and $J\le n^r$, there exist  a function $\phi$ realized by a ReLU network with at most $256r^3n$ parameters such that  
%	\begin{equation*}
%		\phi(x)=\lfloor x\rfloor \quad \tn{for any $x\in \bigcup_{j=0}^{J-2}[j,\,j+1-\delta]$}
%	\end{equation*}
%	and 
%	\begin{equation*}
%		\phi(x)=J-1 \quad \tn{for any $x\in [J-1,\,J]$.}
%	\end{equation*}
%\end{proposition}

The construction of $\phi_2$ is mainly based on Proposition~\ref{prop:point:fitting} below, whose proof relies on 
the bit extraction technique proposed in   \cite{Bartlett98almostlinear}. As we shall see later, some pre-processing is necessary for meeting the requirements of applying Proposition~\ref{prop:point:fitting} to construct $\phi_2$.

\begin{proposition}
	\label{prop:point:fitting}
	Given any $\varepsilon>0$ and $n,s\in \N^+$,
	assume   
	${y}_j\ge 0$ for $j=0,1,\cdots,J-1$ are samples  with $J\le n^{s+1}$ and
	\[|y_{j}-y_{j-1}|\le \varepsilon\quad \tn{for $j=1,2,\cdots,J-1.$}\] 
	Then there exists 
	$\phi\in \nestnet_{s}\big\{350(s+7)^2(n+1)\big\}$
%	a function $\phi$ realized by a height-$s$ ReLU NestNet with at most $3567(s+2)^4(n+1)$ parameters 
such that
	\begin{enumerate}[(i)]
		\item $|\phi(j)-{y}_j|\le \varepsilon$ for $j=0,1,\cdots,J-1$.
		\item $0\le \phi(x)\le  \max\{y_{j}:j=0,1,\cdots,J-1\}$ for any $x\in\R$.
	\end{enumerate}
\end{proposition}

The proofs of these two propositions can be found in Sections~\ref{sec:proof:prop:floor:approx} and \ref{sec:proof:prop:point:fitting}.
We will give the detailed proof of Theorem~\ref{thm:main:gap} in Section~\ref{sec:detaied:proof:main:gap}.

%%%%%%%%%%%%%%%%%%%%%%%%%%%%%%%%%%%%%%%%%%%
%%%%%%%%%%%%%%%%%%%%%%%%%%%%%%%%%%%%%%%%
\subsection{Detailed proof of Theorem~\ref{thm:main:gap}}
\label{sec:detaied:proof:main:gap}

We essentially construct an almost piecewise constant function realized by a ReLU NestNet with at most $\calO(n)$ parameters to approximate $f$. \black{We may assume $f$ is not a constant function since it is a trivial case. Then $\omega_f(r)>0$ for any $r>0$.}
It is clear that $|f(\bm{x})-f(\bmzero)|\le \omega_f(\sqrt{d})$ for any ${\bm{x}}\in [0,1]^d$. By defining $\tildef\coloneqq  f-f(\bmzero)+\omega_f(\sqrt{d})$, we have
$\omega_\tildef(r)=\omega_f(r)$ for any $r\ge 0$ and  $0\le \tildef (\bm{x}) \le 2\omega_f(\sqrt{d})$ for any ${\bm{x}}\in [0,1]^d$. 

Set $K=\lfloor n^{(s+1)/d}\rfloor$ and let $\delta$ be an arbitrary number in $(0,\tfrac{1}{3K}]$.
The proof can be divided into four main steps as follows:
\begin{enumerate}
	\item %Normalize $f$ as $\tildef$, d
	Divide $[0,1]^d$ into a union of sub-cubes $\{Q_{\bm{\beta}}\}_{\bm{\beta}\in \{0,1,\cdots,K-1\}^d}$ and the trifling region $\Omega([0,1]^d,K,\delta)$, and denote $\bmx_\bmbeta$ as the vertex of $Q_\bmbeta$ with minimum $\|\cdot\|_1$ norm.
	
	\item Construct a sub-network based on Proposition~\ref{prop:floor:approx} to implement a vector function $\bmPhi_1$ projecting the whole cube $ Q_\bmbeta$ to the $d$-dimensional index $\bmbeta$ for each $\bmbeta$, i.e., $\bmPhi_1(\bmx)=\bmbeta$ for all $\bmx\in Q_\bmbeta$.
	
	\item Construct a sub-network to implement a  function $\phi_2$  mapping the index $\bmbeta$ approximately to $\tildef(\bmx_\bmbeta)$. This core step can be further divided into three sub-steps:	
	\begin{enumerate}
		\item[3.1.] Construct a sub-network to implement $\psi_1$ bijectively mapping the index set $\{0,1,\cdots,K-1\}^d$ to an auxiliary set $\mathcal{A}_1\subseteq \big\{\tfrac{j}{2K^d}:j=0,1,\cdots,2K^d\big\}$ defined later. See Figure~\ref{fig:g+A12} for an illustration.
		
		\item[3.2.] Determine a continuous piecewise linear function $g$ with a set of breakpoints $\calA_1\cup\calA_2\cup\{1\}$, where $\calA_2\in \big\{\tfrac{j}{2K^d}:j=0,1,\cdots,2K^d\big\}$ is a set defined later.  Moreover, $g$ should satisfy two conditions: 1) the values of $g$ at breakpoints in $\calA_1$ is given based on $\{\tildef(\bmx_\bmbeta)\}_\bmbeta$, i.e., $g\circ \psi_1(\bmbeta)=\tildef(\bmx_\bmbeta)$; 2) the values of $g$ at breakpoints in $\calA_2\cup\{1\}$ is defined to reduce the variation of $g$, which is necessary for applying Proposition~\ref{prop:point:fitting}.
		
		%		such that $\tildef$ and $g\circ \phi_2 \circ \bmPhi_1$ have the same value at the elements of $\{\tfrac{k}{K}:k=0,1,\cdots,K-1\}^d$, i.e. $\tildef\approx g\circ \phi_2\circ \bmPhi_1$;
		
		\item[3.3.] Apply Proposition~\ref{prop:point:fitting} to construct a sub-network to implement a function $\psi_2$ approximating $g$ well on $\calA_1\cup\calA_2\cup\{1\}$. Then the desired function $\phi_2$ is given by  $\phi_2=\psi_2\circ\psi_1$ satisfying $\phi_2(\bmbeta)=\psi_2\circ\psi_1(\bmbeta)\approx g\circ\psi_1(\bmbeta)=\tildef(\bmx_\bmbeta)$.
		%		based on Proposition~\ref{prop:point:fitting} to implement a function  $\psi_3$ mapping $\mathcal{A}_1$ approximately to $\big\{\tildef(\bmx):\bmx\in \{\tfrac{k}{K}:k=0,1,\cdots,K-1\}^d\big\}$ such that $g\approx \phi_3$ on $\mathcal{A}_1$;
	\end{enumerate}
	
	\item Construct the final  network to implement the desired function $\phi$ via $\phi= \phi_2 \circ \bmPhi_1 +f(\bmzero)-\omega_f(\sqrt{d})$. Then we have
	$\phi_2 \circ \bmPhi_1(\bmx)=\phi_2(\bmbeta)\approx \tildef(\bmx_\bmbeta)\approx \tildef(\bmx)$ for any $\bmx\in Q_\bmbeta$ and $\bmbeta\in \{0,1,\cdots,K-1\}^d$, implying $\phi(\bmx)= \phi_2 \circ \bmPhi_1(\bmx) +f(\bmzero)-\omega_f(\sqrt{d})
	\approx \tildef(\bmx) +f(\bmzero)-\omega_f(\sqrt{d})= f(\bmx)$.
\end{enumerate}

\begin{figure}[htbp!]
	\centering
	\includegraphics[width=0.95\textwidth]{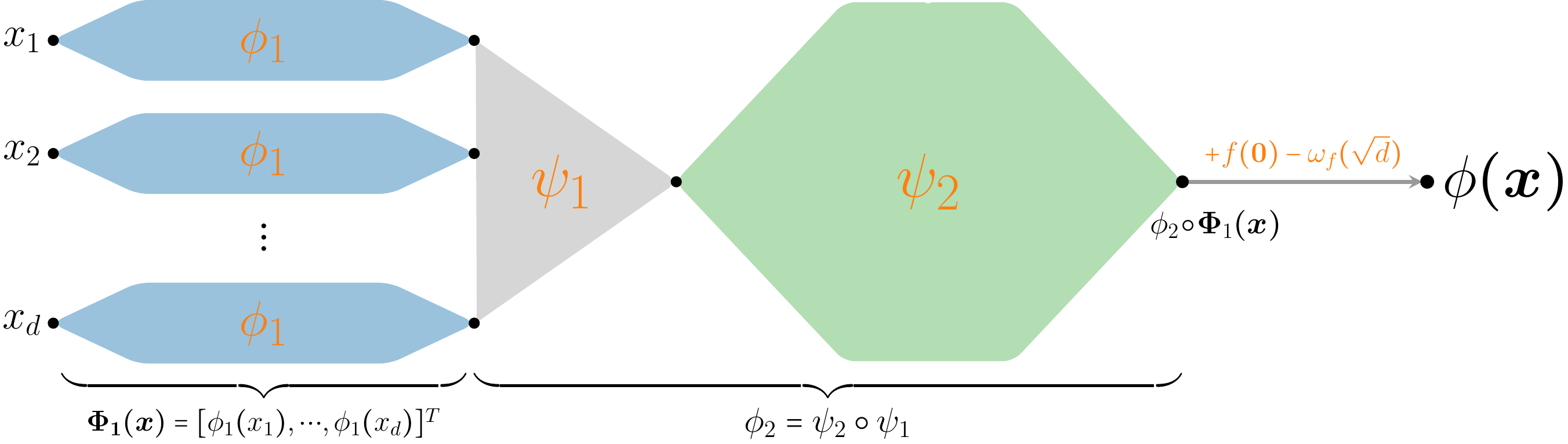}
	\caption{An illustration of the NestNet architecture realizing $\phi=\phi_2\circ\bmPhi_1+f(\bmzero)-\omega_f(\sqrt{d})$. Here, $\phi_1$ is implemented via Proposition~\ref{prop:floor:approx}; $\psi_1:\R^d\to \R$ is an affine linear function;  $\psi_2$ is implemented via Proposition~\ref{prop:point:fitting}.}
	\label{fig:final:net}
\end{figure}

See Figure~\ref{fig:final:net} for an illustration of the NestNet architecture realizing $\phi=\phi_2\circ\bmPhi_1+f(\bmzero)-\omega_f(\sqrt{d})$.   
The details of the steps mentioned above can be found below.

\mystep{1}{Divide $[0,1]^d$ into  $\{Q_{\bm{\beta}}\}_{\bm{\beta}\in \{0,1,\cdots,K-1\}^d}$ and  $\Omega([0,1]^d,K,\delta)$.}

Define  $\bmx_\bmbeta \coloneqq \bmbeta/K$ and 
\[
Q_{\bm{\beta}}\coloneqq\Big\{{\bm{x}}= [x_1,x_2,\cdots,x_d]^T\in [0,1]^d:x_i\in[\tfrac{\beta_i}{K},\tfrac{\beta_i+1}{K}-\delta\cdot \one_{\{\beta_i\le K-2\}}], \quad i=1,2,\cdots,d\Big\}
\]
for each $d$-dimensional index  ${\bm{\beta}}= [\beta_1,\beta_2,\cdots,\beta_d]^T\in \{0,1,\cdots,K-1\}^d$. Recall that $\Omega([0,1]^d,K,\delta)$ is the trifling region defined in Equation~\eqref{eq:triflingRegionDef}. Apparently, $\bmx_\bmbeta=\bmbeta/K$ is the vertex of $Q_\bmbeta$ with minimum $\|\cdot\|_1$ norm and 
\[[0,1]^d= \big(\cup_{\bm{\beta}\in \{0,1,\cdots,K-1\}^d}Q_{\bm{\beta}}\big)\bigcup \Omega([0,1]^d,K,\delta).\]
See Figure~\ref{fig:Q+TR} for illustrations.

\begin{figure}[htbp!]
	\centering
	\begin{minipage}{0.8\textwidth}
		\centering
		\begin{subfigure}[b]{0.435\textwidth}
			\centering
			\includegraphics[width=0.85\textwidth]{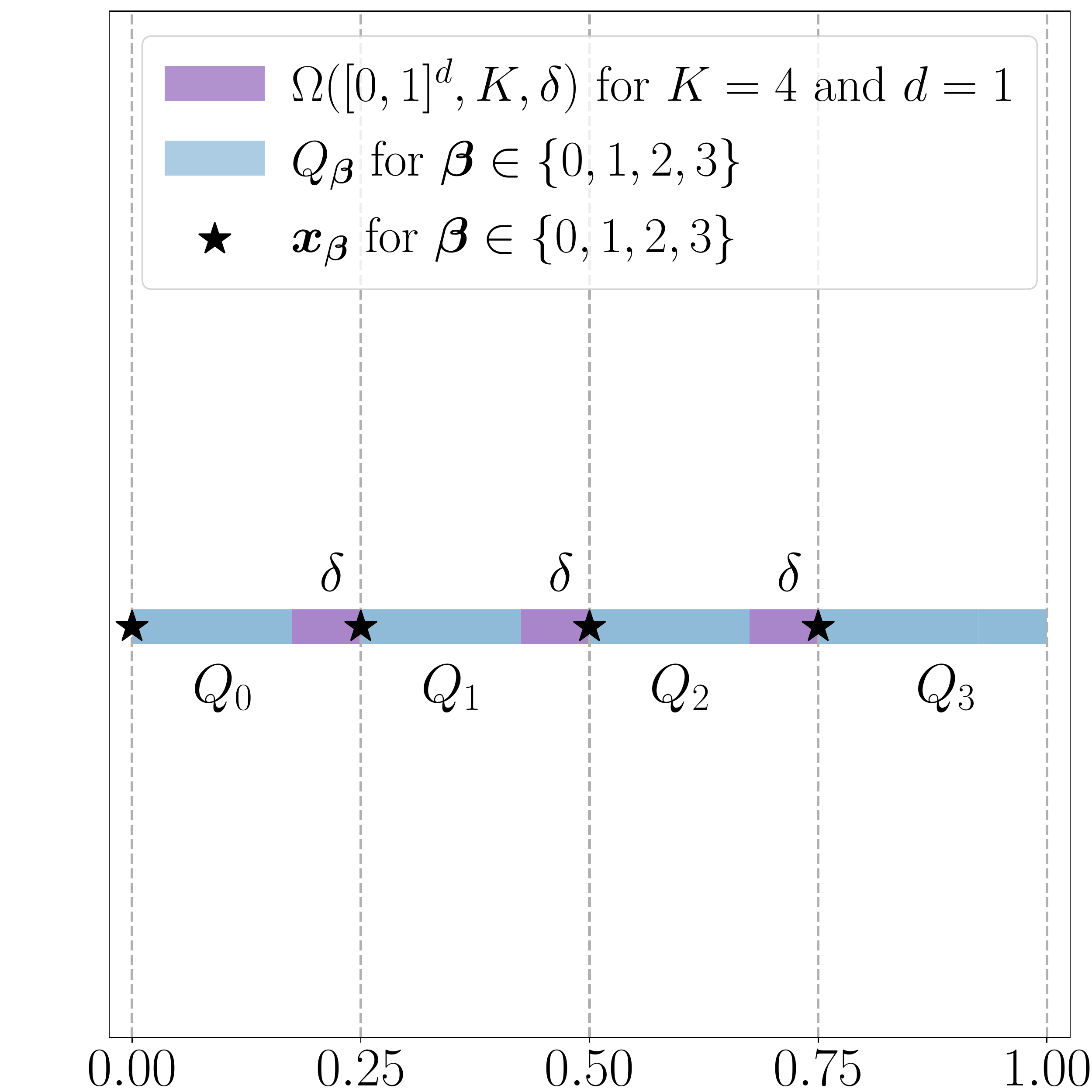}
			\subcaption{}
		\end{subfigure}
		\begin{minipage}{0.064\textwidth}
			\hspace{2pt}
		\end{minipage}
		\begin{subfigure}[b]{0.435\textwidth}
			\centering
			\includegraphics[width=0.85\textwidth]{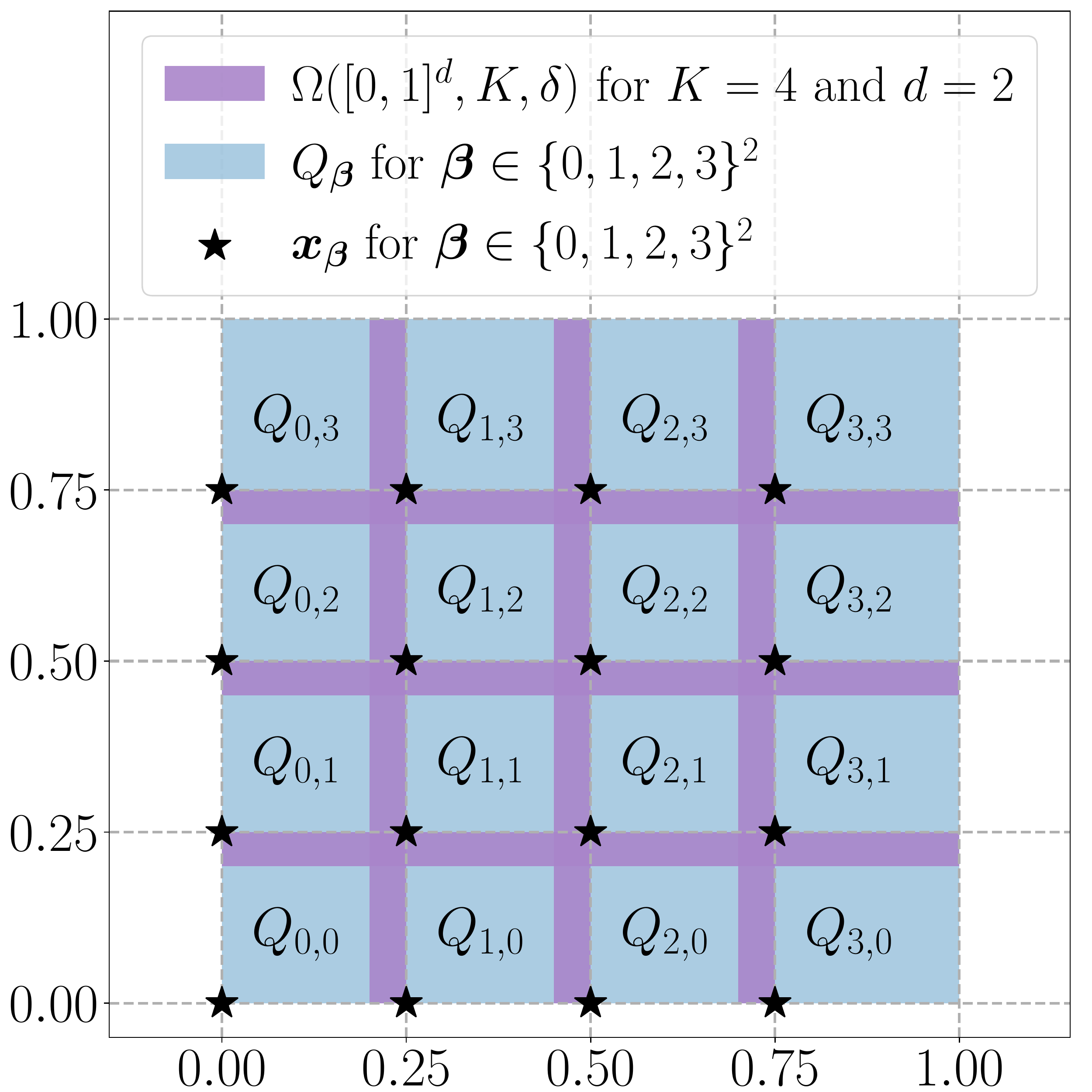}
			\subcaption{}
		\end{subfigure}
	\end{minipage}
	\caption{Illustrations of  $\Omega([0,1]^d,K,\delta)$,  $Q_\bmbeta$, and $\bmx_\bmbeta$ for $\bmbeta\in \{0,1,\cdots,K-1\}^d$. (a) $K=4$ and $d=1$. (b) $K=4$ and $d=2$. }
	\label{fig:Q+TR}
\end{figure}

\mystep{2}{Construct $\bmPhi_1$ mapping $\bmx\in Q_\bmbeta$ to  $\bmbeta$.}

Note that 
\begin{equation*}
	K-1=\lfloor n^{(s+1)/d}\rfloor-1\le  n^{s+1}\le
	\big(n^s\big)^2\le 4^{(n^s)}=  2^{2(n^{s})}
	\le 2^{(2n)^{s}}=2^{\tilden^{s}},
\end{equation*} 
where $\tilden=2n$.
By Proposition~\ref{prop:floor:approx} with $r=s$ and $J=K-1\le 2^{\tilden^{s}}=2^{\tilden^r}$ therein, there exists
\begin{equation*}
	\tildephi_1\in \nestnet_{s}\big\{36(s+7)\tilden\big\}
	=\nestnet_{s}\big\{36(s+7)(2n)\big\}
	=\nestnet_{s}\big\{72(s+7)n\big\}
\end{equation*}
%a function $\tildephi_1$ realized by a ReLU network with at most $256(s+1)^3n$ parameters
such that
\begin{equation*}
	\tildephi_1(x)=\lfloor x\rfloor \quad \tn{for any $x\in \bigcup_{k=0}^{K-2}[k,\,k+1-\tildedelta]$ with $\tildedelta=K\delta$}
\end{equation*}
and 
\begin{equation*}
	\tildephi_1(x)=K-1 \quad \tn{for any $x\in [K-1,\,K]$.}
\end{equation*}

Define $\phi_1(x)\coloneqq \tildephi_1(Kx)$ for any $x\in \R$. Then, we have $\phi_1\in \nestnet_{s}\big\{72(s+7)n\big\}$ and
\begin{equation*}
	\phi_1(x)=\black{k}\quad \tn{if $x\in [\tfrac{k}{K},\tfrac{k+1}{K}-\delta\cdot \one_{\{k\le K-2\}}]$\quad for $k=0,1,\cdots,K-1$.}
\end{equation*}    
It follows that $\phi_1(x_i)=\beta_i$ if $\bmx=[x_1,x_2,\cdots,x_d]^T\in Q_\bmbeta$ for each $\bmbeta=[\beta_1,\beta_2,\cdots,\beta_d]^T$.

By defining
\begin{equation*}
	\bmPhi_1(\bmx)\coloneqq \big[\phi_1(x_1),\,\phi_1(x_2),\,\cdots,\,\phi_1(x_d)\big]^T\quad \tn{for any } \bmx=[x_1,x_2,\cdots,x_d]^T\in \R^d,
\end{equation*}
we have 
\begin{equation}\label{eq:bmphi:output:beta}
	\bmPhi_1(\bmx)=\bmbeta \quad \tn{if} \ \bmx\in Q_\bmbeta\quad \tn{for each $\bmbeta\in \{0,1,\cdots,K-1\}^d$}.
\end{equation} 

%Moreover, $\bmPhi_1$ can be realized by a ReLU network with at most $d^2\big(256(s+2)^3n\big)=256d^2(s+2)^3n$ parameters.

\mystep{3}{Construct $\phi_2$ mapping $\bmbeta$ approximately to $\tildef(\bmx_\bmbeta)$.}

The construction of the sub-network implementing $\phi_2$ is essentially based on Proposition~\ref{prop:point:fitting}. 
To meet the requirements of applying Proposition~\ref{prop:point:fitting}, we first define two auxiliary sets $\calA_1$ and $\calA_2$ as 
\begin{equation*}
	\calA_1\coloneqq \Big\{\tfrac{i}{K^{d-1}}+\tfrac{k}{2K^d}:i=0,1,\cdots,K^{d-1}-1\tn{\quad and \quad} k=0,1,\cdots,K-1\Big\}
\end{equation*}
and 
\begin{equation*}
	\calA_2\coloneqq \Big\{\tfrac{i}{K^{d-1}}+\tfrac{K+k}{2K^d}:i=0,1,\cdots,K^{d-1}\black{-1}\tn{\quad and \quad}k=0,1,\cdots,K-1\Big\}.
\end{equation*}
Clearly, 
\begin{equation*}
	\calA_1\cup\calA_2\cup\{1\}=\big\{\tfrac{j}{2K^d}:j=0,1,\cdots,2K^d\big\}\quad \tn{and}\quad \calA_1\cap\calA_2=\emptyset.
\end{equation*} 
See Figure~\ref{fig:Q+TR} for an illustration of $\calA_1$ and $\calA_2$. Next, we further divide this step into three sub-steps.

\mystep{3.1}{Construct $\psi_1$ bijectively mapping  $\{0,1,\cdots,K-1\}^d$ to $\mathcal{A}_1$.}

Inspired by the binary representation, we define
\begin{equation}
	\psi_1(\bmx)\coloneqq \frac{x_d}{2K^d}+\sum_{i=1}^{d-1}\frac{x_i}{K^i}\quad \tn{for any $\bmx=[x_1,x_2,\cdots,x_d]^T\in \R^d$.}
\end{equation}
Then $\psi_1$ is a linear function bijectively mapping the index set $\{0,1,\cdots,K-1\}^d$ to
\begin{equation*}
	\begin{split}
		&\quad \Big\{\psi_1(\bmbeta):\bm{\beta}\in \{0,1,\cdots,K-1\}^d\Big\}=\bigg\{\tfrac{\beta_d}{2K^d}+\sum_{i=1}^{d-1}\tfrac{\beta_i}{K^i}:\bm{\beta}\in \{0,1,\cdots,K-1\}^d\bigg\}\\
		&=\Big\{\tfrac{i}{K^{d-1}}+\tfrac{k}{2K^d}:i=0,1,\cdots,K^{d-1}\black{-1}\tn{\quad and\quad } k=0,1,\cdots,K-1\Big\}=\calA_1.
	\end{split}
\end{equation*}

\mystep{3.2}{Construct $g$ to satisfy $g\circ\psi_1(\bmbeta)=\tildef(\bmx_\bmbeta)$ and to meet the requirements of applying Proposition~\ref{prop:point:fitting}.}

Let $g:[0,1]\to \R$ be a continuous piecewise linear function with a set of breakpoints 
\begin{equation*}
	\left\{\tfrac{j}{2K^d}: j=0,1,\cdots,2K^d\right\}=\calA_1\cup\calA_2\cup\{1\}.
\end{equation*} 
Moreover, the values of $g$ at these breakpoints are assigned as follows:
\begin{itemize}
	\item At the breakpoint $1$, let $g(1)=\tildef(\bm{1})$, where $\bm{1}=[1,1,\cdots,1]^T\in \R^d$.
		
	\item For the breakpoints in $\mathcal{A}_1=\big\{\psi_1(\bmbeta):\bm{\beta}\in \{0,1,\cdots,K-1\}^d\big\}$, we set
	\begin{equation}
		\label{eq:ftog}
		g\big(\psi_1(\bmbeta)\big)=\tildef(\bmx_\bmbeta)\quad \tn{for any  $\bm{\beta}\in \{0,1,\cdots,K-1\}^d$}.
	\end{equation} 
	
	\item 
	The values of $g$ at the breakpoints in $\mathcal{A}_2$ are assigned to reduce the variation of $g$, which is a requirement of applying Proposition~\ref{prop:point:fitting}. Recall that 
	\begin{equation*}
		\big\{\tfrac{i}{K^{d-1}}-\tfrac{K+1}{2K^d},\ \tfrac{i}{K^{d-1}}\big\}\subseteq\calA_1\cup\{1\}\quad \tn{for $i=1,2,\cdots,K^{d-1}$,}
	\end{equation*}
	implying the values of $g$ at $\tfrac{i}{K^{d-1}}-\tfrac{K+1}{2K^d}$ and $\tfrac{i}{K^{d-1}}$ have been assigned in the previous cases for. Thus, the values of $g$ at the breakpoints in $\calA_2$ can be successfully assigned by letting $g$ linear on each interval 
	$[\tfrac{i}{K^{d-1}}-\tfrac{K+1}{2K^d},\, \tfrac{i}{K^{d-1}}]$ for $i=1,2,\cdots,K^{d-1}$ since $\calA_2\subseteq \bigcup_{i=1}^{K^{d-1}}[\tfrac{i}{K^{d-1}}-\tfrac{K+1}{2K^d},\, \tfrac{i}{K^{d-1}}]$. See Figure~\ref{fig:g+A12} for an illustration.
\end{itemize}

\begin{figure}[htbp!]
	\centering
	\includegraphics[width=0.88\textwidth]{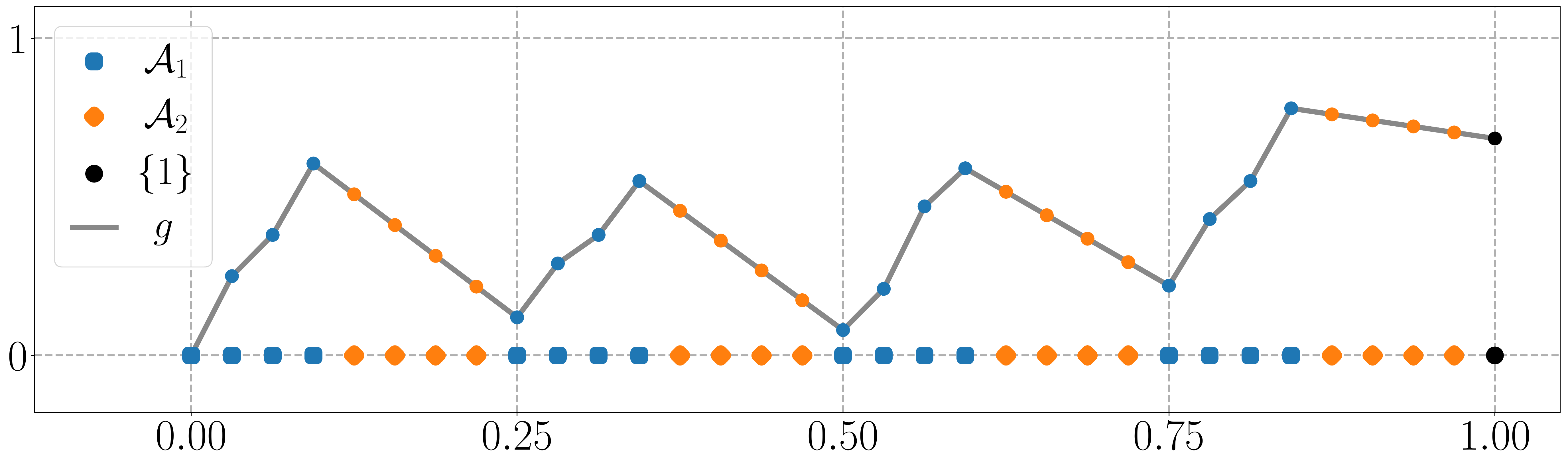}
	\caption{An illustration of $\mathcal{A}_1$, $\mathcal{A}_2$, $\{1\}$, and $g$ for $K=4$ and $d=2$.}
	\label{fig:g+A12}
\end{figure}

Apparently, such a function $g$ exists. See Figure~\ref{fig:g+A12} for an illustration of $g$.
It is easy to verify that
\begin{equation*}
	\label{eq:gErrorEstimation}
	\left|g(\tfrac{j}{2K^d})-g(\tfrac{j-1}{2K^d})\right|\le \max\Big\{\omega_\tildef(\tfrac{\sqrt{d}}{K}),\,
	\tfrac{\omega_\tildef({\sqrt{d}})}{K}\Big\}\le  \omega_\tildef(\tfrac{\sqrt{d}}{K})=\omega_f(\tfrac{\sqrt{d}}{K})
\end{equation*} 
for $j=1,2,\cdots,2K^d$. Moreover, we have
\begin{equation*}
	0\le g(\tfrac{j}{2K^d})\le 2\omega_f(\sqrt{d}) \quad \tn{for}\ j=0,1,\cdots,2K^d.
\end{equation*}

\mystep{3.3}{Construct $\psi_2$ approximating $g$ well on $\calA_1\cup\calA_2\cup\{1\}$.}

Observe that 
\begin{equation*}
	2K^d=2 \big(\lfloor n^{(s+1)/d}\rfloor \big)^d\le 2n^{s+1}\le (2n)^{s+1}=\tilden^{s+1},\quad \tn{where $\tilden=2n$}.
\end{equation*}
By Proposition~\ref{prop:point:fitting} with $y_j=g(\tfrac{j}{2K^2})$ and $\varepsilon=\omega_f(\tfrac{\sqrt{d}}{K})>0$ therein, there exists 
\begin{equation*}
	\tildepsi_2 \in \nestnet_{s}\Big\{350(s+7)^2(\tilden+1)\Big\}
	= \nestnet_{s}\Big\{350(s+7)^2(2n+1)\Big\}
\end{equation*}
%a function $\tildepsi_2$ realized by a height-$s$ NestNet with at most $3567(s+2)^4(\tilden+1)=3567(s+2)^4(2n+1)$ parameters
 such that
%\begin{equation*}
%	\begin{split}
	%		\tildepsi_2  \in \NNF(\NNinput=1\NNspace\NNwidth\le 16N+30\NNspace\NNdepth\le 6\lceil \sqrt{2}L\rceil+10\NNspace\NNoutput=1)
	%	\end{split}
%\end{equation*}
%such that
\begin{equation*}
	%\label{eq:phi1Minusg}
	|\tildepsi_2(j)-g(\tfrac{j}{2K^d})|\le \omega_f(\tfrac{\sqrt{d}}{K})\quad  \tn{for } j=0,1,\cdots,2K^d-1
\end{equation*}
and 
\begin{equation*}
	%\label{eq:phi3tUB}
	\begin{split}
		0\le \tildepsi_2(x) \le  \max\big\{g(\tfrac{j}{2K^d}):j=0,1,\cdots,2K^d-1\big\}\le 2\omega_f(\sqrt{d}) \quad \tn{for any $x\in\R$.}
	\end{split}
\end{equation*}

By defining $\psi_2(x)\coloneqq \tildepsi_2(2K^dx)$ for any $x\in \R$, we have 
%$\psi_2\in \NNF(\NNinput=1\NNspace\NNwidth\le 16N+30\NNspace\NNdepth\le 6\lceil \sqrt{2}L\rceil+10\NNspace\NNoutput=1)$,
\begin{equation}
	\label{eq:phi3tUB}
	\begin{split}
		0\le \psi_2(x)=\tildepsi_2(2K^dx) \le 2\omega_f(\sqrt{d}) \quad \tn{for any } x\in\R
	\end{split}
\end{equation}
and 
\begin{equation}
	\label{eq:phi1Minusg}
	|\psi_2(\tfrac{j}{2K^d})-g(\tfrac{j}{2K^d})|=|\tildepsi_2(j)-g(\tfrac{j}{2K^d})|\le \omega_f(\tfrac{\sqrt{d}}{K}) \quad \tn{for $ j=0,1,\cdots,2K^d-1.$}
\end{equation}

%\mystep{3.4}{Construct $\phi_2$ mapping $\bmbeta$ approximately to $\tildef(\bmx_\bmbeta)$.}
Let us end Step $3$ by defining the desired function $\phi_2$ as 
$\phi_2\coloneqq \psi_2\circ \psi_1$. 
%Note that $\psi_1:\R^d\to\R $ is an affine linear map,  and hence 
Recall that $\psi_1(\bmbeta)=\calA_1\subseteq \big\{\tfrac{j}{2K^d}:j=0,1,\cdots,2K^d-1\big\}$. 
Then, by Equations~\eqref{eq:ftog} and \eqref{eq:phi1Minusg},   we have
\begin{equation}
	\label{eq:phi2-tildef}
	\begin{split}
		\big|\phi_2(\bmbeta)-\tildef(\bmx_\bmbeta)\big|%&=\left|\psi_2\big(\sigma(\psi_1(\bmbeta))\big)-g(\psi_1(\bmbeta))\right|\\
		=\Big|\psi_2(\psi_1(\bmbeta))-g(\psi_1(\bmbeta))\Big|\le \omega_f(\tfrac{\sqrt{d}}{K})
	\end{split}
\end{equation}
for any $\bmbeta\in\{0,1,\cdots,K-1\}^d$.
Moreover, by Equation~\eqref{eq:phi3tUB} and $\phi_2= \psi_2\circ \psi_1$, we have  
\begin{equation}
	\label{eq:phi2tUB}
	\begin{split}
		0\le \phi_2(\bmx)=\psi_2\big(\psi(\bmx)\big)\le 2\omega_f(\sqrt{d}) \quad \tn{for any } \bmx\in\R^d.
	\end{split}
\end{equation}

\mystep{4}{Construct the final network to implement the desired function $\phi$.}

Define $\phi\coloneqq \phi_2\circ\bmPhi_1+f(\bmzero)-\omega_f(\sqrt{d})$. By Equation~\eqref{eq:phi2tUB}, we have
\begin{equation*}
	0\le \phi_2\circ\bmPhi_1(\bmx)\le 2\omega_f(\sqrt{d})
\end{equation*}
for any $\bmx\in\R^d$, implying
\begin{equation*}
	f(\bmzero)-\omega_f(\sqrt{d})\le \phi(\bmx)=\phi_2\circ\bmPhi_1(\bmx)+f(\bmzero)-\omega_f(\sqrt{d})\le f(\bmzero)+\omega_f(\sqrt{d}).
\end{equation*}
It follows that
$ \|\phi\|_{L^\infty(\R^d)}\le |f(\bmzero)|+ \omega_f(\sqrt{d})$. 

Next, let us estimate the approximation error.
Recall that $f=\tildef+f(\bmzero)-\omega_f(\sqrt{d})$ and $\phi=\phi_2\circ\bmPhi_1+f(\bmzero)-\omega_f(\sqrt{d})$. By Equations~\eqref{eq:bmphi:output:beta} and \eqref{eq:phi2-tildef}, for any $\bmx\in Q_\bmbeta$ and $\bmbeta\in \{0,1,\cdots,K-1\}^d$, we have
\begin{equation*}
	\begin{split}
		|f(\bmx)-\phi(\bmx)|
		&=\big|\tildef(\bmx)-\phi_2\circ\bmPhi_1(\bmx)\big|=|\tildef(\bmx)-\phi_2(\bmbeta)|\\
		&\le |\tildef(\bmx)-\tildef(\bmx_\bmbeta)|+|\tildef(\bmx_\bmbeta)-\phi_2(\bmbeta)|\\
		&\le \omega_f(\tfrac{\sqrt{d}}{K})+\omega_f(\tfrac{\sqrt{d}}{K})\le 2\omega_f\big(2\sqrt{d}\,n^{-(s+1)/d}\big),
	\end{split}
\end{equation*}
where the last inequality comes from the fact 
\begin{equation*}
	\begin{split}
		K= \lfloor n^{(s+1)/d}\rfloor
		&\ge n^{(s+1)/d}/2 \quad\tn{for $n\in\N^+$.}
	\end{split}
\end{equation*}
Recall the fact $\omega_f(j\cdot r)\le j\cdot\omega_f(r)$ for any $j\in\N^+$ and $r\in [0,\infty)$. Therefore, for any $\bmx\in \bigcup_{\bm{\beta}\in \{0,1,\cdots,K-1\}^d} Q_\bmbeta\black{=} [0,1]^d\backslash \Omega([0,1]^d,K,\delta)$, we have
\begin{equation*}
	\begin{split}
		|\phi(\bmx)-f(\bmx)|
		\le 2\omega_f\Big(2\sqrt{d}\,n^{-(s+1)/d}\Big)
		&\le 2\Big\lceil 2\sqrt{d}\Big\rceil \omega_f\big(n^{-(s+1)/d}\big)\\
		&\le 6\sqrt{d}\,\omega_f\big(n^{-(s+1)/d}\big). 
	\end{split}
\end{equation*}

\begin{figure}[htbp!]
	\centering
	\includegraphics[width=0.86\textwidth]{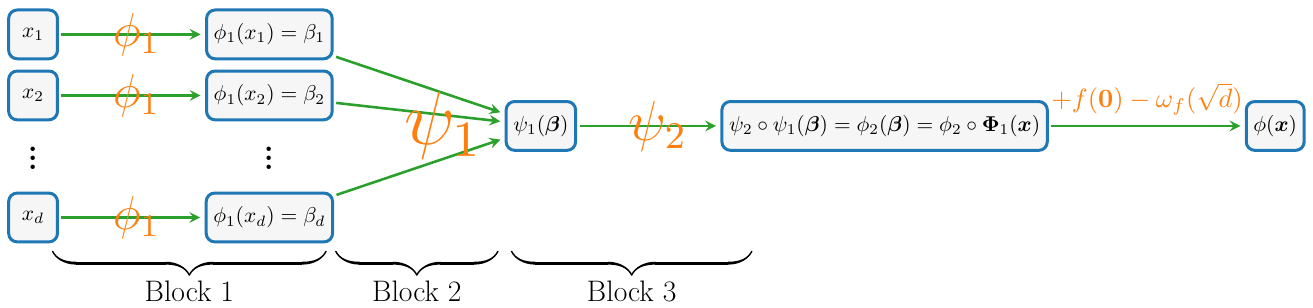}
	\caption{An illustration of the final NestNet realizing $\phi=\phi_2\circ\bmPhi_1+f(\bmzero)-\omega_f(\sqrt{d})$ for $\bmx=[x_1,x_2,\cdots,x_d]^T\in Q_\bmbeta$ for each $\bmbeta\in\{0,1,\cdots,K-1\}^d$.}
	\label{fig:phi:from:phi:1:psi:1:2}
\end{figure}

It remains to estimate the number of parameters in the NestNet realizing $\phi$, which is shown in Figure~\ref{fig:phi:from:phi:1:psi:1:2}. 
Recall that $\phi_1\in \nestnet_{s}\big\{72(s+7)n\big\}$,
$\psi_1$ is an affine linear map,
and $\psi_2\in \nestnet_{s}\big\{350(s+7)^2(2n+1)\big\}$.
%Block $1$ has at most $d^2\big(256(s+1)^3n\big)=256d^2(s+1)^3n$ parameters; Block $2$ has at most $d+1$ parameters; and Block $3$ has at most $3531(s+1)^4(2n+1)$ parameters. 
Therefore, $\phi=\phi_2\circ\bmPhi_1+f(\bmzero)-\omega_f(\sqrt{d})$ can be realized by a height-$s$ NestNet with at most
\begin{equation*}
	\underbrace{d^2\big(72(s+7)n\big)}_{\tn{Block 1}}
	\   +\  
	\underbrace{(d+1)}_{\tn{Block 2}}
	\   +\ 
	\underbrace{350(s+7)^2(2n+1)}_{\tn{Block 3}}
	\  +\ 1
	\   \le \
	355d^2(s+7)^2(2n+1)
\end{equation*}
parameters, which means we finish the proof of Theorem~\ref{thm:main:gap}.

%%%%%%%%%%%%%%%%%%%%%%%%%%%%%%%%%%%%%
%%%%%%%%%%%%%%%%%%%%%%%%%%%%%%%
\section{Proof of Proposition~\ref{prop:floor:approx} }
\label{sec:proof:prop:floor:approx}

%To simplify the proof of Proposition~\ref{prop:floor:approx}, we introduce the following lemma.

The key point of proving Proposition~\ref{prop:floor:approx} is 
the composition architecture of neural networks.
%the bit extraction technique proposed in \cite{Bartlett98almostlinear}.
%To simplify the proof of Proposition~\ref{prop:point:fitting},
To simplify the proof,
%of Proposition~\ref{prop:floor:approx},
we first establish several lemmas for proving Proposition~\ref{prop:floor:approx} 
%and give their proofs 
in Section~\ref{sec:lemmas:for:floor:approx}. 
Next, we present the detailed proof of Proposition~\ref{prop:floor:approx} in Section~\ref{sec:detailed:proof:floor:approx} based on the lemmas established in Section~\ref{sec:lemmas:for:floor:approx}.

\subsection{Lemmas for proving Proposition~\ref{prop:floor:approx}}
\label{sec:lemmas:for:floor:approx}

\begin{lemma}\label{lem:floor:approx:nestnet:s:order}
	Given any $n,r\in \N^+$ and $\delta\in \big(0,\,\tfrac{1}{C({r,n})}\big)$ with $C({r,n})=\prod_{i=1}^{r}2^{n^{i}}$,
	there exists $\phi \in \nestnet_{r}\big\{(12r+68)n\big\}$
	%	a function $\phi$ realized by a $(\sigma,\varrho_1,\varrho_2)$-activated network  with $298(s+1)^3(n+1)$ parameters 
	such that
	\begin{equation*}%\label{eq:floor:approx:n:power:r}
		\phi(x)=\lfloor x\rfloor \quad \tn{for any $x\in \bigcup_{\ell=0}^{2^{n^{r}}-1}\big[\ell,\,\ell+1-C({r,n})\cdot\delta\big]$.}
	\end{equation*}
\end{lemma}
We will prove Lemma~\ref{lem:floor:approx:nestnet:s:order} by induction. To simplify the proof, we introduce two lemmas for the base case and the induction step.

First, we introduce the following lemma for the base case of proving Lemma~\ref{lem:floor:approx:nestnet:s:order}.
\begin{lemma}\label{lem:floor:approx:base:case}
	Given any $n\in \N^+$ and $\delta\in (0,1)$, 
	there exists 
	a function $\phi$ realized by a ReLU network  of width $4$ and depth $4n-1$
	such that
	\begin{equation*}
		\phi(x)=\lfloor x\rfloor \quad \tn{for any $x\in \bigcup_{\ell=0}^{2^{n}-1}[\ell,\,\ell+1-\delta]$.}
	\end{equation*}
\end{lemma}
\begin{proof}%[Proof of Lemma~\ref{lem:floor:approx:base:case}]
%	By setting $m=2^{n^s}$, we have $m^n=\big(2^{n^s}\big)^n=2^{(n^s)n}=2^{n^{s+1}}$ and
%\begin{equation}\label{eq:floor:approx:m}
%	g(x)=\lfloor x\rfloor \quad \tn{for any $x\in \bigcup_{\ell=0}^{m-1}[\ell,\,\ell+1-\delta]$.}
%\end{equation}

Set $\tildedelta=2^{-n}\delta$ and define
\begin{equation*}
	\phi_0(x)\coloneqq \frac{\sigma(x-1+\tildedelta)-\sigma(x-1)}{\tildedelta}\quad \tn{for $x\in \R$.}
\end{equation*}
Clearly, $\phi_0$ can be realized by a one-hidden-layer ReLU network of width $2$. Moreover, we have
\begin{equation*}
	\phi_0(x)=\frac{\sigma(x-1+\tildedelta)-\sigma(x-1)}{\tildedelta}=\frac{0-0}{\tildedelta}=0\quad \tn{if $x\in [0,1-\tildedelta]$}
\end{equation*}
and
\begin{equation*}
	\phi_0(x)=\frac{\sigma(x-1+\tildedelta)-\sigma(x-1)}{\tildedelta}=\frac{(x-1+\tildedelta)-(x-1)}{\tildedelta}=1\quad \tn{if $x\in [1,2-\tildedelta]$}.
\end{equation*}

By  fixing
\begin{equation*}
	x\in
%	 \bigcup_{\ell=0}^{2^{n^{s+1}}-1}
%[\ell,\,\ell+1-2^{n^{s+1}}\delta]=
	  \bigcup_{\ell=0}^{2^n-1}[\ell,\,\ell+1-\delta]
	  =\bigcup_{\ell=0}^{2^n-1}[\ell,\,\ell+1-2^n\tildedelta],
\end{equation*}
we have $\lfloor x\rfloor \in \{0,1,\cdots,2^n-1\}$, implying that $\lfloor x\rfloor$ can be represented as
\begin{equation*}
	\lfloor x\rfloor = \sum_{i=0}^{n-1}z_i 2^i\quad \tn{for $z_0,z_1,\cdots,z_{n-1}\in \{0,1\}$.}
\end{equation*}
Then, for $j=0,1,\cdots,n-1$, we have $\sum_{i=0}^{j} z_i2^i+1\le z_j2^j+\sum_{i=0}^{j-1}2^i+1 \le    z_j2^j+2^j$, implying
\begin{equation*}
	\begin{split}
		\tfrac{x-\sum_{i=j+1}^{n-1} z_i2^i}{2^j}
		\in \Big[\tfrac{\lfloor x\rfloor-\sum_{i=j+1}^{n-1} z_i2^i}{2^j},\,\tfrac{\lfloor x\rfloor+1-2^n\tildedelta-\sum_{i=j+1}^{n-1} z_i2^i}{2^j}\Big]
		&=\Big[\tfrac{\sum_{i=0}^{j} z_i2^i}{2^j},\,\tfrac{\sum_{i=0}^{j} z_i2^i+1-2^n\tildedelta}{2^j}\Big]\\
		&\subseteq\Big[\tfrac{ z_j2^j}{2^j},\,\tfrac{ z_j2^j+2^j-2^n\tildedelta}{2^j}\Big]
		\subseteq [z_j,\, z_j+1-\tildedelta].
	\end{split}
\end{equation*}
It follows that
\begin{equation*}
	\phi_0\Big(\tfrac{x-\sum_{i=j+1}^{n-1} z_i2^i}{2^j}\Big)
	=z_j\quad \tn{for $j=0,1,\cdots,n-1$}.
\end{equation*}
Therefore, the desired function $\phi$ can be realized by the network in Figure~\ref{fig:floor:from:z:i:2:i}.

\begin{figure}[htbp!]
	\centering
	\includegraphics[width=0.985\textwidth]{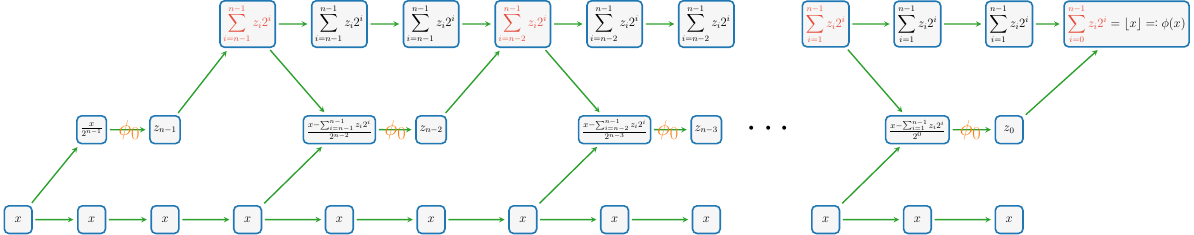}
	\caption{An illustration of the NestNet realizing $\phi$. Here,  $\phi_0$ represent an one-hidden-layer ReLU network of width $2$.}
	\label{fig:floor:from:z:i:2:i}
\end{figure}

Clearly, 
\begin{equation*}
	\phi(x)=\lfloor x\rfloor\quad \tn{for any 
	$x\in
	\bigcup_{\ell=0}^{2^n-1}[\ell,\,\ell+1-\delta].$}
\end{equation*}
Moreover, $\phi$ can be realized by a ReLU network of width $1+2+1=4$ and depth $(1+1+1)+(1+1+1+1)(n-1)=4n-1$. Hence, we finish the proof of Lemma~\ref{lem:floor:approx:base:case}.
\end{proof}

Next, we introduce the following lemma for the induction step of proving Lemma~\ref{lem:floor:approx:nestnet:s:order}.
\begin{lemma}\label{lem:floor:approx:induction:step}
	Given any $n,s,\hatn\in \N^+$ and $\delta\in \big(0,\,\tfrac{1}{2^{n^{s+1}}}\big)$, if $g \in \nestnet_{s}\{\hatn\}$ satisfying
	\begin{equation*}
			g(x)=\lfloor x\rfloor \quad \tn{for any $x\in \bigcup_{\ell=0}^{2^{n^s}-1}[\ell,\,\ell+1-\delta]$.}
	\end{equation*}
	Then there exists $\phi \in \nestnet_{s+1}\big\{\hatn+12n-7\big\}$
	%	a function $\phi$ realized by a $(\sigma,\varrho_1,\varrho_2)$-activated network  with $298(s+1)^3(n+1)$ parameters 
	such that
	\begin{equation*}
		\phi(x)=\lfloor x\rfloor \quad \tn{for any $x\in \bigcup_{\ell=0}^{2^{n^{s+1}}-1}[\ell,\,\ell+1-2^{n^{s+1}}\delta]$.}
	\end{equation*}
\end{lemma}
\begin{proof}%[Proof of Lemma~\ref{lem:floor:approx:induction:step}]
	By setting $m=2^{n^s}$, we have $m^n=\big(2^{n^s}\big)^n=2^{(n^s)n}=2^{n^{s+1}}$ and
	\begin{equation}\label{eq:floor:approx:m}
		g(x)=\lfloor x\rfloor \quad \tn{for any $x\in \bigcup_{\ell=0}^{m-1}[\ell,\,\ell+1-\delta]$.}
	\end{equation}
	By fixing
	\begin{equation*}
		x\in \bigcup_{\ell=0}^{2^{n^{s+1}}-1}[\ell,\,\ell+1-2^{n^{s+1}}\delta]= \bigcup_{\ell=0}^{m^n-1}[\ell,\,\ell+1-m^n\delta],
	\end{equation*}
	we have $\lfloor x\rfloor \in \{0,1,\cdots,m^n-1\}$, implying that $\lfloor x\rfloor$ can be represented as
	\begin{equation*}
		\lfloor x\rfloor = \sum_{i=0}^{n-1}z_i m^i\quad \tn{for $z_0,z_1,\cdots,z_{n-1}\in \{0,1,\cdots,m-1\}$.}
	\end{equation*}
Then, for $j=0,1,\cdots,n-1$, we have 
\begin{equation*}
	\sum_{i=0}^{j} z_im^i+1\le z_jm^j+ \sum_{i=0}^{j-1}(m-1)m^i+1= z_jm^j+m^j,
\end{equation*} 
implying
\begin{equation*}
	\begin{split}
		\tfrac{x-\sum_{i=j+1}^{n-1} z_im^i}{m^j}
		&\in \Big[\tfrac{\lfloor x\rfloor-\sum_{i=j+1}^{n-1} z_im^i}{m^j},\,\tfrac{\lfloor x\rfloor+1-m^n\delta-\sum_{i=j+1}^{n-1} z_im^i}{m^j}\Big]\\
		&=\Big[\tfrac{\sum_{i=0}^{j} z_im^i}{m^j},\,\tfrac{\sum_{i=0}^{j} z_im^i+1-m^n\delta}{m^j}\Big]\\
		&\subseteq \Big[\tfrac{ z_jm^j}{m^j},\,\tfrac{ z_jm^j+m^j-m^n\delta}{m^j}\Big]
		\subseteq \big[z_j,\, z_j+1-\delta\big].
	\end{split}
\end{equation*}
It follows that
\begin{equation*}
	g\Big(\tfrac{x-\sum_{i=j+1}^{n-1} z_im^i}{m^j}\Big)
	=z_j\quad \tn{for $j=0,1,\cdots,n-1$}.
\end{equation*}
Therefore, the desired function $\phi$ can be realized by the network in Figure~\ref{fig:floor:from:z:i:m:i}.

\begin{figure}[htbp!]
	\centering
	\includegraphics[width=0.985\textwidth]{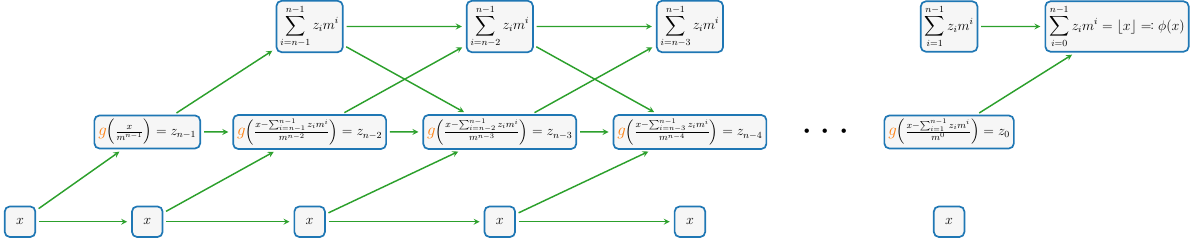}
	\caption{An illustration of the NestNet realizing $\phi$. Here,  $g$ is regarded as an activation function.}
	\label{fig:floor:from:z:i:m:i}
\end{figure}

 Clearly, 
\begin{equation*}
	\phi(x)=\lfloor x\rfloor\quad 
	\tn{for any $
	x\in
	\bigcup_{\ell=0}^{m^n-1}[\ell,\,\ell+1-m^n\delta]
	= \bigcup_{\ell=0}^{2^{n^{s+1}}-1}[\ell,\,\ell+1-2^{n^{s+1}}\delta].$}
\end{equation*}
Moreover, the fact $g\in \nestnet_{s}\{\hatn\}$ implies that $\phi$ can be realized by a height-$(s+1)$ NestNet with at most
\begin{equation*}
	\underbrace{(1+1)2+(2+1)3+ (3+1)3(n-2)+ (3+1)}_{\tn{outer network}}
	\   +\ 
	\underbrace{\hatn}_{g}
	\  = \ 
	\hatn + 12n-7
\end{equation*}
parameters. Hence, we finish the proof of Lemma~\ref{lem:floor:approx:induction:step}.
\end{proof}

With Lemmas~\ref{lem:floor:approx:base:case} and \ref{lem:floor:approx:induction:step} in hand, we are ready to prove Lemma~\ref{lem:floor:approx:nestnet:s:order}.
\begin{proof}[Proof of Lemma~\ref{lem:floor:approx:nestnet:s:order}]
	We will use the mathematical induction to prove Lemma~\ref{lem:floor:approx:nestnet:s:order}.
	First, we consider the base case $r=1$. By Lemma~\ref{lem:floor:approx:base:case}, there exists 
	a function $\phi$ realized by a ReLU network  of width $4$ and depth $4n-1$
	such that
	\begin{equation*}
		\phi(x)=\lfloor x\rfloor \quad \tn{for any $x\in \bigcup_{\ell=0}^{2^{n}-1}[\ell,\,\ell+1-\delta]\subseteq \bigcup_{\ell=0}^{2^{n}-1}[\ell,\,\ell+1-C(r,n)\cdot\delta]$ with $r=1$.}
	\end{equation*}
	Moreover, the network realizing $\phi$ has at most
	$(4+1)4\big((4n-1)+1\big)=80n$ parameters, implying $\phi\in \nestnet_{1}\{80n\}\subseteq \nestnet_{1}\{(12r+68)n\}$ for $r=1$. Thus, the base case $r=1$ is proved.
	
	Next, assume Lemma~\ref{lem:floor:approx:nestnet:s:order} holds for $r=s\in\N^+$. We need to show it is also true for $r=s+1$.
	By the induction hypothesis, 
	there exists $g \in \nestnet_{s}\big\{(12s+68)n\big\}$
	%	a function $\phi$ realized by a $(\sigma,\varrho_1,\varrho_2)$-activated network  with $298(s+1)^3(n+1)$ parameters 
	such that
	\begin{equation*}
		g(x)=\lfloor x\rfloor \quad \tn{for any $x\in \bigcup_{\ell=0}^{2^{n^{s}}-1}[\ell,\,\ell+1-C(s,n)\cdot\delta]$.}
	\end{equation*}	
	By Lemma~\ref{lem:floor:approx:induction:step} with $\hatn=(12s+68)n$ therein  and  setting $\hatdelta=C(s,n)\cdot\delta$, 
	there exists 
	\begin{equation*}
		\phi \in \nestnet_{s+1}\big\{\hatn+12n-7\big\}\subseteq \nestnet_{s+1}\big\{(12s+68)n+12n-7\big\}
		\subseteq \nestnet_{s+1}\Big\{\big(12(s+1)+68\big)n\Big\}
	\end{equation*}
	%	a function $\phi$ realized by a $(\sigma,\varrho_1,\varrho_2)$-activated network  with $298(s+1)^3(n+1)$ parameters 
	such that
	\begin{equation*}
		\phi(x)=\lfloor x\rfloor \quad \tn{for any $x\in \bigcup_{\ell=0}^{2^{n^{s+1}}-1}[\ell,\,\ell+1-2^{n^{s+1}}\hatdelta\,]$.}
	\end{equation*}
Observe that
\begin{equation*}
	2^{n^{s+1}}\hatdelta
	=2^{n^{s+1}}C(s,n)\cdot\delta
	=
	2^{n^{s+1}}\Big(\prod_{i=1}^{s}2^{n^{i}}\Big)\cdot\delta
	=\Big(\prod_{i=1}^{s+1}2^{n^{i}}\Big)\cdot\delta =C(s+1,n)\cdot\delta.
\end{equation*}
It follows that
\begin{equation*}
	\phi(x)=\lfloor x\rfloor \quad \tn{for any $x\in \bigcup_{\ell=0}^{2^{n^{s+1}}-1}[\ell,\,\ell+1-C(s+1,n)\cdot\delta]$.}
\end{equation*}
Thus, Lemma~\ref{lem:floor:approx:nestnet:s:order} is proved for the case $r=s+1$,
which means we finish the induction step. Hence, by the principle of induction, we complete the proof of Lemma~\ref{lem:floor:approx:nestnet:s:order}.
\end{proof}

\subsection{Detailed proof of Proposition~\ref{prop:floor:approx}}
\label{sec:detailed:proof:floor:approx}

%The proof of Lemma~\ref{lem:floor:approx:exp} is placed later in the section.
%Now, let us prove Proposition~\ref{prop:floor:approx} by assuming Lemma~\ref{lem:floor:approx:exp} is true.
%\begin{proof}[Proof of Proposition~\ref{prop:floor:approx}]
%	By Lemma~\ref{lem:floor:approx:exp} with $v=r\lceil \log_2 (n+1)\rceil \in \N^+$ therein, there exists a function $\phi_0$ realized by a ReLU network of width $v+1$ and depth $4v-3$ such that 
Set $C(r,n)=\prod_{i=1}^{r}2^{n^{i}}$ and $\tildedelta=\tfrac{\delta}{C(r,n)}\in \big(0,\, \tfrac{1}{C(r,n)}\big)$.
By Lemma~\ref{lem:floor:approx:nestnet:s:order}, there exists $\phi_0\in\nestnet_{r}\big\{(12r+68)n\big\}$ such that
	\begin{equation*}
		\phi_0(x)=\lfloor x\rfloor \quad \tn{for any $x\in \bigcup_{\ell=0}^{2^{n^r}-1}[\ell,\,\ell+1-C(r,n)\cdot\tildedelta]=\bigcup_{\ell=0}^{2^{n^r}-1}[\ell,\,\ell+1-\delta]$.}
	\end{equation*}

	It follows from $J\le 2^{n^r}$ that
	\begin{equation*}
		\phi_0(x)=\lfloor x\rfloor \quad \tn{for any $x\in \bigcup_{j=0}^{J-1}[j,\,j+1-\delta]$.}
	\end{equation*}
	Set
	\begin{equation*}
		\tildeM= \max_{x\in [J,J+1]} |\phi_0(x)| \quad \tn{and}\quad M=\frac{\tildeM+J}{\delta}.
	\end{equation*}
	Then, for any $x\in [J,\,J+1]$, we have
	\begin{equation*}
		\begin{split}
			\phi_0(x)+M\sigma\big(x-(J-\delta)\big)
			\ge -\tildeM+ M\delta=-\tildeM+(\tildeM+J)=J,
		\end{split}
	\end{equation*}
	implying
	\begin{equation*}
		\min\Big\{\phi_0(x)+M\sigma\big(x-(J-\delta)\big),\, J\Big\}=J.
	\end{equation*}
	Moreover, for any $x\in \bigcup_{j=0}^{J-1}[j,\,j+1-\delta]$,  we have $\sigma\big(x-(J-\delta)\big)=0$, implying
	\begin{equation*}
		\min\Big\{\phi_0(x)+M\sigma\big(x-(J-\delta)\big),\, J\Big\}=\min\Big\{\phi_0(x),\, J\Big\}=\min\Big\{\lfloor x\rfloor,\, J\Big\}=\lfloor x\rfloor.
	\end{equation*}
	
	Therefore, by defining 
	\begin{equation*}
		\phi(x)\coloneqq \min\Big\{\phi_0(x)+M\sigma\big(x-(J-\delta)\big),\, J\Big\} \quad \tn{for any $x\in \bigcup_{j=0}^{J}\big[j,\,j+1-\delta\cdot\one_{\{j\le J-1\}}\big]$,}
	\end{equation*}
	we have
	\begin{equation*}
		\phi(x)=\lfloor x\rfloor \quad \tn{for any $x\in \bigcup_{j=0}^{J-1}[j,\,j+1-\delta]$}
	\end{equation*}
	and 
	\begin{equation*}
		\phi(x)=J \quad \tn{for any $x\in [J,\,J+1]$.}
	\end{equation*}
	
	\begin{figure}[htbp!]       
		\centering         \includegraphics[width=0.8\textwidth]{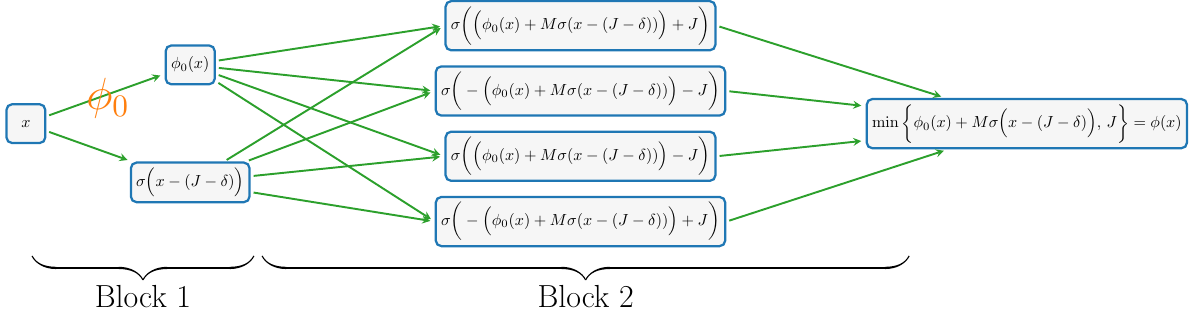}
		\caption{An illustration of the network realizing $\phi$ for any $x\in\bigcup_{j=0}^{J}\big[j,\,j+1-\delta\cdot\one_{\{j\le J-1\}}\big]$ based on the fact
			$\min\{a,b\}=\tfrac{1}{2}\big(\sigma(a+b)-\sigma(-a-b)-\sigma(a-b)-\sigma(-a+b)\big)$. }
		\label{fig:phi:from:phi0:min}
	\end{figure}
	
	Moreover, $\phi$ can be realized by the network in Figure~\ref{fig:phi:from:phi0:min}. The fact $\phi_0\in \nestnet_{r}\big\{(12r+68)n\big\}$ implies that  $\phi$ can be realized by a height-$r$ NestNet with at most
	\begin{equation*}
		\underbrace{3\Big((12r+68)n\Big)}_{\tn{Block 1}}
		\  +\ 
		\underbrace{(2+1)4+(4+1)}_{\tn{Block 2}}
		\   \le \ 36(r+7)n
	\end{equation*}
parameters.
	So we finish the proof of Proposition~\ref{prop:floor:approx}.
\section{Proof of Proposition~\ref{prop:point:fitting}}
\label{sec:proof:prop:point:fitting}

The key idea of proving Proposition~\ref{prop:point:fitting} is the bit extraction technique proposed in \cite{Bartlett98almostlinear}.
%To simplify the proof of Proposition~\ref{prop:point:fitting},
First, we establish several lemmas for proving Proposition~\ref{prop:point:fitting} and give their proofs in Section~\ref{sec:lemmas:for:point:fitting} except for Lemma~\ref{lem:bit:extraction}, the proof of which is placed in  Section~\ref{sec:proof:bit:extraction} since it is complicated.
Next, we present the detailed proof of Proposition~\ref{prop:point:fitting} in Section~\ref{sec:detailed:proof:point:fitting} based on the lemmas established in Section~\ref{sec:lemmas:for:point:fitting}.

\subsection{Lemmas for proving Proposition~\ref{prop:point:fitting}}
\label{sec:lemmas:for:point:fitting}

To simplify the proof of Proposition~\ref{prop:point:fitting},
we establish several lemmas as the intermediate step.
We first establish a lemma to show that any continuous piecewise linear functions on $\R$ can be realized by one-hidden-layer ReLU networks.
% in the following lemma.
\begin{lemma}\label{lem:cpl(p)}
	Given any $p\in \N^+$, any continuous piecewise linear function on $\R$ with at most $p$ breakpoints can be realized by a one-hidden-layer ReLU network of width $p+1$.
\end{lemma}
\begin{proof}%[Proof of Lemma~\ref{lem:cpl(p)}]
	We will use the mathematical induction to prove Lemma~\ref{lem:cpl(p)}. 
	First, we consider the base case $p=1$.  Suppose $f:\R\to\R$ is a continuous piecewise linear function on $
	\R$ with  at most $p=1$ breakpoints. Then there exist $a_1,a_2,x_0\in\R$ such that
	\begin{equation*}
		f(x)=
		\begin{cases}
			a_1(x-x_0)+f(x_0) &\tn{if \  }  x\ge x_0\\
			a_2(x_0-x)+f(x_0) &\tn{if \  }  x< x_0.\\
		\end{cases}
	\end{equation*}
	Thus, $f(x)=a_1\sigma(x-x_0)+a_2\sigma(x_0-x)+f(x_0)$ for any $x\in \R$, implying $f$ can be realized by a one-hidden-layer ReLU network of width $2=p+1$ for $p=1$.
	Hence, Lemma~\ref{lem:cpl(p)} is proved for the case $p=1$.
	
	Now, assume Lemma~\ref{lem:cpl(p)} holds for $p=k\in \N^+$, we would like to show it is also  true for $p=k+1$.
	Suppose $f:\R\to\R$ is a continuous piecewise linear function on with  at most $k+1$ breakpoints. We may assume the biggest breakpoint of $f$ is $x_0$ since it is trivial for the case that $f$ has no breakpoint. Denote the slopes of the linear pieces left and right next to $x_0$ by $a_1$ and $a_2$, respectively.
	Define 
	\[\tildef(x)\coloneqq f(x)- (a_2-a_1)\sigma(x-x_0)\quad \tn{ for any $x\in\R$.}\]
	Then $\tildef$ 
	has at most $k$ breakpoints.
	By the induction hypothesis, $\tildef$ can be realized by a one-hidden-layer ReLU network of width $k+1$.
	Thus, there exist   $w_{0,j},b_{0,j},w_{1,j},b_1$ for $j=1,2,\cdots,k+1$  such that
	\begin{equation*}
		\tildef(x)=\sum_{j=1}^{k+1} w_{1,j}\sigma(w_{0,j}x+b_{0,j})+b_1\quad \tn{for any $x\in\R$.}
	\end{equation*}
	Therefore, for any $x\in\R$, we have
	\begin{equation*}
		f(x)=(a_2-a_1)\sigma(x-x_0)+\tildef(x)
		= 	(a_2-a_1)\sigma(x-x_0)
		+\sum_{j=1}^{k+1} w_{1,j}\sigma(w_{0,j}x+b_{0,j})+b_1,
	\end{equation*}
	implying $f$ can be realized by a one-hidden-layer ReLU network of width $k+2=(k+1)+1=p+1$ for $p=k+1$.
	Thus, we finish the induction process. Therefore, by the principle of induction, we complete the proof of Lemma~\ref{lem:cpl(p)}.
\end{proof}

Next, we establish a lemma to extract the sum of $n^{s}$ bits via a height-$s$ NestNet with $\calO(n)$ parameters.
\begin{lemma}
	\label{lem:bit:extraction}
	Given any $n,s\in \N^+$,
	there exists
	$\phi\in \nestnet_{s}\big\{57(s+7)^2(n+1)\big\}$
%	 a function  $\phi$ realized by a height-$s$ ReLU NestNet with at most $587(s+2)^4(n+1)$ parameters 
	 such that: For any $\theta_1,\theta_2,\cdots,\theta_{n^{s}}\in \{0,1\}$, we have
	\begin{equation}\label{eq:bit:extraction:s+1}
		\phi \big(k+\bin 0.\theta_1\theta_2\cdots\theta_{n^{s}}\big)=\sum_{\ell=1}^{k}\theta_\ell \quad \tn{for $k=0,1,\cdots,n^{s}$.}
	\end{equation}
\end{lemma}

The proof of Lemma~\ref{lem:bit:extraction} is complicated and hence is placed in Section~\ref{sec:proof:bit:extraction}. 
Then, based on Lemma~\ref{lem:bit:extraction}, we establish a new lemma, Lemma~\ref{lem:bit:extraction:from:indices} below, which is a key intermediate conclusion to prove Proposition~\ref{prop:point:fitting}.
\begin{lemma}
	\label{lem:bit:extraction:from:indices}
	Given any $n,s\in\N^+$ and $\theta_{i,\ell}\in \{0,1\}$ for $i =0,1,\cdots,n-1$ and $\ell=0,1,\cdots,m-1$, where $m=n^{s}$, there exists  
	$\phi\in \nestnet_{s}\big\{58(s+7)^2(n+1)\big\}$
%	a function $\phi$ realized by
%	a height-$s$ ReLU NestNet  with at most $588(s+2)^4(n+1)$ parameters
	such that
	\[\phi(j)=\sum_{\ell=0}^{ k }\theta_{i,\ell} \quad \tn{
		for $j=0,1,\cdots,nm-1$,}\]
	where $(i,k)$ is the unique index pair satisfying $j=im+k$ with $i\in \{0,1,\cdots,n-1\}$ and $k\in \{0,1,\cdots,m-1\}$.
\end{lemma}
\begin{proof}
	We first construct a network to extract the unique index pair $(i,k)$  from $j\in \{0,1,\cdots,nm-1\}$ with the following condition
	\begin{equation*}
		j=im+k \quad \tn{with $i\in \{0,1,\cdots,n-1\}$ and $k\in \{0,1,\cdots,m-1\}$}.
	\end{equation*}
	There exists  a continuous piecewise linear function $\phi_1$ with $2n$ breakpoints such that
	\begin{equation*}
		\phi_1(x)=\lfloor x\rfloor \quad \tn{for any $x\in \bigcup_{\ell=0}^{n-1}[\ell,\,\ell+1-\delta]$ with $\delta=\tfrac{1}{2m}$.}
	\end{equation*}
	By Lemma~\ref{lem:cpl(p)}, $\phi_1$ can be realized by a one-hidden-layer ReLU network of width $2n+1$. Moreover,
	for any $j\in \{0,1,\cdots,nm-1\}$, we have 
	\begin{equation*}
		\phi_1(\tfrac{j}{m})=\lfloor \tfrac{j}{m}\rfloor = i\quad \tn{and} \quad j-m\phi_1(\tfrac{j}{m})=j-mi=k,
	\end{equation*}
	where $(i,k)$ is the unique index pair satisfying $j=im+k$ with $i\in \{0,1,\cdots,n-1\}$ and $k\in \{0,1,\cdots,m-1\}$. By defining
	\begin{equation*}
		\bmPhi_1(x)\coloneqq \left[
		\begin{array}{c}
			\phi_1(\tfrac{x}{m})\\
			x-m\phi_1(\tfrac{x}{m})
		\end{array}
		\right]\quad \tn{for any $x\ge 0$,}
	\end{equation*}
	we have 
	\begin{equation*}
		\bmPhi_1(j)=\left[
		\begin{array}{c}
			\phi_1(\tfrac{j}{m})\\
			j-m\phi_1(\tfrac{j}{m})
		\end{array}
		\right]=\left[
		\begin{array}{c}
			i\\
			k
		\end{array}
		\right] \quad \tn{for $j=0,1,\cdots,nm-1$,}
	\end{equation*}
	where $(i,k)$ is the unique index pair satisfying $j=im+k$ with $i\in \{0,1,\cdots,n-1\}$ and $k\in \{0,1,\cdots,m-1\}$.
	Moreover, $\bmPhi_1$ can be realized by a one-hidden-layer ReLU network of width $2(2n+1)+1=4n+3$. Hence, the network realizing $\bmPhi_1$ has at most 
	$(1+1)(4n+3)+\big((4n+3)+1\big)2=16n+14$ parameters.

	Define 
	\begin{equation*}
		z_i\coloneqq \bin    0.\theta_{i,0}\theta_{i,1}\cdots \theta_{i,m-1}\quad\tn{for $i=0,1,\cdots,n-1$.}
	\end{equation*} 	
	There exists  a continuous piecewise linear function $\tildephi_2$ with $n$ breakpoints such that
	\begin{equation*}
		\tildephi_2(i)=z_i \quad \tn{for $i=0,1,\cdots,n-1$.}
	\end{equation*}
	By Lemma~\ref{lem:cpl(p)}, $\tildephi_2$ can be realized by a one-hidden-layer ReLU network of width $n+1$.

	By Lemma~\ref{lem:bit:extraction}, 
	there exists 
	$\phi_3\in \nestnet_{s}\big\{57(s+7)^2(n+1)\big\}$
%	a function $\phi_3$ realized by a height-$s$ ReLU NestNet  with at most $587(s+2)^4(n+1)$ parameters 
	such that: For any $\xi_1,\xi_2,\cdots,\xi_{n^{s}}\in \{0,1\}$, we have
	\begin{equation*}
		\phi_3 \big(k+\bin 0.\xi_1\xi_2\cdots\xi_{n^{s}}\big)=\sum_{\ell=1}^{k}\xi_\ell \quad \tn{for $k=1,2,\cdots,n^{s}$.}
	\end{equation*}
	It follows from $m=n^{s}$ that,
	for any  $\xi_0,\xi_1,\cdots,\xi_{m-1}\in \{0,1\}$, we have
	\[\phi_3(k+\bin 0.\xi_0\xi_1\cdots \xi_{m-1})=\sum_{\ell=1}^{ k }\xi_{\ell-1}=\sum_{\ell=0}^{ k-1 }\xi_{\ell}\quad \tn{for  $ k =1,2,\cdots,m$,}\]
	implying 
	\[\phi_3(k+1+\bin 0.\xi_0\xi_1\cdots \xi_{m-1})=\sum_{\ell=0}^{ k }\xi_\ell\quad \tn{for  $ k =0,1,\cdots,m-1$.}\]
	Then, for $i=0,1,\cdots,n-1$ and   $ k =0,1,\cdots,m-1$, we have
	\begin{equation*}
		\phi_3\big(k+1+\tildephi_2(i)\big)
		=\phi_2(k+1+z_i)
		=\phi_3(k+1+\bin 0.\theta_{i,0}\theta_{i,1}\cdots \theta_{i,m-1})=\sum_{\ell=0}^{ k }\theta_{i,\ell}.
	\end{equation*}
	
	By defining 
	\begin{equation*}
		\phi_2(x,y)\coloneqq y+1+\tildephi_2(x)\quad \tn{for any $x,y\in [0,\infty)$}
	\end{equation*}
	and $\phi\coloneqq \phi_3\circ\phi_2\circ\bmPhi_1$,
	we have
	\begin{equation*}
		\begin{split}
			\phi(j)=\phi_3\circ\phi_2\circ\bmPhi_1(j)
			=\phi_3\circ\phi_2(i,k)=\phi_3\big(k+1+\tildephi_2(i)\big)=\sum_{\ell=0}^{ k }\theta_{i,\ell}
		\end{split}
	\end{equation*}
	for $j=0,1,\cdots,nm-1$,
	where $(i,k)$ is the unique index pair satisfying $j=im+k$ with $i\in \{0,1,\cdots,n-1\}$ and $k\in \{0,1,\cdots,m-1\}$.
	
	It remains to estimate the number of parameters in the NestNet realizing $\phi=\phi_3\circ\phi_2\circ\bmPhi_1$. Observe that $\phi_2$ can be realized by a one-hidden-layer ReLU network of width 
	$(n+1)+1=n+2$. Then, the network realizing $\phi_2$ has at most
	$(2+1)(n+2)+\big((n+2)+1\big)=4n+9$ parameters. Therefore, $\phi$ can be realized by a height-$s$ NestNet with at most
	\begin{equation*}
		\underbrace{(16n+14)}_{\bmPhi_1} + \underbrace{(4n+9)}_{\phi_2} +\underbrace{57(s+7)^2(n+1)}_{\phi_3}
		\le 58(s+7)^2(n+1)
	\end{equation*}
	parameters, which means we complete the proof of Lemma~\ref{lem:bit:extraction:from:indices}.
\end{proof}

\subsection{Detailed proof of Proposition~\ref{prop:point:fitting} }
\label{sec:detailed:proof:point:fitting}
%
%
%Let us prove Proposition~\ref{prop:point:fitting} by assuming Lemma~\ref{lem:bit:extraction} is true.
%\begin{proof}[Proof of Proposition~\ref{prop:point:fitting}]

We may assume $J=mn=n^{s+1}$ with $m=n^{s}$ since we can set $y_{J-1}=y_{J}=\cdots=y_{mn-1}$ if $J<mn$. 	
Define 
\begin{equation*}
	a_{j}\coloneqq \lfloor y_{j}/\varepsilon\rfloor 
	\quad \tn{for $j=0,1,\cdots,nm-1$.}
\end{equation*}
Our goal is to construct a function $\phi$ such that $\phi(j)=a_j\varepsilon$ for $j=0,1,\cdots,nm-1$.

For $i=0,1,\cdots,n-1$, we define
\begin{equation*}
	b_{i,\ell}=
	\begin{cases}
		0 & \tn{for}\  \ell=0\\
		a_{im+\ell}-a_{im+\ell-1} & \tn{for}\  \ell=1,2,\cdots,m-1.
	\end{cases}
\end{equation*}
Since $|y_{j}-y_{j -1}|\le \varepsilon$ for all $j$, we have
$|a_{j}-a_{j-1}|\le 1$. It follows that
$b_{i, \ell}\in \{-1,0,1\}$ for $i=0,1,\cdots,n-1$ and $\ell=0,1,\cdots,m-1$. Hence, there exist $c_{i, \ell }\in\{0,1\}$ and $d_{i, \ell }\in\{0,1\}$ such that 
\begin{equation*}
	b_{i, \ell}=c_{i, \ell}-d_{i, \ell}\quad \tn{ for $i=0,1,\cdots,n-1$ and $\ell=0,1,\cdots,m-1$.}
\end{equation*}

Since any $j\in\{0,1,\cdots,nm-1\}$ can be uniquely indexed as $j=im+k$ with $i\in \{0,1,\cdots,n-1\}$ and $k\in \{0,1,\cdots,m-1\}$, we have
\begin{equation*}
	\begin{split}
		a_{j}=a_{im+k}=a_{im}+\sum_{\ell=1}^{k }(a_{im+\ell}-a_{im+\ell-1})
		&=a_{im}+\sum_{\ell=1}^{ k }b_{i,\ell}
		=a_{im}+\sum_{\ell=0}^{ k }b_{i,\ell}\\
		&=a_{im}
		+\sum_{\ell=0}^{ k }c_{i,\ell}
		-\sum_{\ell=0}^{ k }d_{i,\ell}.
	\end{split}
\end{equation*}

There exists  a continuous piecewise linear function $\phi_1$ with $2n$ breakpoints such that
\begin{equation*}
	\phi_1(x)=a_{im} \quad \tn{for any $x\in [im,im+m-1]$ and $i=0,1,\cdots,n-1$.}
\end{equation*}
Then, we have
\begin{equation*}
	\phi_1(j)=a_{im} \quad \tn{for $j=0,1,\cdots,nm-1$,}
\end{equation*}
where $(i,k)$ is the unique index pair satisfying $j=im+k$ with $i\in \{0,1,\cdots,n-1\}$ and $k\in \{0,1,\cdots,m-1\}$.
By Lemma~\ref{lem:cpl(p)}, $\phi_1$ can be realized by a one-hidden-layer ReLU network of width $2n+1$. 

By Lemma~\ref{lem:bit:extraction:from:indices}, there exist
$\phi_2,\phi_3\in \nestnet_{s}\big\{58(s+7)^2(n+1)\big\}$
%two functions $\phi_2$ and $\phi_3$, each of which is realized by a height-$s$ ReLU NestNet at most $588(s+2)^4(n+1)$ parameters,
such that
\[\phi_2(j)=\sum_{\ell=0}^{ k }c_{i,\ell} \quad \tn{and}\quad \phi_3(j)=\sum_{\ell=0}^{ k }d_{i,\ell} \quad \tn{
	for $j=0,1,\cdots,nm-1$,}\]
where $(i,k)$ is the unique index pair satisfying $j=im+k$ with $i\in \{0,1,\cdots,n-1\}$ and $k\in \{0,1,\cdots,m-1\}$.

Hence, by indexing $j\in \{0,1,\cdots,nm-1\}$ as $j=im+k$ for $i=\{0,1,\cdots,n-1\}$ and $k\in\{0,1,\cdots,m-1\}$, we have  
\begin{equation*}
	%	\label{eq:returnaml}
	a_{j}=a_{im}
	+\sum_{\ell=0}^{ k }c_{i,\ell}
	-\sum_{\ell=0}^{ k }d_{i,\ell}
	=\phi_{1}(j)+\phi_{2}(j)-\phi_{3}(j).
\end{equation*}

By defining
\begin{equation*}
	\tildephi(x)\coloneqq \Big(\phi_1(x)+\phi_2(x)+\phi_3(x)\Big)\varepsilon\quad \tn{for any $x\in\R$,}
\end{equation*}
%we have
%\begin{equation*}
%	\begin{split}
	%		\big|\tildephi(j)-y_j\big|
	%		&=\Big|\big(\phi_1(j)+\phi_2(j)+\phi_3(j)\big)\varepsilon-y_j \Big|
	%		=\big|a_j\varepsilon-y_j\big|\\
	%		&=\big|\lfloor y_j/\varepsilon\rfloor\varepsilon-y_j\big|
	%		=\big|\lfloor y_j/\varepsilon\rfloor-y_j/\varepsilon\big|\varepsilon
	%		\le \varepsilon
	%	\end{split}
%\end{equation*}
%for $j=0,1,\cdots,nm-1$. Moreover,
we have $\tildephi(j)=a_j\varepsilon$ for $j=0,1,\cdots,nm-1$ and
$\tildephi$ can be realized by the height-$s$ NestNet in Figure~\ref{fig:phi:from:psi:tphi}.
% and the definition of $\tildephi$ on $(-\infty,0)$ is automatically given by the NestNet.

\begin{figure}[htbp!]
	\centering
	\includegraphics[width=0.78\textwidth]{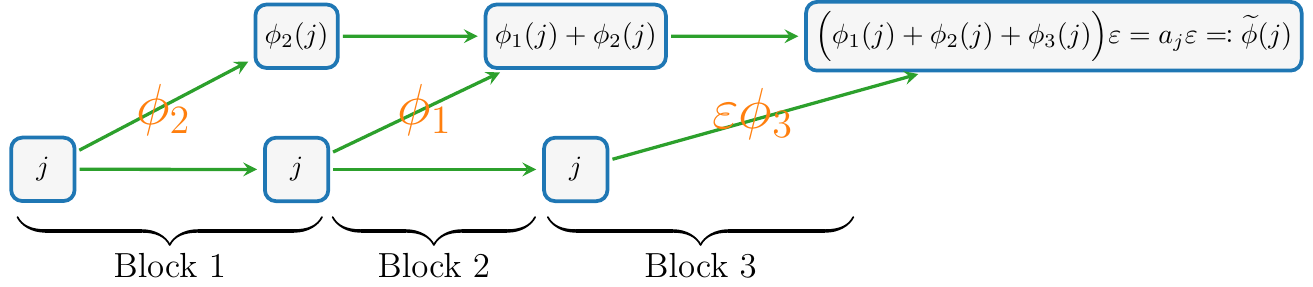}	
	\caption{An illustration of the NestNet realizing $\tildephi$ for $j=0,1,\cdots,J-1$. }
	\label{fig:phi:from:psi:tphi}
\end{figure}

In Figure~\ref{fig:phi:from:psi:tphi},
Block $1$ or $3$ has at most 
\begin{equation*}
	3\Big(58(s+7)^2(n+1)\Big)=174(s+7)^2(n+1)
\end{equation*}
parameters; 
Block $2$  is of width $(2n+1)+2=2n+3$ and depth $1$, and hence has at most
\begin{equation*}
	(2+1)(2n+3)+\big((2n+3)+1\big)2=10n+17
\end{equation*}
parameters.
Then, $\tildephi$ can be realized by a height-$s$ ReLU NestNet with at most
\begin{equation*}
	2\big(174(s+7)^2(n+1)\big)+10n+17=349(s+7)^2(n+1)
\end{equation*}
parameters.
Note that
$\tildephi$ may not be bounded. Thus, 
we define 
\begin{equation*}
	\psi(x)\coloneqq \min\big\{\sigma(x),\, M\big\}\quad \tn{for any $x\in\R$,}
	%	=\frac{\sigma(\sigma(x)+M)-\sigma(-\sigma(x)-M)-\sigma(\sigma(x)-M)-\sigma(-\sigma(x)+M)}{2}
\end{equation*}
where
\begin{equation*}
	M=\max\{ y_{j}: j=0,1,\cdots,nm-1\}.
\end{equation*}
Then, the desired function $\phi$ can be define via $\phi\coloneqq \psi\circ\tildephi$. Clearly, 
\begin{equation*}
	0\le \phi(x)\le M=\max\{ y_{j}: j=0,1,\cdots,J-1\} \quad \tn{for any $x\in \R$.}
\end{equation*}
It follows from $0\le a_j\varepsilon= \lfloor y_j/\varepsilon\rfloor \varepsilon\le y_j\le M$ for $j=0,1,\cdots,J-1$ that
\begin{equation*}
	\phi(j)=\psi\circ\tildephi(j)=\psi\big(a_j\varepsilon\big)
	=\min\big\{\sigma(a_j\varepsilon),\, M\big\}=a_j\varepsilon,
\end{equation*}
implying
\begin{equation*}
	\begin{split}
		\big|\phi(j)-y_j\big|
		=\big|a_j\varepsilon-y_j\big|
		=\Big|\lfloor y_j/\varepsilon\rfloor\varepsilon-y_j\Big|
		=\Big|\lfloor y_j/\varepsilon\rfloor-y_j/\varepsilon\Big|\varepsilon
		\le \varepsilon.
	\end{split}
\end{equation*}

It remains to show that $\phi$ can be realized by a height-$s$ ReLU NestNet with the desired size.
Clearly, $\psi$ can be realized by the network in Figure~\ref{fig:min:sigma:M}, which is of width $4$ and depth $2$.

\begin{figure}[htbp!]
	\centering
	\includegraphics[width=0.68\textwidth]{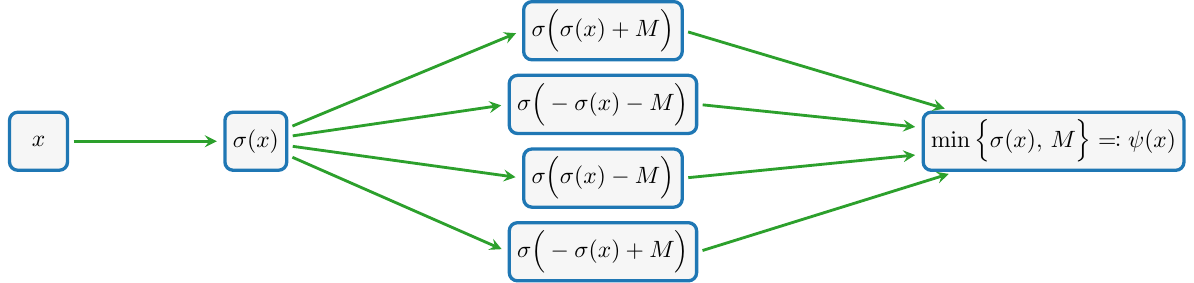}	
	\caption{An illustration of the network realizing $\psi$ based on the fact
		$\min\{a,b\}=\tfrac{1}{2}\big(\sigma(a+b)-\sigma(-a-b)-\sigma(a-b)-\sigma(-a+b)\big)$. }
	\label{fig:min:sigma:M}
\end{figure}

Therefore, $\phi$ can be realized by a height-$s$ ReLU NestNet with at most
\begin{equation*}
	349(s+7)^2(n+1)+ (4+1)4(2+1)\le 350(s+7)^2(n+1)
\end{equation*}
parameters.
Hence, we finish the proof of Proposition~\ref{prop:point:fitting}.
%\end{proof}

\subsection{Proof of Lemma~\ref{lem:bit:extraction} for Proposition~\ref{prop:point:fitting}}
\label{sec:proof:bit:extraction}

We will use the mathematical induction to prove Lemma~\ref{lem:bit:extraction}.
To this end, we introduce two lemmas  for the  base case and the
induction step. 

\begin{lemma}
	\label{lem:bit:extraction:base:case}
	Given any $n\in \N^+$,
	there exists a function $\phi$ realized by a  ReLU network with $128n+294$ parameters such that: For any $\theta_1,\theta_2,\cdots,\theta_{n}\in \{0,1\}$, we have
	\begin{equation}\label{eq:bit:extraction:s=0}
		\phi \big(k+\bin 0.\theta_1\theta_2\cdots\theta_{n}\big)=\sum_{\ell=1}^{k}\theta_\ell \quad \tn{for $k=0,1,\cdots,n$.}
	\end{equation}
\end{lemma}

%\begin{lemma}\label{lem:bit:extraction:induction:step}
%	Given any $n,m,s\in \N^+$ with $m\le n^s$, suppose $\varrho_1:\R\to \R$ is a function satisfying
%	\begin{equation}\label{eq:floor:g1}
%		\varrho_1(x)=\lfloor x \rfloor\quad \tn{for any $\displaystyle x\in \bigcup_{\ell=0}^{n-1}[\ell,\,\ell+1-\delta]$ with $\delta=2^{-nm}$}
%	\end{equation}
%	and $\varrho_2:\R\to\R$ is a function satisfying
%	\begin{equation}\label{eq:any:xi}
%		\varrho_2\big(p+\bin 0.\xi_1\xi_2\cdots\xi_m \big)=\sum_{j=1}^{p}\xi_j\quad \tn{for $p=0,1,\cdots,m$ and any $\xi_1,\xi_2,\cdots,\xi_m\in \{0,1\}$}.
%		%		\,\footnote{By convention, $\sum_{j=n_1}^{n_2} a_j=0$ if $n_1>n_2$, no matter what $a_j$ is for each $j$.}
%	\end{equation}
%	Then, there exists a function $\phi$ realized by a $(\sigma,\varrho_1,\varrho_2)$-activated network  with $298(s+1)^3(n+1)$ parameters such that: For any $\theta_1,\theta_2,\cdots,\theta_{nm}\in \{0,1\}$, we have
%	\begin{equation*}
%		\phi \big(k+\bin 0.\theta_1\theta_2\cdots\theta_{nm}\big)=\sum_{\ell=1}^{k}\theta_\ell \quad \tn{for $k=0,1,\cdots,nm$.}
%	\end{equation*}
%\end{lemma}

\begin{lemma}\label{lem:bit:extraction:induction:step}
	Given any $n,r,\hatn\in \N^+$, if $g \in \nestnet_{r}\{\hatn\}$ satisfying
	\begin{equation}\label{eq:any:xi}
		g\big(p+\bin 0.\xi_1\xi_2\cdots\xi_{n^r} \big)=\sum_{j=1}^{p}\xi_j\quad \tn{for  any $\xi_1,\xi_2,\cdots,\xi_{n^r}\in \{0,1\}$ and $p=0,1,\cdots,n^r$,}
		%		\,\footnote{By convention, $\sum_{j=n_1}^{n_2} a_j=0$ if $n_1>n_2$, no matter what $a_j$ is for each $j$.}
	\end{equation}
	then there exists $\phi \in \nestnet_{r+1}\big\{\hatn+114(r+7)(n+1)\big\}$
%	a function $\phi$ realized by a $(\sigma,\varrho_1,\varrho_2)$-activated network  with $298(s+1)^3(n+1)$ parameters 
	such that: For any $\theta_1,\theta_2,\cdots,\theta_{n^{r+1}}\in \{0,1\}$, we have
	\begin{equation*}
		\phi \big(k+\bin 0.\theta_1\theta_2\cdots\theta_{n^{r+1}}\big)=\sum_{\ell=1}^{k}\theta_\ell \quad \tn{for $k=0,1,\cdots,n^{r+1}$.}
	\end{equation*}
\end{lemma}

The proofs of Lemmas~\ref{lem:bit:extraction:base:case} and \ref{lem:bit:extraction:induction:step} can be found in Sections~\ref{sec:proof:bit:extraction:base:case} and \ref{sec:proof:bit:extraction:induction:step}, respectively.
We remark that the function $\phi$ in Lemma~\ref{lem:bit:extraction:induction:step} is independent of $\theta_1,\theta_2,\cdots,\theta_{nm}$. 
%Thus, the parameters in $\phi$ can be pre-specified. That means $\phi$ can be also  regarded as a function.
%With Lemmas~\ref{lem:bit:extraction:base:case} and \ref{lem:bit:extraction:induction:step} in hand, we are ready to  prove Lemma~\ref{lem:bit:extraction}.
The proof of Lemma~\ref{lem:bit:extraction} mainly relies on Lemma~\ref{lem:bit:extraction:base:case} and repeated applications of Lemma~\ref{lem:bit:extraction:induction:step}. The details can be found below.
\begin{proof}[Proof of Lemma~\ref{lem:bit:extraction}]
%	The proof of Lemma~\ref{lem:bit:extraction} mainly relies on Lemma~\ref{lem:bit:extraction:ReLU:net} and repeated applications of Lemma~\ref{lem:bit:extraction:induction:step}. Below, we show how to prove 

We will use the mathematical induction to prove Lemma~\ref{lem:bit:extraction}.
First, let us consider the base case $s=1$. 
By Lemma~\ref{lem:bit:extraction:base:case}, there exists a function realized by a ReLU network with $128n+294$ parameters such that: For any $\theta_1,\theta_2,\cdots,\theta_{n}\in \{0,1\}$, we have
\begin{equation*}
	\phi \big(k+\bin 0.\theta_1\theta_2\cdots\theta_{n}\big)=\sum_{\ell=1}^{k}\theta_\ell \quad \tn{for $k=0,1,\cdots,n$.}
\end{equation*}
%	Set $m=1$ and define $\varrho_2(x,y)\coloneqq \min \{2x,\,y\}$ for any $x,y\in\R$. Then, for any $\xi\in \{0,1\}$, we have
%	\begin{equation*}
	%		\varrho_2(\bin 0.\xi,\, 0)=\min \big\{2\cdot\bin 0.\xi,\, 0\big\}= \min \big\{\xi,\, 0\big\}=0
	%	\end{equation*}
%	and
%	\begin{equation*}
	%		\varrho_2(\bin 0.\xi,\, 1)=\min \big\{2\cdot\bin 0.\xi,\, 1\big\}=\min \big\{\xi,\, 1\big\}=\xi.
	%	\end{equation*}
%That means $\varrho_2$ satisfies the condition in Equation~\eqref{eq:any:xi}. Then, by Lemma~\ref{lem:bit:extraction:induction:step}, there exists a $(\sigma,\varrho_1,\varrho_2)$-activated network $\phi$ with $570(s+1)^3(n+1)=570(n+1)$ parameters such that 
% \begin{equation*}
	%	\phi \big(k+\bin 0.\theta_1\theta_2\cdots\theta_{n}\big)=\sum_{\ell=1}^{k}\theta_\ell \quad \tn{for $k=0,1,\cdots,n$.}
	%\end{equation*}
	That means Equation~\eqref{eq:bit:extraction:s+1} holds for $s=1$. Moreover, $\phi$ can also be regarded as a height-$1$ ReLU NestNet with $128n+294\le 57(s+7)^2(n+1)$ parameters for $s=1$, which means Lemma~\ref{lem:bit:extraction} is proved for the case $s=1$.
	
	Next, assume Lemma~\ref{lem:bit:extraction} holds for $s=r\in\N^+$. 
	We need to show that it is also true for $s=r+1$ by applying Lemma~\ref{lem:bit:extraction:induction:step}.
%	To this end, we need to construct $\varrho_1$ and $\varrho_2$ in Lemma~\ref{lem:bit:extraction:induction:step}
%	by setting $s=r$  and $m=n^{r}=n^s$  therein. It is easy to construct a continuous piecewise linear function $\varrho_1:\R\to \R$ with $2n$ breakpoints satisfying
%	\begin{equation*}
%		\varrho_1(x)=\lfloor x \rfloor\quad \tn{for any $\displaystyle x\in \bigcup_{\ell=0}^{n-1}[\ell,\,\ell+1-\delta]$ with $\delta=2^{-nm}$.}
%	\end{equation*}
%	Then, $\varrho_1$ satisfies Equation~\eqref{eq:floor:g1}.
%	Moreover, by Lemma~\ref{lem:cpl(p)} with $p=2n$ therein,
%	$\varrho_1$ can be realized by a one-hidden-layer ReLU network of width $2n+1$.  Clearly, $\varrho_1$ has at most $(1+1)(2n+1)+\big((2n+1)+1\big)=6n+4$ parameters.
%		
	By the induction hypothesis, 
	there exists 
	\begin{equation*}
		g\in \nestnet_{r}\Big\{57(r+7)^2(n+1)\Big\}
	\end{equation*}
%	a height-$r$ ReLU NestNet $\varrho_2$ with $587(r+1)^4(n+1)$ parameters 
	such that: For any $\xi_1,\xi_2,\cdots,\xi_{n^{r}}\in \{0,1\}$, we have
	\begin{equation*}
		g \big(k+\bin 0.\xi_1\xi_2\cdots\xi_{n^{r}}\big)=\sum_{\ell=1}^{k}\theta_\ell \quad \tn{for $k=0,1,\cdots,n^{r}$.}
	\end{equation*}
	It follows from $m=n^{r}$ that
	\begin{equation*}
		g\big(p+\bin 0.\xi_1\xi_2\cdots\xi_m \big)=\sum_{j=1}^{p}\xi_j\quad \tn{for any $\xi_1,\xi_2,\cdots,\xi_m\in \{0,1\}$ and $p=0,1,\cdots,m$,}
	\end{equation*}
	which means $g$ satisfies Equation~\eqref{eq:any:xi}.
	Then, by Lemma~\ref{lem:bit:extraction:induction:step} with $m=n^{r}$ and $\hatn=57(r+7)^2(n+1)$ therein, there exists 
	\begin{equation*}
		\phi\in \nestnet_{r+1}\Big\{\hatn + 114(r+7)(n+1)\Big\}
	\end{equation*}
%	a function  $\phi$ realized by a $(\sigma,\varrho_1,\varrho_2)$-activated network with 
%	\begin{equation*}
%		298(s+1)^3(n+1)
%		=298(r+1)^3(n+1)
%		=298(r+1)^3(n+1)
%	\end{equation*}
%	parameters 
	such that: For any $\theta_1,\theta_2,\cdots,\theta_{nm}\in \{0,1\}$, we have
	\begin{equation*}
		\phi \big(k+\bin 0.\theta_1\theta_2\cdots\theta_{nm}\big)=\sum_{\ell=1}^{k}\theta_\ell \quad \tn{for $k=0,1,\cdots,nm$.}
	\end{equation*}
	It follows from $m=n^{r}$ that, for any $\theta_1,\theta_2,\cdots,\theta_{n^{r+1}}\in \{0,1\}$, we have
	\begin{equation*}
		\phi \big(k+\bin 0.\theta_1\theta_2\cdots\theta_{n^{r+1}}\big)=\sum_{\ell=1}^{k}\theta_\ell \quad \tn{for $k=0,1,\cdots,n^{r+1}$,}
	\end{equation*}
	which means Equation~\eqref{eq:bit:extraction:s+1} holds for $s=r+1$.
%	Moreover, the fact $		\varrho_2\in \nestnet_{r}\big\{299(r+1)^4(n+1)\big\}$ implies that $\phi$ can be realized as a height-$(r+1)$ NestNet with at most
%	\begin{equation*}
%		\begin{split}
%			& \hspace{12pt}\underbrace{298(r+1)^3(n+1)}_{\tn{outer network}}
%			+ \underbrace{\big(6n+4\big)}_{\varrho_1} 
%			+  \underbrace{299(r+1)^4(n+1)}_{\varrho_2}\\
%			&\le 299(n+1)\Big( (r+1)^3+ (r+1)^4\Big)\le 299\Big(\big((r+1)+1\big)^4\Big)(n+1),
%		\end{split}
%	\end{equation*}
%	where the last inequality comes from the following fact
%	\begin{equation*}
%		\big((r+1)+1\big)^4\ge (r+1)^4 + (r+1)^3.
%	\end{equation*}
%	Thus, we prove Lemma~\ref{lem:bit:extraction} for the case $s=r+1$, which means we finish the induction step.
Moreover, we have 
\begin{equation*}
	\begin{split}
		\hatn+114(r+7)(n+1)
		&=57(r+7)^2(n+1) +114(r+7)(n+1)\\
		&=57(n+1)\Big((r+7)^2+2(r+7)\Big)\\
		&\le 57(n+1)\big((r+7)+1\big)^2
		=57\big((r+1)+7\big)^2(n+1).
	\end{split}
\end{equation*}
This implies that 
\begin{equation*}
	\phi\in \nestnet_{r+1}\Big\{\hatn + 114(r+7)(n+1)\Big\}
	\subseteq 
	\nestnet_{r+1}\Big\{57\big((r+1)+7\big)^2(n+1)\Big\}.
\end{equation*}
Thus, we prove Lemma~\ref{lem:bit:extraction} for the case $s=r+1$, which means we finish the induction step.
 Hence, by the principle of induction, we complete the proof of Lemma~\ref{lem:bit:extraction}.
\end{proof}

%\begin{lemma}\label{lem:index:decomposition}
%	Given any $m,n\in\N^+$ and $k\in \{1,2,\cdots,nm\}$, there exist a ReLU network $\phi_1$ with
%		such that 
%	\begin{equation}
	%		\phi_1(k)=i \quad \tn{and}\quad \phi_2(k)=j,
	%	\end{equation}
%	where $k=im+j$ is the unique represent for $i\in \{0,1,\cdots,n-1\}$ and $j\in \{1,2,\cdots,m\}$.
%\end{lemma}

%\subsection{Proof of Lemmas in Section~\ref{sec:proof:bit:extraction}}

\subsubsection{Proof of Lemma~\ref{lem:bit:extraction:base:case} for Lemma~\ref{lem:bit:extraction}}
\label{sec:proof:bit:extraction:base:case}

To simplify the proof of Lemma~\ref{lem:bit:extraction:base:case}, we introduce the following lemma.
\begin{lemma}\label{lem:bit:extraction:ReLU:net}
	Given any $n\in \N^+$, there exists a function $\phi$ realized by a ReLU network of width $7$ and depth $2n+1$ such that: For any $\theta_1,\theta_2,\cdots,\theta_{n}\in \{0,1\}$, we have
	\begin{equation*}
		\phi \big(\bin 0.\theta_1\theta_2\cdots\theta_{n},\,k\big)=\sum_{\ell=1}^{k}\theta_\ell \quad \tn{for $k=0,1,\cdots,n$.}
	\end{equation*}
\end{lemma}
Lemma~\ref{lem:bit:extraction:ReLU:net} is the Lemma~$3.5$ of \cite{shijun:2}. The detailed proof can be found therein.
With Lemma~\ref{lem:bit:extraction:ReLU:net} in hand, we are ready to prove Lemma~\ref{lem:bit:extraction:base:case}.

\begin{proof}[Proof of Lemma~\ref{lem:bit:extraction:base:case}]
	By Lemma~\ref{lem:bit:extraction:ReLU:net}, there exists a function $\phi_0$  realized by  a ReLU network of width $7$ and depth $2n+1$ such that: For any $\theta_1,\theta_2,\cdots,\theta_{n}\in \{0,1\}$, we have
	\begin{equation*}
		\phi_0 \big(\bin 0.\theta_1\theta_2\cdots\theta_{n},\,k\big)=\sum_{\ell=1}^{k}\theta_\ell \quad \tn{for $k=1,2,\cdots,n$.}
	\end{equation*}	
	The equation above is not true for $k=0$. We will construct $\phi_2$ such that
	\begin{equation*}
		\phi_2(\bin 0.\theta_1\theta_2\cdots\theta_{n},\,k)
		=\sum_{\ell=1}^{k}\theta_\ell \quad \tn{for $k=0,1,\cdots,n$.}
	\end{equation*}
	To this end, we first set
	\begin{equation*}
		M=\max\big\{|\phi_0(x,\,y)|: x\in [0,1],\, y\in [0,n]\big\}
	\end{equation*}
	and define
	\begin{equation*}
		\phi_1(x,y)\coloneqq \min\big\{M+\phi_0(x,y),\, 2My\big\}\quad \tn{for any $x\in [0,1]$ and $y\in[0,n]$.}
	\end{equation*}
	
	\begin{figure}[htbp!]
		\centering
		\includegraphics[width=0.76\textwidth]{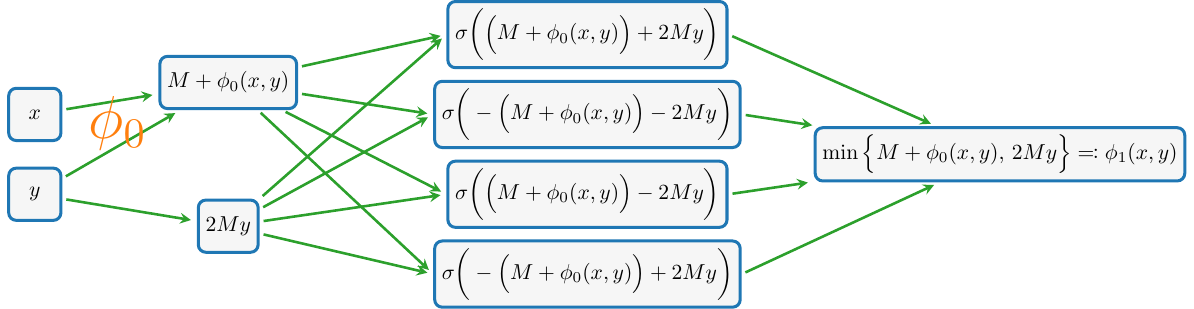}	
		\caption{An illustration of the network realizing $\phi_{1}$ for any $x\in [0,1]$ and $y\in[0,n]$ based on the fact
			$\min\{a,b\}=\tfrac{1}{2}\big(\sigma(a+b)-\sigma(-a-b)-\sigma(a-b)-\sigma(-a+b)\big)$. }
		\label{fig:phi:1:x:y}
	\end{figure}
	
	As we can see from Figure~\ref{fig:phi:1:x:y}, $\phi_1$ can be realized by a ReLU network of width $\max\{7,4\}=7$ and depth $(2n+1)+2=2n+3$.
	Moreover, we have
	\begin{equation*}
		\begin{split}
			\phi_1(\bin 0.\theta_1\theta_2\cdots\theta_{n},\,k)
			&=\min\big\{M+\phi_0(\bin 0.\theta_1\theta_2\cdots\theta_{n},\,k),\, 2Mk\big\}\\
			&=
			\begin{cases}
				M+\sum_{\ell=1}^{k}\theta_\ell &\tn{for}\  k=1,2,\cdots,n\\
				0 &\tn{for}\  k=0.
			\end{cases}
		\end{split}
	\end{equation*}
	Define 
	\begin{equation*}
		\phi_2(x,y)\coloneqq \sigma\big(\phi_1(x,y)-M\big)\quad \tn{for any $x\in [0,1]$ and $y\in[0,\infty)$.}
	\end{equation*}
	Then, $\phi_2$ can be realized by a ReLU network of width $7$ and depth $(2n+3)+1=2n+4$. Moreover, we have
	\begin{equation*}
		\begin{split}
			\phi_2(\bin 0.\theta_1\theta_2\cdots\theta_{n},\,k)
			&=\sigma\big(\phi_1(\bin 0.\theta_1\theta_2\cdots\theta_{n},\,k)-M\big)\\
			&=
			\begin{cases}
				\sigma\big(\sum_{\ell=1}^{k}\theta_\ell)=\sum_{\ell=1}^{k}\theta_\ell  & \tn{for}\   k=1,2,\cdots,n\\
				\sigma(-M)=0 &\tn{for}\  k=0.
			\end{cases}
		\end{split}
	\end{equation*}
	That is,
	\begin{equation*}
		\phi_2(\bin 0.\theta_1\theta_2\cdots\theta_{n},\,k)
		=\sum_{\ell=1}^{k}\theta_\ell \quad \tn{for $k=0,1,\cdots,n$.}
	\end{equation*}

	Next, we will construct $\bmPsi$ to extract $k$ and $\bin 0.\theta_1\theta_2\cdots\theta_n$ from $k+\bin 0.\theta_1\theta_2\cdots\theta_n$.
	It is easy to construct a continuous piecewise linear function $\psi:\R\to \R$ with $2n$ breakpoints satisfying
	\begin{equation*}
		\psi(x)=\lfloor x \rfloor\quad \tn{for any $\displaystyle x\in \bigcup_{\ell=0}^{n-1}[\ell,\,\ell+1-\delta]$ with $\delta=2^{-n}$.}
	\end{equation*}
	By Lemma~\ref{lem:cpl(p)} with $p=2n$ therein, $\psi$ can be realized by a one-hidden-layer ReLU network of width $2n+1$.  By defining 
	\begin{equation*}
		\bmPsi(x)\coloneqq \left[\begin{matrix*}[c]
			x-\psi(x)\\
			\psi(x)
		\end{matrix*}\right]=\left[\begin{matrix*}[c]
			\sigma(x)-\psi(x)\\
			\psi(x)
		\end{matrix*}\right]\quad \tn{for any $x\in [0,\infty)$.}
	\end{equation*}
	Then, $\bmPsi$ can be realized by a one-hidden-layer ReLU network of width $1+ 2(2n+1)=4n+3$. That means, the network realizing $\bmPsi$ has at most 
	\begin{equation*}
		(1+1)(4n+3)+\big((4n+3)+1\big)2=16n+14
	\end{equation*} 
	parameters. Moreover, for any $\theta_1,\theta_2,\cdots,\theta_n\in \{0,1\}$ and $k=0,1,\cdots,n$,
	we have  
	\begin{equation*}
		\psi(k+\bin 0.\theta_1\theta_2\cdots\theta_n)=\lfloor k+\bin 0.\theta_1\theta_2\cdots\theta_n\rfloor = k,
	\end{equation*}
	implying
	\begin{equation*}
		\begin{split}
			\bmPsi(k+\bin 0.\theta_1\theta_2\cdots\theta_n)
			&= \left[\begin{matrix*}[c]
				k+\bin 0.\theta_1\theta_2\cdots\theta_n-\psi(k+\bin 0.\theta_1\theta_2\cdots\theta_n)\\
				\psi(k+\bin 0.\theta_1\theta_2\cdots\theta_n)
			\end{matrix*}\right]\\
			&=\left[\begin{matrix*}[c]
				\bin 0.\theta_1\theta_2\cdots\theta_n\\
				k
			\end{matrix*}\right].
		\end{split}
	\end{equation*}
	
	Finally, the desired function $\phi$ can be defined via $\phi\coloneqq \phi_2\circ \bmPsi$. Clearly, the network realizing $\phi_2$ is of width $7$ and depth $2n+4$, and hence has at most 
	\begin{equation*}
		(7+1)7\big((2n+4)+1\big)=56(2n+5)
	\end{equation*}
	parameters, implying
	$\phi$ can be realized by a ReLU network with at most 
	\begin{equation*}
		56(2n+5) + (16n+14) =128n+294
	\end{equation*}
	parameters. Moreover, for any $\theta_1,\theta_2,\cdots,\theta_n\in \{0,1\}$ and $k=0,1,\cdots,n$, we have 
	\begin{equation*}
		\begin{split}
			\phi(k+\bin 0.\theta_1\theta_2\cdots\theta_n)
			&=\phi_2\circ \bmPsi(k+\bin 0.\theta_1\theta_2\cdots\theta_n)\\
			&=\phi_2(\bin 0.\theta_1\theta_2\cdots\theta_n,\, k)
			= \sum_{\ell=1}^{k} \theta_\ell.
		\end{split}
	\end{equation*}
	Thus, we finish the proof of Lemma~\ref{lem:bit:extraction:base:case}.
\end{proof}

\subsubsection{Proof of Lemma~\ref{lem:bit:extraction:induction:step} for Lemma~\ref{lem:bit:extraction}}
\label{sec:proof:bit:extraction:induction:step}

%Finally, let us prove Lemma~\ref{lem:bit:extraction:induction:step}.
%\begin{proof}[Proof of Lemma~\ref{lem:bit:extraction:induction:step}]

%To simplify the proof of Proposition~\ref{prop:floor:approx}, we introduce the following lemma.
%
%\begin{lemma}\label{lem:floor:approx:exp:nestnet}
%	Given any $v\in \N^+$ and $\delta\in (0,1)$, there exists a  function $\phi$ realized by a ReLU network  of width $v+1$ and depth $4v-3$ such that
%	\begin{equation}%\label{eq:floor:approx:exp:condition}
%		\phi(x)=\lfloor x\rfloor \quad \tn{for any $x\in \bigcup_{\ell=0}^{2^v-1}[\ell,\,\ell+1-\delta]$.}
%	\end{equation}
%\end{lemma}
%
%
%\begin{proof}[Proof of Lemma~\ref{lem:bit:extraction:induction:step}]

The key idea of proving Lemma~\ref{lem:bit:extraction:induction:step} is to construct a network with $n$ blocks, each of which extracts the sum of $n^r$ bits via $g$. Then the whole network can extract the sum of $n^{r+1}$ bits as we expect.

To simplify our notation, we set $m=n^r$.
Given any $nm$ binary bits $\theta_\ell\in \{0,1\}$ for $\ell=1,2,\cdots,nm$,
we divide these $nm$ bits into $n$ classes according to their indices, where the $i$-th class is composed of $m$ bits $\theta_{im+1},\cdots,\theta_{im+m}$ for  $i=0,1,\cdots,n-1$. We will show how to extract the $m$ bits of the $i$-th class, stored in $\bin 0.\theta_{im+1}\cdots\theta_{im+m}$. 
%	The key idea comes from the following equation:
%	\begin{equation*}
	%		\bin 0. \theta_{i m+1}\cdots\theta_{i m+m} 
	%		= \frac{\lfloor 2^m\cdot \bin 0. \theta_{i m+1}\cdots\theta_{nm}\rfloor}{2^m}\quad \tn{for $i=0,1,\cdots,n-1$},
	%	\end{equation*}
%where $\lfloor\cdot\rfloor$ is the floor function.
%Since the target network $\phi$ has only two activation functions $\sigma$ and $g$, not including the floor function. Thus, we need to construct a $(\sigma,g)$-activated network $\phi_1$ to replace/realize the floor function.
%There exists a continuous piecewise linear function $\phi_1\in \cpl(\R,n-2)$ satisfying:
%\begin{itemize}
%	\item $\phi_1$ matches
%	the set of samples
%	\begin{equation*}
	%		\bigcup_{i=0}^{n-1}\big\{(i,i),\,(i+1-\delta,i)\big\},\quad 
	%		\tn{where $\delta=2^{-nm}$.}
	%	\end{equation*} 
%	\item The breakpoint set of $\phi_1$ is 
%	\begin{equation*}
	%		\Big(\bigcup_{i=0}^{n-1}\big\{i,\,i+1-\delta\big\}\Big)\Big\backslash 
	%		\{0,\,n-\delta\}.
	%	\end{equation*} 
%\end{itemize}

First, let us show how to construct a network to extract $k$ and $\bin 0.\theta_1\theta_2\cdots\theta_{nm}$ from $k+0.\theta_1\theta_2\cdots\theta_{nm}$.
%By Lemma~\ref{lem:floor:approx:exp} with setting $v=(s+1)\lceil \log_2(n+1)\rceil \in\N^+$ and $\delta=2^{-nm}$ therein, there exists a function $\phi_0$ realized by a ReLU network of width $v+1$ and depth $4v-3$ such that
%\begin{equation*}
%	\phi_0(x)=\lfloor x\rfloor \quad \tn{for any $x\in \bigcup_{\ell=0}^{2^{v}-1}[\ell,\ell+1-\delta]$.}
%\end{equation*}
By setting $\tilden=2n$ and  Proposition~\ref{prop:floor:approx} with $J=2^{\tilden^r}$ therein, there exists 
\begin{equation*}
	\tildeg\in \nestnet_{r}\big\{36(r+7)\tilden\big\}
	=\nestnet_{r}\big\{36(r+7)(2n)\big\}
	=\nestnet_{r}\big\{72(r+7)n\big\}
\end{equation*}
 such that
\begin{equation*}
	\tildeg(x)=\lfloor x\rfloor \quad 
	\tn{for any $x\in \bigcup_{\ell=0}^{J-1}[\ell,\ell+1-\delta]$.}
\end{equation*}
Observe that 
\begin{equation*}
	J-1=2^{\tilden^r}=2^{(2n)^r}-1\ge 2^{2(n^r)}-1= 2^{2m}-1=4^m-1\ge m^2\ge nm.
\end{equation*}
%Then, we have 
%\begin{equation*}
%	\phi_0(x)=\lfloor x\rfloor \quad \tn{for any $x\in \bigcup_{\ell=0}^{nm}[\ell,\ell+1-\delta]$.}
%\end{equation*}
It follows from $\bin 0.\theta_1\theta_2\cdots\theta_{nm}\le 1-2^{-nm}=1-\delta$ that
\begin{equation*}
	k+\bin 0.\theta_1\theta_2\cdots\theta_{nm}\in \bigcup_{\ell=0}^{nm}[\ell,\ell+1-\delta]
	\subseteq \bigcup_{\ell=0}^{J-1}[\ell,\ell+1-\delta]
\end{equation*}
for $k=0,1,\cdots,nm$. Thus, we have
\begin{equation}\label{eq:extract:k}
	\tildeg(k+\bin 0.\theta_1\theta_2\cdots\theta_{nm})
	=k\quad 
	\tn{for $k=0,1,\cdots,nm$.}
\end{equation}

%Recall that
%\begin{equation*}
%	\varrho_1(x)=\lfloor x\rfloor\quad \tn{ for any $x\in \bigcup_{\ell=0}^{n-1}[\ell,\ell+1-\delta]$ with $\delta=2^{-nm}$}.
%\end{equation*}
It is easy to verify that
\begin{equation*}
	2^m \cdot \bin 0. \theta_{i m+1}\cdots \theta_{nm}\in \bigcup_{\ell=0}^{2^m-1}[\ell,\ell+1-\delta]\quad \tn{ for $i=0,1,\cdots,n-1$}.
\end{equation*}
Since $2^m-1=2^{n^r}-1\le 2^{(2n)^r}-1=J-1$, we have
\begin{equation*}
	\tildeg(2^m \cdot \bin 0. \theta_{i m+1}\cdots \theta_{nm})
	=\lfloor 2^m \cdot \bin 0. \theta_{i m+1}\cdots \theta_{nm} \rfloor
	\quad \tn{ for  $i =0,1,\cdots,n-1$}.
\end{equation*}

Therefore, for $i= 0,1,\cdots,n-1$, we have
\begin{equation*}
	\bin 0. \theta_{im+1}\cdots\theta_{im +m} 
	= \frac{\lfloor 2^m\cdot \bin 0. \theta_{i m+1}\cdots\theta_{nm}\rfloor}{2^m} 
	=\frac{\tildeg(2^m\cdot \bin 0. \theta_{i m+1}\cdots\theta_{nm})}{2^m}
\end{equation*}
and 
\begin{equation*}
	\begin{split}
		%		&\quad 	
		\bin 0. \theta_{(i+1) m+1}\cdots\theta_{nm} 
		&= 2^m\Big(\bin 0. \theta_{im+1}\cdots\theta_{nm} - \bin 0. \theta_{im+1}\cdots\theta_{im+m}\Big)\\
		&= 2^m\Big(\bin 0. \theta_{im+1}\cdots\theta_{nm} - \frac{\tildeg(2^m\cdot \bin 0. \theta_{i m+1}\cdots\theta_{nm})}{2^m}\Big).
		%		=\phi_2(\bin 0. \theta_{i m+1}\cdots\theta_{nm}).
	\end{split}
\end{equation*}
By defining 
\begin{equation*}
	\phi_1(x)\coloneqq \frac{\tildeg(2^m x)}{2^m}
	\quad \tn{and}\quad
	\phi_2(x)\coloneqq 2^m\Big(x - \frac{\tildeg(2^m x)}{2^m}\Big)=\Big(\sigma(x) - \frac{\tildeg(2^m x)}{2^m}\Big) \quad \tn{for $x\ge 0$,}
\end{equation*}
 we have
\begin{equation}\label{eq:theta:phi:1}
	\bin 0. \theta_{im+1}\cdots\theta_{im +m} 
	= \phi_1(\bin 0. \theta_{i m+1}\cdots\theta_{nm})
\end{equation}
and 
\begin{equation}\label{eq:theta:phi:2}
	\begin{split}
		%		&\quad 	
		\bin 0. \theta_{(i+1) m+1}\cdots\theta_{nm} 
		=\phi_2(\bin 0. \theta_{i m+1}\cdots\theta_{nm})
		%		=\phi_2(\bin 0. \theta_{i m+1}\cdots\theta_{nm}).
	\end{split}
\end{equation}
for any $i\in \{0,1,\cdots,n-1\}$. Moreover, $\phi_1$ can be realized by a one-hidden-layer $\tildeg$-activated network of width $1$; $\phi_2$ can be realized by a one-hidden-layer $(\sigma,\tildeg)$-activated network of width $2$.

%	Next, let us show how to co
Define 
\begin{equation*}
	\phi_{3,i}(x)\coloneqq \min\{\sigma(x-i m),\,m\}\quad \tn{for any $x\in\R$ and $i=0,1,\cdots,n-1$.}
\end{equation*}
%	 where $\sigma$ is ReLU, i.e., $\sigma(x)=\max\{0,x\}$. 
For any $k\in\{1,2,\cdots,nm\}$, there exist $k_1\in \{0,1,\cdots,n-1\}$ and $k_2\in\{1,2,\cdots,m\}$ such that $k=k_1m+k_2$. Then we have
\begin{equation}\label{eq:phi:3i}
	\phi_{3,i}(k)= \min\{\sigma(k-i m),\,m\}=
	\begin{cases}
		m &\tn{if}\   i\le k_1-1\\
		k_2 &\tn{if}\   i=k_1 \\
		0 &\tn{if}\   i\ge k_1+1.\\
	\end{cases}
\end{equation}
Observe that
\begin{equation*}
	\begin{split}					\big\{1,2,\cdots,k\big\}&=\big\{1,2,\cdots,k_1m+k_2\big\}\\
		&=\Big(\bigcup_{i=1}^{k_1-1}\big\{im+j: j=1,2,\cdots,m\big\}\Big)\bigcup \big\{k_1m+j: j=1,2,\cdots,k_2\big\}.
	\end{split}
\end{equation*}
It follows that
\begin{equation}\label{eq:sum:theta:k}
	\begin{split}
		\sum_{\ell=1}^k\theta_\ell
		=\sum_{\ell=1}^{k_1m+k_2} \theta_\ell
		&=\sum_{i=0}^{k_1-1} \bigg(\sum_{j=1}^m \theta_{i m+j} \bigg)
		+ \sum_{j=1}^{k_2}\theta_{k_1 m+j} 
		+ 0\\
		&=\sum_{i=0}^{k_1-1} \bigg(\sum_{j=1}^m \theta_{i m+j} \bigg)
		+ \sum_{i=k_1}^{k_1}\bigg(\sum_{j=1}^{k_2}\theta_{i m+j}\bigg)
		+ \sum_{i=k_1+1}^{n-1}\bigg(\sum_{j=1}^0 \theta_{i m+j}\bigg)\\
		&=\sum_{i=0}^{k_1-1} \bigg(\sum_{j=1}^{\phi_{3,i}(k)} \theta_{i m+j} \bigg)
		+ \sum_{i=k_1}^{k_1}\bigg(\sum_{j=1}^{\phi_{3,i}(k)}\theta_{i m+j}\bigg)
		+ \sum_{i=k_1+1}^{n-1}\bigg(\sum_{j=1}^{\phi_{3,i}(k)} \theta_{i m+j}\bigg)\\
		&=\sum_{i=0}^{n-1}\bigg(\sum_{j=1}^{\phi_{3,i}(k)} \theta_{i m+j}\bigg)
	\end{split}
\end{equation}
for $k\in \{1,2,\cdots,nm\}$,
where the %penultimate 
second to last equality 
comes from Equation~\eqref{eq:phi:3i}.
It is easy to verify that Equation~\eqref{eq:sum:theta:k} also holds for $k=0$, i.e., 
\begin{equation*}
	\begin{split}
		\sum_{\ell=1}^0\theta_\ell
		=0=\sum_{i=0}^{n-1}\bigg(\sum_{j=1}^{0} \theta_{i m+j}\bigg)=\sum_{i=0}^{n-1}\bigg(\sum_{j=1}^{\phi_{3,i}(0)} \theta_{i m+j}\bigg).
	\end{split}
\end{equation*}
Therefore, we have 
\begin{equation}\label{eq:theta:sum:decomposition}
	\begin{split}
		\sum_{\ell=1}^k\theta_\ell
		=\sum_{i=0}^{n-1}\bigg(\sum_{j=1}^{\phi_{3,i}(k)} \theta_{i m+j}\bigg) \quad \tn{for any $k\in \{0,1,\cdots,nm\}$.}
	\end{split}
\end{equation}

Fix $i\in \{0,1,\cdots,n-1\}$.
By setting $p=\phi_{3,i}(k)\in \{0,1,\cdots,m\}$ and $\xi_j=\theta_{im+j}$ for $j=1,2,\cdots,m$ in Equation~\eqref{eq:any:xi}, we have
\begin{equation}\label{eq:g:theta:sum}
	g\big(\phi_{3,i}(k)+\bin  0.\theta_{i m+1}\theta_{i m+2}\cdots\theta_{i m+m}\big)
	=\sum_{j=1}^{\phi_{3,i}(k)} \theta_{i m+j}.
\end{equation}

%	Then, the architecture of the target network $\phi$
With Equations~\eqref{eq:extract:k}, \eqref{eq:theta:phi:1}, \eqref{eq:theta:phi:2}, \eqref{eq:theta:sum:decomposition}, and \eqref{eq:g:theta:sum} in hand, we are ready to construct the desired function $\phi$, which can be realized by the NestNet in Figure~\ref{fig:nmBits}. Clearly, we have 
\begin{equation*}
	\phi(k+\bin 0.\theta_1\cdots\theta_{nm})=\sum_{\ell=1}^{k}\theta_\ell\quad \tn{for $k=0,1,\cdots,nm$.}
\end{equation*}
Note that $nm=n\cdot  n^r=n^{r+1}$. Then we have
\begin{equation*}
	\phi(k+\bin 0.\theta_1\cdots\theta_{n^{r+1}})=\sum_{\ell=1}^{k}\theta_\ell\quad \tn{for $k=0,1,\cdots,n^{r+1}$.}
\end{equation*}

\begin{figure}[htbp!]       
	\centering          
	\includegraphics[width=0.99\textwidth]{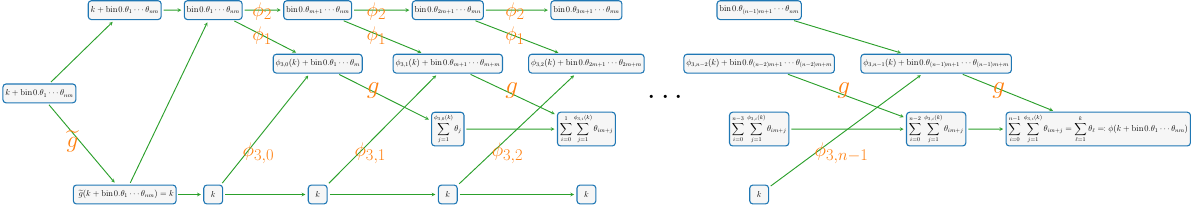}
	\caption{An illustration of the NestNet realizing $\phi$ based on Equations~\eqref{eq:extract:k}, \eqref{eq:theta:phi:1}, \eqref{eq:theta:phi:2}, \eqref{eq:theta:sum:decomposition}, and \eqref{eq:g:theta:sum}. Here, $g$ and $\tildeg$ are regarded as activation functions.}
	\label{fig:nmBits}
\end{figure}

It remains to estimate the number of parameters in the NestNet realizing $\phi$.
Recall that $\phi_1$ can be realized by a one-hidden-layer $\tildeg$-activated network of width $1$ and $\phi_2$ can be realized by a one-hidden-layer $(\sigma,\tildeg)$-activated network of width $2$.
%To this end, we need to estimate the number of parameters in Blocks $1$ and $2$ of Figure~\ref{fig:nmBits}.
%First, let us estimate the number of parameters in Block $1$. Recall that $\phi_0$ can be realized by a $\sigma$-activated network of width $v+1$ and depth $4v-3$. Then, Block $1$ of Figure~\ref{fig:nmBits} is of width $(v+1)+1=v+2$
%and depth $(4v-3)+2=4v-1$. Thus, the number of parameters in Block $1$ of Figure~\ref{fig:nmBits} is bounded by
%\begin{equation*}
%	\begin{split}
%		(v+2)\Big((v+2)+1\Big)\Big((4v-1)+1\Big)
%		&\le 4(v+2)^3
%		=4 \Big( (s+1)\lceil \log_2(n+1)\rceil+2\Big)^3\\
%		&\le 4(s+1)^3 \Big(\lceil \log_2(n+1)\rceil+2\Big)^3\le 256(s+1)^3n,
%	\end{split}
%\end{equation*}
%where the last inequality comes from the following fact
%\begin{equation*}
%	\lceil \log_2(x+1) \rceil+2\le \log_2(x+1)+3\le 4x^{1/3}\quad \tn{for any $x\in [1,\infty)$.}
%\end{equation*}
%
%Next, let us estimate the number of parameters in Block $2$  of Figure~\ref{fig:nmBits}. 
%Clearly, $\phi_1(x)= \varrho_1(2^m x)/2^m$ implies that $\phi_1$ can be realized by a $\varrho_1$-activated network of width $1$ and depth $1$.
%Moreover, 
%\begin{equation*}
%	\phi_2(x)= 2^m\big(x - \varrho_1(2^m x)/2^m\big)
%	=2^m \sigma(x)-\varrho_1(2^m x) \quad \tn{for any $x\ge 0$.}
%\end{equation*}
%Thus, $\phi_2$ limited on $[0,\infty)$ can be realized by a $(\sigma,\varrho_1)$-activated network of width $2$ and depth $1$. 
%We remark that only the property of $\phi_2$ on $[0,1]$ is used   in Figure~\ref{fig:nmBits}. 

Observe that 
\begin{equation*}
	\min\{a,b\}=\tfrac{1}{2}\big(\sigma(a+b)-\sigma(-a-b)-\sigma(a-b)-\sigma(-a+b)\big)\quad \tn{for any $a,b\in\R$}.
\end{equation*}
As we can see from Figure~\ref{fig:phi:3:i:arc},  $\phi_{3,i}$ can be realized by a $\sigma$-activated network of width $4$ and depth $2$ for each $i\in \{0,1,\cdots,n-1\}$.

\begin{figure}[htbp!]
	\centering
	\includegraphics[width=0.72\textwidth]{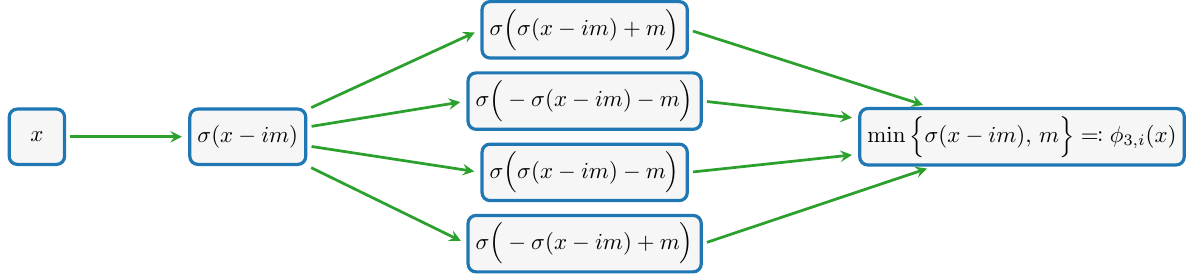}	
	\caption{An illustration of $\phi_{3,i}$ for each $i\in \{0,1,\cdots,n-1\}$. }
	\label{fig:phi:3:i:arc}
\end{figure}
%By Lemma~\ref{lem:cpl(p)}, we have
%\begin{equation*}
%	\phi_1,\phi_2\in \cpl(n-2) \subseteq  
%	\NNF(\NNinput=1\NNspace\NNwidth=5\NNspace\NNdepth=n-1\NNspace\NNoutput=1).
%\end{equation*}

Thus, the network in Figure~\ref{fig:nmBits} can be regarded as a $(\sigma,g,\tildeg)$-activated network of width $2+1+1+1+4+1=10$
and 
depth
$2+(2+1)n=3n+2$. 
Recall that $g\in \nestnet_{r}\{\hatn\}$ and $\tildeg\in \nestnet_{r}\{72(r+7)n\}$.
This implies that
$\phi$ can be realized by a height-$(r+1)$ NestNet with at most
\begin{equation*}
	\underbrace{(10+1)10\big((3n+2)+1\big)}_{\tn{outer network}}
	\ +\ 
	\underbrace{\hatn}_{g}
	\   +\
	\underbrace{72(r+7)n}_{\tildeg}
	\  \le \ 
	\hatn + 114(r+7)(n+1)
\end{equation*}
parameters, which means we finish the proof of Lemma~\ref{lem:bit:extraction:induction:step}.

%Then,
%Block $2$ of Figure~\ref{fig:nmBits} is a $(\sigma,\varrho_1,\varrho_2)$-activated network of width 
%$2+1+1+1+4+1=10$
%and depth $(2+1)n+1=3n+1$. That means the number of parameters in Block $2$ of Figure~\ref{fig:nmBits}  is bounded by 
%\begin{equation*}
%	10(10+1)\big((3n+1)+1\big)\le 110(3n+2).
%\end{equation*}
%Therefore, $\phi$ can be realized by a $(\sigma,\varrho_1,\varrho_2)$-activated network with  at most
%\begin{equation*}
%	256(s+1)^3n+ 110(3n+2)\le (256+330/8)(s+1)^3 (n+1)\le 298(s+1)^3 (n+1)
%\end{equation*}

%parameters, which means we finish the proof of Lemma~\ref{lem:bit:extraction:induction:step}.
%\end{proof}

\fi	
\end{document}